\documentclass[10pt,letterpaper]{article}
\usepackage[whole]{bxcjkjatype}
\usepackage{natbib}
\usepackage{aaai22}
\nocopyright
\usepackage{times}
\usepackage{helvet}
\usepackage{courier}
\usepackage{tabularx}
\usepackage{xparse}
\usepackage{amsmath}
\usepackage{amsthm}
\usepackage{amssymb}
\usepackage{amsfonts}
\usepackage{xspace}
\usepackage{relsize}
\usepackage{graphicx}
\usepackage{multirow}
\usepackage{url}
\usepackage{comment}
\usepackage{math_commands}
\usepackage[svgnames]{xcolor}
\usepackage{hyperref}
\usepackage{textcomp}
\usepackage{adjustbox}
\usepackage{multirow}
\usepackage{multicol}
\usepackage{subcaption}

\usepackage{algorithm}
\usepackage[noend]{algpseudocode} 

\usepackage{microtype}  
\usepackage{booktabs}   

\usepackage{xr}

\usepackage{epsdice}

\usepackage{stackengine}
\stackMath
\newcommand\tsup[2][2]{%
 \def\useanchorwidth{T}%
  \ifnum#1>1%
    \stackon[-.5pt]{\tsup[\numexpr#1-1\relax]{#2}}{\scriptscriptstyle\sim}%
  \else%
    \stackon[.5pt]{#2}{\scriptscriptstyle\sim}%
  \fi%
}

\renewcommand{\tilde}[1]{\tsup[1]{#1}}
\newcommand{\uline}[1]{\underline{#1}}

\newcommand{\function}[1]{\textsc{#1}}

\newcommand{\mycolor}[2]{\textcolor{#1}{#2}}

\newcommand{\red}[1]{\mycolor{red}{#1}}
\newcommand{\green}[1]{\mycolor{Green}{#1}}
\newcommand{\blue}[1]{\mycolor{blue}{#1}}

\newcommand{\gray}[1]{\mycolor{gray}{#1}}

\newcommand{\orange}[1]{\mycolor{orange}{#1}}

\newcommand{\violet}[1]{\mycolor{violet}{#1}}

\newcommand{\spc}[2][c]{%
  \begin{tabular}[#1]{@{}c@{}}#2\end{tabular}}
\def\_{\\[-0.3em]}

\newcommand{\si}[1]{\mathrm{[#1]}}

\newtheorem{defi}{Definition}
\newtheorem{ex}{Example}
\newtheorem{theo}{Theorem}
\newtheorem{lemma}{Lemma}
\newtheorem{corollary}{Corollary}

\newtheorem{conv}{Convention}
\newtheorem{fact}{Fact}

\makeatletter
\let\@myref\ref

\newcommand{\refsec}[1]{Sec.\,\@myref{#1}}
\newcommand{\refseq}[1]{Sec.\,\@myref{#1}}
\newcommand{\refig}[1]{Fig.\,\@myref{#1}}
\newcommand{\reftbl}[1]{Table \@myref{#1}}
\newcommand{\refstep}[1]{Step \@myref{#1}}
\newcommand{\refalgo}[1]{Alg.\,\@myref{#1}}
\newcommand{\refchap}[1]{Chap.\,\@myref{#1}}
\newcommand{\reflst}[1]{List \@myref{#1}}
\newcommand{\refeq}[1]{Eq.\,\@myref{#1}}

\newcommand{\reftheo}[1]{Thm.\,\@myref{#1}}
\newcommand{\refline}[1]{line\,\@myref{#1}}
\newcommand{\refdef}[1]{Def.\, \@myref{#1}}
\newcommand{\refex}[1]{Example\,\@myref{#1}}
\newcommand{\refconv}[1]{Conv.\,\@myref{#1}}
\newcommand{\reffact}[1]{Fact.\,\@myref{#1}}

\newcommand{\refsecs}[2]{Sec.\,\@myref{#1}-\@myref{#2}}
\newcommand{\refseqs}[2]{Sec.\,\@myref{#1}-\@myref{#2}}
\newcommand{\refigs}[2]{Fig.\,\@myref{#1}-\@myref{#2}}
\newcommand{\reftbls}[2]{Tables \@myref{#1}-\@myref{#2}}
\newcommand{\refsteps}[2]{Steps \@myref{#1}-\@myref{#2}}
\newcommand{\refalgos}[2]{Alg.\,\@myref{#1}-\@myref{#2}}
\newcommand{\refchaps}[2]{Chap.\,\@myref{#1}-\@myref{#2}}
\newcommand{\reflsts}[2]{Lists \@myref{#1}-\@myref{#2}}
\newcommand{\refeqs}[2]{Eq.\,\@myref{#1}-\@myref{#2}}
\newcommand{\refpages}[2]{p.\pageref{#1}-\@myref{#2}}
\newcommand{\reftheos}[2]{Thm.\,\@myref{#1}-\@myref{#2}}
\newcommand{\reflines}[2]{line\,\@myref{#1}-\@myref{#2}}
\newcommand{\refdefs}[2]{Def.\,\@myref{#1}-\@myref{#2}}
\newcommand{\refexs}[2]{Example\,\@myref{#1}-\@myref{#2}}
\newcommand{\refconvs}[2]{Conv.\,\@myref{#1}-\@myref{#2}}
\newcommand{\reffacts}[2]{Facts.\,\@myref{#1}-\@myref{#2}}

\makeatother

\newcounter{list}[section]

\makeatletter
\@ifundefined{citet}{
  \NewDocumentCommand{\citet}{o m}{%
    \IfNoValueTF{#1}%
      {\citeauthor{#2} (\citeyear{#2})}
      {\citeauthor{#2} (\citeyear[#1]{#2})}%
  }
}{}
\@ifundefined{citep}{
  \NewDocumentCommand{\citep}{o m}{%
    \IfNoValueTF{#1}%
      {\cite{#2}}
      {\cite[#1]{#2}}%
  }
}{}
\makeatother

\newlength{\maxwidth}
\newcommand{\algalign}[2]
{\makebox[\maxwidth][r]{$#1{}$}${}#2$}

\newcommand{\dl}{Deep Learning\xspace}

\newcommand{\sota}{State-of-the-Art\xspace}

\makeatletter

\newcommand{\newheuristic}[2]{%
 \def#1{%
  \ifmmode%
  h^\text{#2}\xspace%
  \else%
  \text{#2}\xspace%
  \fi%
 }%
}

\newheuristic{\lmcut}{LMcut}
\newheuristic{\mands}{M\&S}
\newheuristic{\pdb}{PDB}
\newheuristic{\ff}{FF}
\newheuristic{\ce}{CEA}
\newheuristic{\cg}{CG}
\newheuristic{\ad}{add}
\newheuristic{\hmax}{max}
\newheuristic{\lc}{LC}
\newheuristic{\blind}{blind}

\newcommand{\newlearnedheuristic}[2]{%
 \def#1{%
  \ifmmode%
  H^\text{#2}\xspace%
  \else%
  \text{#2}\xspace%
  \fi%
 }%
}

\newlearnedheuristic{\Hlmcut}{LMcut}
\newlearnedheuristic{\Hmands}{M\&S}
\newlearnedheuristic{\Hpdb}{PDB}
\newlearnedheuristic{\Hff}{FF}
\newlearnedheuristic{\Hce}{CEA}
\newlearnedheuristic{\Hcg}{CG}
\newlearnedheuristic{\Had}{add}
\newlearnedheuristic{\Hmax}{max}
\newlearnedheuristic{\Hlc}{LC}
\newlearnedheuristic{\Hblind}{blind}

\newcommand{\newUnitCostHeuristic}[2]{%
 \def#1{%
  \ifmmode%
  \hat{h}^\text{#2}\xspace%
  \else%
  \text{#2}\xspace%
  \fi%
 }%
}

\newUnitCostHeuristic{\lmcuto}{LMcut}
\newUnitCostHeuristic{\mandso}{M\&S}
\newUnitCostHeuristic{\ffo}{FF}
\newUnitCostHeuristic{\ceo}{CEA}
\newUnitCostHeuristic{\cgo}{CG}
\newUnitCostHeuristic{\ado}{add}
\newUnitCostHeuristic{\gco}{GoalCount}
\newUnitCostHeuristic{\lco}{LC}

\makeatother

\allowdisplaybreaks

\title{
Dr. Neurosymbolic, or: How I Learned to Stop Worrying and Accept Statistics\\
\small (as well as Machine Learning and \dl)
}

\author{
Masataro Asai\textsuperscript{\rm 1}
}
\affiliations {
\textsuperscript{\rm 1}MIT-IBM Watson AI Lab,
masataro.asai@ibm.com
}

\allowdisplaybreaks
\begin{document}
\maketitle
\begin{abstract}
The symbolic AI community is increasingly trying to embrace machine learning
in neuro-symbolic architectures, yet is still struggling due to cultural barriers.
To break the barrier,
this highly opinionated personal memo attempts to explain and rectify the conventions in
Statistics, Machine Learning, and \dl from the viewpoint of outsiders.
It provides a step-by-step protocol for designing a machine learning system
that satisfies a minimum theoretical guarantee necessary for being taken seriously by the symbolic AI community,
i.e., it discusses \emph{in what condition we can stop worrying and accept it.}
Some highlights:
\begin{itemize}
 \item
Most textbooks are written for those who plan to specialize in Stat/ML/DL and are supposed to accept jargons.
This memo is for experienced symbolic researchers that
hear a lot of buzz but are still uncertain and skeptical.

 \item
Information on Stat/ML/DL is currently too scattered or too noisy to invest in.
This memo prioritizes compactness
and pays special attention to concepts that resonate well with symbolic paradigms.
I hope this memo offers time savings.

 \item
It prioritizes general mathematical modeling
and does not discuss any specific function approximator, such as neural networks (NNs), SVMs, decision trees, etc.

 \item
It is open to corrections. 
Consider this memo as something similar to a blog post taking the form of a paper on Arxiv.
\end{itemize}
\end{abstract}
\section{Overview}
\label{sec:toc}

This memo is structured as follows.
\refsec{sec:probability} describes various quantities mathematically defined from probability distributions, and
\refsec{sec:statistics} describes the notions that only make sense in the applied settings.
I separated these sections to distinguish between the mathematical and the applied notions in statistics.

\refsec{sec:ml} discusses \emph{machine learning as a proof system}.
Perhaps the most important message in this section is the notion that
statistical ML is
(1) \emph{sound}, i.e., its optima do not generate invalid predictions, and
(2) \emph{incomplete}, i.e., its optima may never generate some valid predictions, but that
(3) \emph{generalization makes it complete}, i.e., it can generate valid unseen predictions.
Machine learning methods that are not shown to be in this form are not worth trying,
especially from the viewpoint of a \emph{user} rather than a \emph{researcher} of Stat/ML/DL.
\refsec{sec:loss} discusses how such a system can lead to usual square curve fitting
and how loss functions are defined.

\refsec{sec:modeling} discusses \emph{statistical modeling},
a principled procedure for building a complex model.
While modern Machine Learning is criticized as \emph{art} or \emph{alchemy},
statistical modeling somewhat standardizes the design of \dl systems.
Just following this procedure allows you to define a statistically sound model.
I contrast statistical modeling with other branches of constraint modeling to demonstrate the similarity,
such as MILP, (MAX)SAT, ASP, CSP, SMT.

Finally,
\refsec{sec:variational}
discusses one major practical approximation method for machine learning (VAEs \cite{kingma2014semi}).
I not only demonstrate an example of a specific case,
but also propose a general algorithm for systematically performing those approximations.
Such an algorithm is poorly documented in the existing literature
and could standardize the design process of \dl systems.
The resulting algorithm is published online,
providing a Prolog implementation \cite{prolog-elbo} and
a practical python implementation integrated with Pytorch Lightning \cite{elbonara}.
\refsec{sec:expectation} explains how the loss function formulae that appear in these methods are computed in practice.

The appendix covers less important topics in light of \dl applications.
\refsec{sec:measure} briefly covers the measure theory to define random variables and probability distributions.
A job seeker should at least be aware of the concepts (I was once asked about them during a job interview).
\refsec{sec:tier2} contains more concepts not discussed in \refsec{sec:probability}.
\refsec{sec:frequentist} briefly covers frequentist statistical learning theory (e.g., PAC learning).
\refsec{sec:likelihood-free} explains
a subset of GANs \cite{goodfellow2014generative} that are sound instances of machine learning
(Vanilla GANs are not sound, therefore are unstable to train).
\refsec{sec:uncertainty} discusses uncertainty, confidence, pseudocounts, and conjugate priors.
\refsec{sec:zoo} contains a Distribution Zoo, which helps select which distribution to use for specific applications.
\refsec{sec:future-work} discusses a list of peripheral topics that we plan to include in the future revisions.

\section{Formal Concepts in Statistics}
\label{sec:probability}

For practical purposes, there is no need to understand the probability theory via
axiomatic measure theory (\refsec{sec:measure}) unless you try to solve a deep theoretical problem.
This is because most complications are due to ill-behaved subsets of $\R$ (e.g., sets of all irrational numbers),
which do not exist in the real world.
Indeed, in practice, all ``continuous values'' in modern computers are floating-points with certain widths.
Hence, it is safe to treat the continuous and the discrete entities in the same manner and
I do not distinguish an integral $\int_x f(x)dx$ and a sum $\sum_x f(x)$ hereafter.
Less important or more advanced concepts are included in the appendix \refsec{sec:tier2}.

\begin{defi}
 A \emph{probability distribution} of a random variable $\rx$ defined on a set $X$
 is a function $f$ from a value $x\in X$ to $f(x)\in \R^{0+}$
 which satisfies $1=\sum_{x\in X} f(x)$.
\end{defi}

$f$ is called a probability mass function (PMF) when $X$ is discrete
and a probability density function (PDF) when $X$ is continuous.
Typically, we denote a probability distribution as $f=p(\rx)$.
Confusingly, the letter $p$ and $\rx$ \emph{together} denotes a single function:
Unlike normal mathematical functions where $f(x)$ and $f(y)$ are equivalent under the variable substitution,
two notations $p(\rx)$ and $p(\ry)$ denote different PMFs/PDFs, i.e.,
$p(\rx)=f_1(\rx)$, $p(\ry)=f_2(\ry)$, and $f_1\not=f_2$ to be explicit.
To denote two different distributions for the same random variable,
an alternative letter replaces $p$, e.g., $q(\rx)$.

\begin{defi}
 $f(x)=p(\rx=x)=p(x)$ is called a probability mass/density of observing an event $\rx=x$.
\end{defi}

\begin{defi}
 \label{defi:joint}
 A \emph{joint distribution} $p(\rx, \ry)$ is a function of $(x,y)\in X\times Y$
 satisfying $1=\sum_{(x,y)\in X\times Y} p(x,y)$,
 $p(x)=\sum_{y\in Y} p(x,y)$, and $p(y)=\sum_{x\in X} p(x,y)$,
 given $p(\rx)$, $p(\ry)$.
\end{defi}

\begin{defi}
 $f(x,y)=p(\rx=x,\ry=y)=p(x,y)$ is called a probability mass/density of observing $\rx=x$ and $\ry=y$ at the same time
 (also written as $p(\rx=x\land\ry=y)$).
\end{defi}

\begin{conv}
 A \emph{marginal distribution} of $p(\rx,\ry_1,\ry_2\ldots)$
 is usually
 a single variable distribution such as
 $p(\rx)=\sum_{\ry_1,\ry_2\ldots} p(\rx,\ry_1,\ry_2\ldots)$,
 but could also be
 a multi-variable distribution such as $p(\rx,\ry_1)=\sum_{\ry_2\ldots} p(\rx,\ry_1,\ry_2\ldots)$.
\end{conv}

\begin{defi}
 Random variables $\rx, \ry$ are independent when $p(\rx,\ry)=p(\rx)p(\ry)$,
 denoted by $\rx \perp \ry$.
\end{defi}

\begin{defi}
 Random variables $\rx, \ry$ are independent and identically distributed (i.i.d) when
 $p(\rx)=p(\ry)$ and $\rx\perp \ry$.
\end{defi}

\begin{defi}
 \label{defi:conditional}
 A \emph{conditional distribution}
 $p(\rx\mid \ry)$ is $\frac{p(\rx,\ry)}{p(\ry)}$.
\end{defi}

\begin{defi}
 \label{def:expectation}
 An \emph{expectation} of a quantity $g(x)$ over $p(\rx)$
 is defined as $\E_{x\sim p(\rx)}g(x)=\E_{p(x)}g(x)=\sum_{x\in X}p(x)g(x)$
 if $\sum_{x\in X}p(x)|g(x)|<\infty$. It does not exist otherwise.
\end{defi}

\begin{defi}
 An \emph{entropy} of $p(\rx)$ is $H(p(\rx))=\E_{p(x)} \brackets{-\log p(x)}$.
 $H(\rx)$ when $p$ is implied.
\end{defi}

Higher entropy means a more random, spread-out distribution.
Entropy \cite{shannon1949synthesis} is an information-theoretic concept:
Imagine receiving a message $x$ from a set $X$ of size $2^N$ with a uniform probability.
The distribution has an entropy $N$ with base 2, or $N$ \emph{bits}, because $-\sum \frac{1}{2^N}\log_2 \frac{1}{2^N}=N$.
To encode the index of $x$ in $X$ as a bitstring, we need one with length $N$.
While symbolic community tends to disregard these concepts as mysterious real numbers,
information-theory connects Computer Science and Statistics.

\begin{defi}
 \label{def:kl}
 A \emph{Kullback-Leibler (KL) divergence} $\KL(q(\rx)||p(\rx))$ is an expectation of log ratio over $q(\rx)$:
 \begin{align}
  \KL(q(\rx)\Mid p(\rx))=\E_{q(x)}\brackets{\log \frac{q(x)}{p(x)}} \geq 0.
 \end{align}
\end{defi}

Equality is satisfied when $q(x)=p(x)$ for all $x$ where $q(x)>0$.
Conceptually it resembles a distance between distributions,
but it is not a distance because it does not satisfy the triangular inequality.
KL divergence is also an information-theoretic concept:
It represents a number of bits additionally necessary to describe $q(\rx)$ based on $p(\rx)$.

An important theorem that appears frequently is Jensen's inequality.
I only provide a special case that is useful in this memo here:
\begin{theo}[Jensen's inequality]
 \label{theo:jenssen}
 For a distribution $p(x)$ and a quantity $g(x)$,
 \begin{align}
  \log \E_{p(x)} \brackets{g(x)} \geq \E_{p(x)} \brackets{\log g(x)}.
 \end{align}
\end{theo}

Bayes' theorem \cite{bayes1763} is a fairly trivial theorem shown from the definition of a conditional distribution.
It is not particularly interesting from a mathematical standpoint (the proof is a simple reformulation),
but it is the core of Bayesian statistics and has a status of being nearly worshiped by the Bayesian school of statisticians.

\begin{theo}[Bayes' theorem \footnote{The original manuscript does not directly show this formula as a theorem. It is a modern interpretation of its essence.}]
\label{theo:bayes}
Given two random variables $A$ and $B$,
\[
 p(A|B) = \frac{p(B|A)p(A)}{p(B)}.
\]
\end{theo}

\begin{defi}
 A $\dbrackets{\text{condition}}$ denotes an indicator function, or sometimes called Kronecker's delta:
\begin{align*}
 \dbrackets{\text{condition}}&=
 \left\{
 \begin{array}{cc}
  1 & \text{if condition is satisfied},\\
  0 & \text{otherwise}.
 \end{array}
 \right.
\end{align*}
\end{defi}

\begin{defi}
 A Dirac's delta $\delta(\rx=c)$, informally speaking, is a
``function'' that represents a pointy, spiking signal.
I do not discuss its theoretical details in this memo. It satisfies
\begin{align*}
 \delta(\rx=c)&=
 \left\{
 \begin{array}{cl}
  \infty& \rx=c,\\
  0 & \text{otherwise},
 \end{array}
 \right.&
 \int_\R \delta(\rx=c)dx &= 1.
\end{align*}
\end{defi}

\section{Applied Concepts in Statistics}
\label{sec:statistics}

\emph{Statistics} is ``a branch of mathematics dealing with the
collection, analysis, interpretation, and presentation of masses of
numerical data'' \cite{merriam-webster}.
It is a tool for scientific study that heavily uses probability theory and combinatorics.
It is not pure math,
as terms are \emph{loaded} with nuances that only make sense in applied settings.
Many statistical concepts like data, interpretation, observation, evidence, ground-truth, priors, posteriors, etc.,
do not exist in pure mathematics, such as the measure theory.
Those notions characterize different roles in applications
played by each probability distribution and each random variable.
The issue with these concepts is that
they are often loosely defined, used informally, or sometimes defined by convention.
This section focuses on this applied aspect of statistics to address the lack of comprehensive formal definitions.

Take the concept of \emph{prior} in Bayes' theorem (\reftheo{theo:bayes}), for example.
Typically, people call $p(A)$ a prior distribution, $p(A|B)$ a posterior distribution,
and $p(B)$ a normalizing constant.
However, there is nothing that syntactically differentiates $p(A)$ from $p(B)$ to tell you that $p(A)$ is a prior;
$p(A)$ is called a prior based on what the variable $A$ represents in an application.
Moreover, these ostensive definitions do not generalize to
a more complex scenario involving multiple random variables.
They lack intensional or extensional definitions from which
we can formally tell, e.g., whether a distribution is a prior or not.

No agreed-upon definition seems to exist.
Contrary to popular belief, Bayes himself did not use these terms in his original manuscript \cite{bayes1763}.
Popular textbooks such as \cite{murphy2012machine}, \cite{gelman1995bayesian}, or \cite[PRML]{bishop2006pattern}
do not have their formal definitions either.
Many articles (including these textbooks) introduce these notions with an informal definition such as
``\emph{a piece of knowledge that a practitioner assumes prior to observing data/evidence}.''
This is merely an \emph{interpretation} of a formal definition, not the definition itself,
because ``knowledge,'' ``prior to,'' ``evidence,'' etc., are not mathematically defined.

The lack of definition seems to be causing unnecessary confusion and debate even within the community.
Recently, some statisticians seem frustrated by an article \cite{van2017neural}
that claims that they have a ``trainable prior,'' citing that a prior should be a fixed distribution.
However, who decided that?
How can one argue over concepts that lack definitions?
I keep asking my fellow colleagues whether they have definitions, and if so which document I should cite.
Their answers tend to be unsatisfactory, for example: ``it is a widely accepted concept,''
``you can't cite them because we have a long history and they are very old,''
``we usually take them for granted and they are usually not the main subject.''
(These are actual answers by highly successful academics from Stat/ML/DL background.)
In contrast,
I can answer propositional logic can be traced back to Aristotle, Plato, Leibniz, DeMorgan, and Boole, and
First Order Logic is by Frege and Peirce, largely thanks to historical notes in \cite{russell1995artificial}.

\subsection{Subjective View of Probability}

To formalize the practical roles of distributions and random variables as mathematical entities,
I first revisit three main interpretations of probabilities.

\begin{conv}[Symmetry, Classical]
 \label{conv:symmetry}
 A ratio of the number of combinations of equally-likely elementary events that satisfy a certain condition
 over the number of all combinations \cite{de1812probability}.
 Classical probability is typically denoted by $\Pr(\ldots)$.
\end{conv}

\begin{conv}[Frequency]
 \label{conv:frequency}
 A ratio of the number of events that satisfied a certain condition,
 over the number of all events observed up until now. \cite{fisher1922mathematical,neyman1933ix,neyman1937outline}
\end{conv}

\begin{conv}[Belief, Subjective, Personal, Epistemic, Bayesian]
 \label{conv:belief}
 A measure of how strongly an agent believes that the next trial satisfies a certain condition
 \cite{von1944theory,savage1954foundations,phanzagl1967}.
\end{conv}

\begin{ex}[Cee-lo]
 The probability of getting three consecutive \epsdice{6}'s (an instant win) by throwing a fair dice three times is $1/6^3=1/216$.
 Imagine you threw a dice 1200 times (400 trials) and got three \epsdice{6}'s twice.
 The frequency is $1/200$.
 You, an optimistic gambler, believe that the next throws will be three \epsdice{6}'s with a probability 0.999.
 That's wishful thinking.
\end{ex}

I adopt a \emph{subjective (belief)} interpretation by defining agents and their beliefs.
In this view, a probability distribution returned by a machine learning system is a belief possessed by the system.
See \refsec{sec:frequentist} for a Frequentist view of machine learning.
The concept of agents and perspectives are typically either missing or implicitly assumed in the literature.

\begin{defi}
 An \emph{agent} is a function $a: \rx \mapsto p^a(\rx)$ that takes a random variable $\rx$
 and returns a probability distribution on it.
 I call $p^a(\rx)$ a distribution of $\rx$ \emph{seen by} $a$,
 or $a$'s distribution, if the meaning is clear from the context.
 Joint and conditional distributions seen by an agent are defined similarly.
\end{defi}

In other words, each agent represents its own beliefs about random variables in the world.
This view clarifies why we can have multiple probability
distributions of the same random variable.
For example,
in statistics, a notion of ``ground-truth distribution'' frequently appears without definition.
This can be seen as a view of God in some monotheistic religions:

\begin{conv}
 Statisticians call a unique special agent $*$ as a \emph{ground-truth}.
 Distributions seen by $*$ are called \emph{ground truth distributions}, and are denoted as, e.g., $p^*(\rx)$.
\end{conv}

\begin{conv}
 Statisticians also assume another special agent $a_{\text{data}}$ as a \emph{data collection agent}
 whose distributions are called \emph{data distributions} or \emph{empirical distributions}.
 Typically\footnotemark[1], it generates distributions by
 obtaining a finite set\footnotemark[2] of \iid samples from the ground-truth distributions,
 and returns a uniform mixture of Dirac's delta distributon on each sample.
\end{conv}

\begin{conv}
 Statisticians sometimes assume a \emph{human agent} $a_{\text{human}}$
 whose distributions are typically discrete.
 Typically, a human agent generates distributions by manual labeling.
 This is common in image classifications, marketing, product reviews, etc.
\end{conv}

\begin{conv}
 Statisticians always assume a \emph{hypothesis agent} $a_{\text{hypo}}$
 represented by a machine learning system,
 which is usually the main subject of the study.
\end{conv}

\footnotetext[1]{
Bayesian approaches do not require this (and thus are said to be better with fewer data),
while Frequentist approaches use it as a theoretical basis.
However, in practice, both approaches assume this, so there is really not much difference.
See appendix \refsec{sec:frequentist}.
}

\footnotetext[2]{
This is also not always the case, for example, when the agent collects new data on demand according to some policy,
as in the context of active learning (which is implicitly used by reinforcement learning, but is not credited well).
}

\subsection{Roles of distributions: Prior, Posterior, etc.}

With this subjective view, I can now formally define the concepts of \emph{prior}, \emph{posterior}, etc.
Existing textbooks do not provide clear-cut classification criteria as shown below.

\begin{defi}
 A \emph{prior} $F(\rx,A)$ on $\rx$ over a set of agents $A$ is a set of possible $p^a(\rx)$, i.e.,
 $F(\rx,A)= \braces{p^a(\rx)\mid a\in A}$.
 In other words, a prior represents a constraint that a certain distribution must satisfy.
\end{defi}

\begin{conv}
 A distribution is a \emph{prior distribution} when
 its prior is singular, i.e., $|F(\rx, A)|=1$.
 \label{conv:prior}
\end{conv}

\begin{ex}
 If you assume $p(\rx)$ satisfies $p(\rx)=\N(0,1)$,
 then $F(\rx,A) = \braces{\N(0,1)}$, thus it is a prior distribution.
\end{ex}

\begin{ex}
 A structural prior, such as a convolutional layer, limits the set of
 distributions that a neural network can represent.
 For example, 1-dimensional convolutional network $f$ used to model a distribution $p(z|x)=\N(f(x),1)$
 has a translation-invariant prior
 $F(\rz|\rx, A) = \braces{ \N(f(x),1) \mid \forall d; f(x)_{i} = f((x_{i-d})_{i=0}^{L})_{i-d} }$.
 \footnote{You can also consider the distribution of weights $p(\theta)$, e.g., $p(z|x)=\sum p(z|x,\theta)p(\theta)$,
 then assume that $p(\theta)=\delta(0)$ outside the convolution, which can be seen as a prior distribution.}
\end{ex}

\begin{ex}
 Conditional independence between variables
 is also a form of priors, because it is a constraint on their joint distribution.
 For example,
 $F(\rz|\rx, A) = \braces{ f \mid \forall \ry\perp\rx; p(\rz\mid\rx,\ry)=p(\rz\mid\rx) }$.
\end{ex}

\begin{conv}
 $a_\text{hypo}$ is called \emph{Bayesian}
 when it has a variable with a singular prior.
\end{conv}

\begin{conv}
 $a_\text{hypo}$ is otherwise called \emph{Frequentist}, i.e.,
 when it has no prior, or the prior is a set of
 all possible distributions $F(\rx,A)= [0,1]^X$ for any variable $\rx\in X$.
 See appendix \refsec{sec:frequentist} for more discussions.
\end{conv}

Next, statisticians attach various adjectives to a distribution
based on what random variable it is about and what random variable it depends on.
These names may overlap and you can combine them:
If a distribution is an X distribution and is also a Y distribution,
you can call it an X Y distribution or sometimes even just an X Y.
These names do not have mathematical significance;
They are simply conventions that are arbitrary and sometimes confusing.

\begin{conv}
 A random variable is
 \emph{observable} when
 $a_{\text{data}}$ has a singular prior for it that you can directly sample from,
 e.g., when $\rx$ follows a uniform distribution over a finite dataset of images.
 It is \emph{labeled} when
 $a_{\text{human}}$ has a singular prior for it.
 It is \emph{latent} otherwise.
\end{conv}

\begin{conv}
 A distribution is a \emph{posterior distribution} when
 it is conditioned on observable variables.
\end{conv}

\begin{conv}
 A distribution is \emph{discriminative}
 if it is of a non-observable (labeled or latent) variable conditioned on observable variables.
 Thus discriminative $\subseteq$ posterior.
\end{conv}

\begin{conv}
 A distribution is \emph{generative}
 if it is of an observable variable conditioned on non-observable variables (e.g., $p(\rx|\ry)$), or
 a joint distribution that includes observable variables (e.g., $p(\rx,\ry)$, and $p(\rx)$).
\end{conv}

\begin{conv}
 If none of above matches,
 a conditional distribution is sometimes called a \emph{model}.
 This concept is redundant because ``conditional distribution'' is enough.
 I do not use this term.
 \label{conv:model}
\end{conv}

\begin{ex}
 When $\rx$ is an image and $\rz$ is a latent,
 $p(\rz)=\N(0,1)$ is a prior distribution,
 $p(\rx\mid\rz)$ is a generative distribution,
 $p(\rz\mid\rx)$ is a discriminative (and posterior) distribution.
 When $\ry$ is a label,
 an image classifier $p(\ry=\text{dog}\mid\rx)$ is a discriminative (and posterior) distribution,
 while a generator $p(\rx\mid\ry=\text{dog})$ of dog pictures is a generative distribution.
\end{ex}

\section{Machine Learning as a Proof System}
\label{sec:ml}

Although \emph{researchers} of ML/Stat/DL have all the rights to explore
messy, ad-hoc, irreproducible, and unjustified methods to perform machine learning on complex tasks,
I do not recommend them for \emph{users} of ML/Stat/DL,
such as symbolic AI researchers not specialized or interested in the learning mechanism itself.
If you review the history of machine learning methods,
it is apparent that those unjustified methods are mere products of immature theoretical understanding and
are eventually superseded by ones with clear theoretical justifications.
Autoencoders (AEs) vs.\ Variational Autoencoders (VAEs, \refsec{sec:variational}),
or GANs vs.\ VEEGAN \refsec{sec:likelihood-free}, are such examples:
The justified methods have a better guarantee, performance, quality, and characteristics.
To us (non-specialists), immature methods waste our time on inessential parts of the hypothesis we want to show.

This section draws your attention to a formal definition of machine learning and its characteristics.
The definition derives modern algorithms regardless of supervised or unsupervised learning,
including
variational inference (e.g., VAE) and density-ratio estimation (e.g., GAN).
An important characteristics of this framework is
its ability to discuss its \emph{soundness} and \emph{completeness} in the classical proof systems sense
by seeing each learned result as a proof.
Whether a machine learning method is derived from this formulation
roughly tells whether the method is worth consideration for non-specialists ML/Stat/DL users.

\subsection{What is Machine Learning?}

Let $p^*(\rx)$ be the ground-truth distribution of an observable random variable(s) $\rx$,
and $p(\rx)$ be its current estimate.
Given a dataset $\X$ of $\rx$, whose elements $x_i$ are indexed by $i$,
let me denote a data distribution as $q(\rx)$, which draws samples from $\X$ uniformly.
$q(\rx)$, $p(\rx)$, $p^*(\rx)$ are completely different from each other.
In this section, $p(\rx)$ is a purely mathematical entity
with no particular implementation --- It has an unlimited capacity and can represent any distribution function.

\begin{conv}
 \label{conv:data}
 A dataset (empirical, data) distribution $q(x)$ is typically defined as follows
 (Sometimes also as $\pdata(x)$).
 \begin{align}
 q(x)   &= \sum_i q(x|i)q(i), \\
 q(x|i) &= \delta(\rx=x_i), \quad (\text{Dirac's $\delta$, i.e., a ``point''})\label{eq:data-delta}\\
 q(i)&= \frac{1}{|\X|}. \quad (\text{uniform over}\ 0\leq i < |\X|)
 \end{align}
\end{conv}

Machine Learning is a problem
of finding $p(x)$ that makes the dataset $\X$ most likely.
This idea is formalized as follows:
\begin{defi}
\emph{Machine Learning} (ML) is a task of maximizing
the expectation of $p(x)$ among $q(x)$.
\begin{align}
 \hat{p}^*(\rx) = \argmax_{p(\rx)} \E_{q(x)} p(x).
\end{align}
 \label{def:ml}
\end{defi}

\begin{conv}
 In practice, we typically minimize a \emph{loss function}, or a \emph{negative log likelihood} (NLL) $-\log p(x)$,
because $-\log$ is monotonic and preserves the optima.
 \label{conv:nll}
\end{conv}

\begin{fact}
 $\hat{p}^*(\rx) \not= p^*(\rx)$.
\end{fact}

\begin{theo}
 \label{theo:ml}
 Actually, $\hat{p}^*(\rx)=q(\rx)$ (perfect overfitting).
\end{theo}

\begin{proof}
 \begin{align}
  0 &\leq \KL(q(\rx)\Mid p(\rx))
  =\E_{q(x)} \log \frac{q(x)}{p(x)}\\
  &=-H(q(\rx))  + \E_{q(x)}\brackets{-\log p(x)}\\
  &= \text{Const.} + \E_{q(x)}\brackets{-\log p(x)}.\label{theo:ml-proof}
 \end{align}
 The first term is a constant because $q(x)$ is a constant function.
 Note that $\KL(q(\rx)\Mid p(\rx))= 0$ if and only if $q(\rx)=p(\rx)$.
 Thus, minimizing the NLL $\E_{q(x)} \brackets{-\log p(x)}$ minimizes $\KL$ and achieves $q(\rx)=p(\rx)$.
\end{proof}

\begin{corollary}
  If $q(x)=p^*(x)$, i.e., if we have a perfect dataset, ML indeed achieves the ground truth.
\end{corollary}

The proof above also suggests that ML is equivalent to
minimizing the KL divergence between $p(x)$ and $q(x)$ up to a constant $H(q(x))$,
which provides another intuitive explanation: It makes the estimate closer to the empirical distribution.

\begin{theo}
\refdef{def:ml} is equivalent to a task of minimizing the KL divergence between $p(x)$ and $q(x)$.
\begin{align}
 \hat{p}^*(\rx) = \argmin_{p(\rx)} \KL(q(x)\Mid p(x)).
\end{align}
\end{theo}

\paragraph{Further notes:}

Typically, we assume $\hat{p}^*(\rx)$ and $p(\rx)$ are of the same family of functions
parameterized by $\theta$ such as neural network weights,
i.e., $\hat{p}^*(\rx)=p_{\theta^*}(\rx)$ and $p(\rx)=p_{\theta}(\rx)$.
Depending on how we treat $\theta$,
machine learning can be further classified into \emph{Frequentist, Partial Bayesian, or Fully Bayesian} approaches.
Frequentist and  Partial Bayesian approaches use \emph{Maximum Likelihood Estimation (MLE)}.
See \refsec{sec:parameters} for more details on learned parameters and MLE.

\subsection{Optimal Solution to ML is Sound}
\label{sec:ml-is-sound}

The Symbolic AI community values a system's logical correctness to a great degree.
Probably the most common reason they avoid machine learning is the worry that
the system could produce wrong results.
To address this worry, I attempt to demonstrate an important implication of ML that,
if $p(\rx)$ converges to the optimum $\hat{p}^*(\rx)$,
the system never generates/predicts data $\rx$ (image visualizations, scalar or categorical predictions, or anything)
that are \emph{invalid/unreal}.
Under a certain definition below,
I propose to refer to this property of ML as the \emph{soundness} of ML.

\renewcommand{\valid}[1][X]{#1^{\checkmark}}
\newcommand{\invalid}[1][X]{#1^{\times}}

Assume the sample space $X$ of $\rx$ can be divided into
a set of valid and invalid data points $\valid$ and $\invalid$, i.e.,
\[
 \invalid = \braces{x\in X\mid p^*(x)=0}.\quad \valid = X \setminus \invalid.
\]

Statisticians may call the assumption unusual, claiming that, e.g., for an image taken by a digital camera,
any sensor noise or a cosmic ray anomaly can theoretically produce any possible value of an image array,
therefore any data point has an infinitesimal but still non-zero density.
To avoid such an issue, let's
assume $X$ is discrete.

Furthermore, I
also ignore the probability differences between the valid examples.
For example, given two valid data $x_1$ and $x_2$,
the former may be more likely ($p^*(\rx=x_1)>p^*(\rx=x_2)$)
but the model may say the otherwise ($p(\rx=x_1)<p(\rx=x_2)$).
There may also be a difference from the ground truth $p(\rx=x_1)\not= p^*(\rx=x_1)$.
To discuss a topic such as the speed of convergence to the optimum,
a more in-depth theoretical discussion is necessary, which is out of the scope of this memo.
I ignore such a difference as long as they are correctly determined as \emph{possible} ($p(\rx=x_1)>0, p(\rx=x_2)>0$),
focusing only on the validity of the samples generated from $p(x)$.

Although this setting would be unusual for statisticians,
this is a fairly reasonable, realistic, and practical scenario in the symbolic community.
In non-deterministic reasoning (rather than probabilistic reasoning),
the probability distribution of certain outcomes is not available,
but only a \emph{list} of possible outcomes is available
(e.g., FOND planning \cite{cimatti2003weak,muise2015leveraging}).
In many such applications, the goal is not to find a policy with which success is most likely (weak solution)
but to find a policy that \emph{always} succeeds even in the least-likely scenario (strong/strong cyclic solution),
which thus \emph{should not} consider the probability distributions.

Note that the dataset $\X\subseteq \valid$ represented by $q(\rx)$ contains only valid examples
because the data are \emph{indeed} observed in the real world, therefore, cannot be invalid.
Invalid data are invalid precisely because they are irreplicable in the real world.
Conversely, the system will never observe invalid data in $\invalid$.
Also, $\valid\setminus \X$ represents valid but unseen data.

We can see a probability distribution as a proof system.
Let's revisit the concept of soundness and completeness in a classical proof system:
\begin{defi}
 A proof system is \emph{sound} if everything that is provable is in fact true.
\end{defi}
\begin{defi}
 A proof system is \emph{complete} if everything that is true has a proof.
\end{defi}
\begin{defi}
 We say $p(\rx)$ \emph{proves} $x\in \valid$ when $p(x)> 0$.
\end{defi}

\begin{theo}
 An optima $\hat{p}^*(x)$ of ML is \emph{sound}, i.e.,
\begin{align*}
 \hat{p}^*(x)> 0 &\then x\in \valid. &(\iff&& x\in \invalid &\then \hat{p}^*(x)=0.)
\end{align*}
\end{theo}
\begin{theo}
 $\hat{p}^*(x)$ can be \emph{incomplete}, i.e.,
\begin{align*}
 x\in \valid &\not\then \hat{p}^*(x)> 0. &(\iff&& \hat{p}^*(x)=0 &\not\then x\in \invalid.)
\end{align*}
\end{theo}

\begin{proof}
 Trivial, because $\hat{p}^*(x)=q(x)$ (perfect overfitting).
 Each statement follows naturally from $q(x)$ (\refconv{conv:data}).
\end{proof}

However, this first proof does not convey the full extent of the surprise.
To fully embrace it, I need another proof:

\begin{proof}
ML achieves the soundness by maximizing $p(\rx)$ for real data $q(\rx)$,
which reduces $p(\rx)$ for invalid data \emph{that it has not even seen}
because a probability distribution sums/integrates to 1: $\sum_\rx p(\rx)=1$.
If there is still an invalid point that has a positive mass,
you can move the mass to valid points and further maximize $p(\rx)$.
See \refig{fig:soundness} for the illustration.

Let $\hat{p}^*(x)>0$ for some $x \in \invalid$. We define a new distribution $p'(\rx)$
by moving all probability mass assigned to $\invalid$ to $\valid$.
Let $C = \sum_{x\in \invalid} \hat{p}^*(x)$, i.e., the mass assigned to $\invalid$.
Obviously $0\leq C \leq 1=\sum_{x\in X} \hat{p}^*(x)$.
Then we can achieve the desired effect by scaling the distribution:
\begin{align*}
 p'(x) &=
 \left\{
 \begin{array}{cl}
  \hat{p}^*(x) / \parens{1-C}  & x \in \valid\\[0.5em]
  0 & x \in \invalid
 \end{array}
 \right.\\
 \E_{q(x)} \hat{p}^*(x) &\leq \E_{q(x)} p'(x)=\frac{1}{1-C} \E_{q(x)} \hat{p}^*(x).
\end{align*}
which contradicts that $\hat{p}^*(x)$ is maximized.
\end{proof}

\subsection{Generalization Makes ML Complete}

The incompleteness was caused by the infinite capacity in $p(x)$ that can express any function.
It can perfectly overfit the data $\X$ by assigning 0 to everything not in $\X$, \emph{including the valid ones}.
This makes the model susceptible to \emph{out-of-distribution} examples and generates wrong predictions.
To address this, one should limit the expressivity of $p(x)$
so that it generalizes beyond $\X$, i.e.,
to start assigning non-zero to unseen valid examples $\valid\setminus \X$
while keep assigning 0 to invalid examples $\invalid$.
See \refig{fig:completeness} for the illustration.

\begin{defi}
 For a set of distributions $F$,
 let $C(x)$ be an equivalence class of $x$ under $F$, i.e.,
\[
 C(x)=\braces{x'\in X \mid \forall f\in F; f(x')=f(x)}.
\]
\end{defi}
\begin{ex}
 Convolutional layers model translation invariant distributions $F$ and
 \emph{cannot discern the translated inputs}.
 $C(x)$ has horizontally/vertically shifted $x$.
\end{ex}
\begin{ex}
 Transformer \cite{vaswani2017attention} model permutation invariant distributions $F$ and
 \emph{cannot discern the permuted sequence}.
 $C(x)$ has all permutations of $x$.
\end{ex}
\begin{defi}
 Let $Y=\braces{C(x)\mid x\in X}$.
 Define $\Y$, $\valid[Y]$, and $\invalid[Y]$ similarly.
 I say $F$ \emph{generalizes from $\X$ to $\valid$} when
 $\Y=\valid[Y]$, i.e.,
 the equivalence classes of $\X$ covers $\valid$.
\end{defi}

\begin{lemma}
 If $\X=\valid$, then $\hat{p}^*(\rx)$ is complete.
\end{lemma}

\begin{theo}
 Suppose $F$ generalizes from $\X$ to $\valid$ and $p^*(\rx)\in F$.
 Suppose no two data points in $\X$ maps to the same class.
 Then $\hat{p}^*_F(x)$, the optima under $F$, is complete:
 \[
 \hat{p}^*_F(x) = \argmax_{p(\rx)\in F} \E_{q(x)} p(x).
 \]
 \label{theo:complete}
\end{theo}
\begin{proof}
 Let $Y=\braces{C(x)\mid x\in X}$.
 Define $\Y$, $\valid[Y]$, and $\invalid[Y]$ similarly.
 The optima $\hat{p}^*(\ry)$ on $\ry\in Y$ using $\Y$ is complete because $\Y=\valid[Y]$.
 Since
 $\hat{p}^*_F(x)=\frac{\hat{p}^*(\ry=C(x))}{|C(x)|}$ by assumption,
 $\hat{p}^*_F(x)$ is also complete.
\end{proof}

\begin{figure}[tb]
\centering
\includegraphics[width=\linewidth]{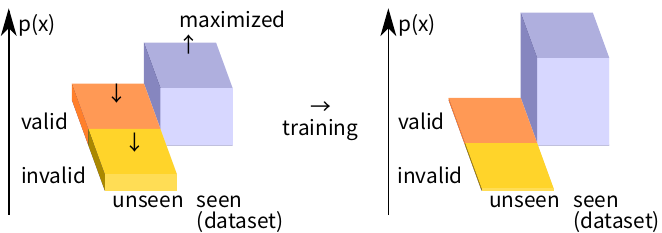}
\includegraphics[width=\linewidth]{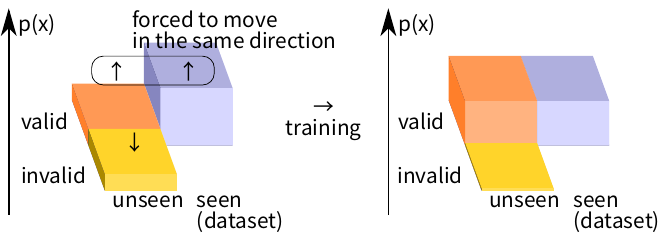}
\caption{
(Top)
An illustration of maximizing $p(x)$ from data. Without generalization, the result is a perfect overfitting,
which results in a sound but incomplete model.
(Bottom)
Maximizing $p(x)$ from data with constraints
that force all valid examples to move in the same direction.
The generalization achieves a sound and complete model.
 }
\label{fig:soundness}
\label{fig:completeness}
\end{figure}

Generalization improves the \emph{sample efficiency},
i.e., you can learn from fewer data, and you do not need a perfect dataset.
Instead, you only need a single instance from each class $C(x)$.

\subsection{In Practice...}

To summarize, informally, there are three conditions for the ground truth to be approximated well:
\begin{enumerate}
 \item $q(x)$ is good.
 \item $p(x)$ is expressive enough to be sound.
 \item $p(x)$ is restricted enough to be complete.
\end{enumerate}

In practice, there are number of reasons that a trained model is unsound and/or incomplete:
insufficient data ($F$ not generalizing from $\X$ to $\valid$),
suboptimal solutions (e.g., early stop),
insufficient generalization (assigning zero to $\valid\setminus \X$),
or over-generalization (assigning non-zero to $\invalid$).

If everything breaks down in practice, why should we care?
It is because some approaches are \emph{unsound even in this idealized optima}.
This soundness of ML is weak and idealistic, but it is still better than nothing because
it significantly prunes the \emph{design space}.
Non-specialist users of Stat/ML/DL in the symbolic AI community,
likely already bamboozled by the mess in the machine learning literature,
be advised: \textbf{Ignore unsound approaches}.

Note that
any existing approach could be shown to become a sound ML with a minor modification.
For example, although previous work on classical learning schemes such as MAXSAT-based learner \cite{YangWJ07}
has not been analyzed in this way, it may turn out to be sound and complete.

\paragraph{Further notes:}

My analysis focuses on the \emph{support} of the density/mass functions,
i.e., its non-zero regions.
In 1-dimensional settings, the edges of the support are the \emph{extrema} (e.g., minimum) of the random variable.
While the mainstream statistics deals with the \emph{means} based on Central Limit Theorem (\reftheo{theo:clt}),
extrema are dealt by \emph{Extreme Value Statistics} based on Extremal Limit Theorem (\refsec{zoo:gumbel}).

Averages are useful, but extrema deserve more attention.
While the mainstream ML focuses on the \emph{most likely} behavior,
real-world safety-critical applications must know the model's highly \emph{unlikely} limit behaviors.
It even makes sense in creative applications like text-to-image models \cite[DALL-E]{dalle,dalle2}:
A novel art emerges from an exaggeration toward the extremes, not from regression to the incompetent norms.
As another example,
we are not only interested in the average travel time to the office,
but also in the worst case (to join a meeting)
and the best case (to know how good my route is; to take the risk to improve the plan).
Distribution Zoo (\refsec{sec:zoo}) covers more details on this topic.

Recently, \emph{Contrastive Learning} has seen great empirical success
and has attracted theoretical attention.
Its theoretical justification is provided by Noise Contrastive Estimation \cite{gutmann2010noise}:
It approximately generates $\invalid$ to actively minimize $p(x)$ for $x\in\invalid$ (\refig{fig:contrastive}).
Examples include
\emph{contrastive loss} \cite{chopra2005learning} in face verification,
\emph{negative sampling} \cite{mikolov2013distributed} in natural language processing,
and \emph{PU-learning} \cite{elkan2008learning}, which learns from a positive and an unlabeled dataset.

VC-dimensions, PAC-learnability of a concept class, Central Limit Theorem, etc.,
analyze more general continuous cases (they are also Frequentist \refsec{sec:frequentist}).

\begin{figure}[tbp]
\centering
\includegraphics[width=\linewidth]{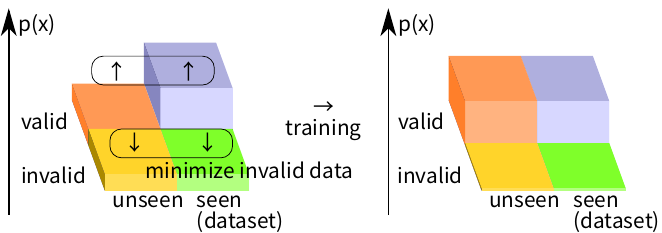}
\caption{
Contrastive learning / Noise Contrastive Estimation \cite{gutmann2010noise}
maximizes $p(x)$ for valid data and minimizes $p(x)$ for invalid data.
The minimization is explicit, unlike non-contrastive learning.
}
\label{fig:contrastive}
\end{figure}

\subsection{Instances of ML}

ML is a general framework applicable to various tasks.

\begin{ex}[Supervised Learning]
 \label{def:supervised}
 Assume an input variable $\rx$ and an output variable $\ry$.
 The dataset $\X=(x_i)_{i=0}^{N}$ and $\Y=(y_i)_{i=0}^{N}$
 represents $N$ input-output pairs.
 \begin{align*}
 q(x,y) &= \sum_i q(x,y|i)q(i)= \sum_i q(x,y|i)\frac{1}{N}, \\
 q(x,y|i) &= q(x|i)q(y|i)=\delta(\rx=x_i)\delta(\ry=y_i),\\
  \hat{p}^*(x,y) &= \argmax_p \E_{q(x,y)} p(x,y)\\
  \propto \hat{p}^*(y|x) &= \argmax_p \E_{q(x,y)} p(y|x).\quad (p(x)=\text{Const.})
 \end{align*}
\end{ex}

\begin{ex}[Classification/Regression]
 A supervised learning task is called a \emph{classification/regression} if
 the output variable is discrete/continuous, respectively.
 When the output is discrete and has $C$ categories,
 \begin{align*}
  q(y|i) &= \cat(\ldots, 0,1,0, \ldots)=\dbrackets{\ry=y_i}\\
 \E_{q(x,y|i)} \log p(y|x) &= \sum_{j=0}^C \dbrackets{y_i = j} \log p(y = j|x).
 \end{align*}
\end{ex}

In other words, $\dbrackets{y_i = j} = 1$ if $j$ is the correct answer
and $\dbrackets{y_i = j} = 0$ otherwise.
Notice that this is a definition of cross entropy for a categorical variable.
A binary classification task is a special case with $C=2$.

\section{Loss Functions: Do the Right Thing}
\label{sec:loss}

You may have read somewhere that \dl is just a glorified square fitting.
It is true that square errors are abundant in \dl, but
why so many methods use them and how do they justify it?
Why they also sometimes use absolute errors?
So far, I have been discussing $\argmax_{p(x)}\E_{q(x)} p(x)$ or $\argmin_{p(x)}\E_{q(x)} - \log p(x)$.
But what are these loss functions, anyways?

\begin{fact}
 The actual form of the loss function is defined by the choice of the distribution.
 \label{fact:loss}
\end{fact}

For example,
a model designer can assume that $\rx$ follows a specific distribution such as
a Gaussian distribution:
\begin{align}
 x\sim p(x)=\N(x\mid \mu,\sigma)=\frac{1}{\sqrt{2\pi\sigma^2}}e^{-\frac{(x-\mu)^2}{2\sigma^2}}. \label{eq:gaussian-pdf}
\end{align}
A machine learning system predicts the value of $\mu$ and $\sigma$,
in which case the NLL (\refconv{conv:nll}) is a squared error of prediction $\mu$ shifted and scaled using $\sigma$:
\begin{align}
 -\log p(x)=\frac{(x-\mu)^2}{2\sigma^2}+\log \sqrt{2\pi\sigma^2}. \label{eq:gaussian-nll}
\end{align}
As another example,
the loss function for a Laplace distribution $\frac{1}{2b}e^{-\frac{|x-\mu|}{b}}$
is a shifted and scaled absolute error $\frac{|x-\mu|}{b} +\log 2b$.

Now the reader may have many questions:
Why the Gaussian distribution is the typical choice?
How can it be theoretically justified?
When and why we should use Laplace distribution, or any other distribution?
These are answered by the \emph{Maximum Entropy Principle} \cite{maxent,jaynes1968prior}:
It is because Gaussian distribution is the \emph{maximum entropy distribution}
among all distributions with range $[-\infty,\infty]=\R$ with the same mean and the variance.

\begin{defi}
 The \emph{maximum entropy distribution} $f^*$ among a set of distributions $F$ is the one with the largest entropy
 $f^*=\argmax_{f\in F} H(f)$.
 In other words, it is ``most random'' in $F$, thus has the least ``unintended'' assumptions among $F$.
\end{defi}

\begin{theo}[Maximum Entropy Principle]
 \label{theo:maxent}
 The optimal distribution for a random variable
 is the maximum entropy distribution among distributions that
 satisfy the user-supplied constraint (domain knowledge). \cite{maxent,jaynes1968prior}
\end{theo}

\begin{theo}
 Gaussian distribution $\N(\mu,\sigma)$ is the maximum entropy distribution $p(x)$ for
 $x\in\R$ with a finite mean $\mu=\E_{p(x)}\brackets{x}$ and a finite variance $\sigma^2=\E_{p(x)}\brackets{(x-\mu)^2}$.
\end{theo}

\begin{fact}
 ML applications often lack the domain knowledge on a variable
 other than it has a finite mean and variance.
 Thus they use Gaussians = square errors.
\end{fact}

\begin{fact}
 Practitioners \textbf{must} choose the right distribution family based on the domain knowledge.
 Choose it wisely because it decides the loss function.
 \textbf{Don't do random hacks.}
\end{fact}

\begin{theo}
 Laplace distribution $L(\mu,b)$ is the maximum entropy distribution $p(x)$ for
 $x\in\R$ with a finite mean $\mu=\E_{p(x)}\brackets{x}$ and a finite $b=\E_{p(x)}\brackets{|X-\mu|}$. \cite{kotz2001laplace}
\end{theo}

\begin{ex}
Use absolute errors = Laplace distribution \textbf{if and only if}
the model designer can expect anomalies in the dataset but
a finite $\E_{p(x)}\brackets{|X-\mu|}$ exist.
The resulting loss function (absolute errors) has a less steep loss curve
that makes the training robust to anomalous inputs.
\end{ex}

\begin{ex}
Sometimes even a finite mean doesn't exist.
Consider Cauchy distribution $C(x_0,\gamma)$ with a \emph{median} $x_0$.
\end{ex}

Each maximum entropy distribution is specific to each assumption.
For example, the maximum entropy distribution for positive reals is Gamma distribution $\Gamma(k,\theta)$.
In other words, if you assume a variable to be positive, you should not use a Gaussian.
Distribution zoo (\refsec{sec:zoo}) contains a list of maximum entropy distributions.

I personally have many objections against the current usage of statistical modeling
in the symbolic AI community / planning and scheduling community
mainly due to the violation of this established principle.
However, I would like to leave this topic for another occasion.

\subsection{Point Estimate and Mean Square Errors}
\label{sec:point}

The NLL of a Gaussian (\refeq{eq:gaussian-nll}) is already close to the mean square error that you may have seen often,
but it is still different from just a square error:
It has a scale $1/2\sigma^2$ and an offset $\log \sqrt{2\pi\sigma^2}$.
Why don't people use the NLL?
Is NLL better or is mean square error better?

\begin{fact}
 The correct characterization is NLL.
 Square error is a hack/simplification derived from NLL.
 (But see the note at the end of this section for alternative explanations.)
\end{fact}

\begin{fact}
 Practitioners often don't bother with the variance.
 Thus they set $\sigma$ to an arbitrary constant and omit it from the loss function,
 resulting in a square error $(x-\mu)^2$.
 By averaging the NLL over $q(i)=1/|\X|$, we obtain a mean square error.
\end{fact}

In many machine learning applications,
there is often no need to predict the variance.
A trained model returns a single most-likely value
rather than a distribution over possible values.
The value returned by such a model is called a \emph{point estimate}:
When we model the output distribution as a Gaussian $\N(x\mid \mu,\sigma)$,
we predict $\mu$, the point where the probability is the largest (\emph{mode}).

Given a distribution, a point estimate can use any of the statistics,
including the mean, the median, the mode, or even a certain top quantile.
Mean/median/mode are identical in Gaussian distributions,
but this is not always the case with other distributions.

\begin{conv}
 A machine learning model is performing a \emph{point estimation}
 if it returns a single representative value (statistic) of a distribution
 instead of the distribution itself.
\end{conv}

\begin{conv}
 \emph{Maximum A-Posteriori} (MAP) estimate is a point estimate using the \emph{mode}.
\end{conv}

\begin{ex}
 The $\mu$ of a Gaussian is a point estimate.
\end{ex}

\begin{ex}
 The $\mu$ of a Laplace is a point estimate.
\end{ex}

\begin{ex}
 The $\mu$ of a Gaussian is a MAP estimate because the mean and the mode of a Gaussian are the same.
\end{ex}

\begin{ex}
 The top 95\% quantile of a Gaussian is a point estimate but is not a MAP estimate.
\end{ex}

Finally,
we can obtain another explanation from Hanlon's razor (never attribute malice to incompetence):
Many ML practitioners are simply not specialized in statistics,
thus are cargo-culting the statisticians who use $(x-\mu)^2$ without understanding the details.
This is also not helped by the fact that many ML textbooks
(e.g., cheap textbooks with titles like ``Machine Learning 101 using Excel'')
use square fitting as the first material to try,
without explaining its theoretical background.
\textbf{Do not fall into this trap.}

\paragraph{Further notes:}

While Frequentist approaches may appear more generous about the choice of loss functions,
only a subset of methods and losses have proven theoretical guarantees (PAC),
which is discussed in \refsec{sec:frequentist}.

In a \emph{distributional estimation} of Gaussians,
the model predicts two values $\mu$ and $\sigma$.
They are simultaneously optimized using the NLL without omitting $\sigma$.
This is useful for quantifying the \emph{uncertainty}
the model has on its own prediction \cite{kendall2017uncertainties}.
See a longer discussion on the uncertainty in \refsec{sec:uncertainty}.

\section{Generative / Statistical Modeling}
\label{sec:modeling}

Modern machine learning tasks often involve tasks beyond a simple prediction.
Such tasks, e.g., action model learning, image generation, multi-modal transfer, reinforcement learning, etc.,
require multiple interdependent latent variables.
With latent variables, things are not as straightforward as before.
However,
few authors of \dl literature attempt to justify their training schemes
with theoretical or statistical clarity.
This often results in an unreliable, irreproducible system that
requires heavy hyperparameter tuning and ad-hoc loss functions.
Finally, the lack of consistent \emph{procedure} for constructing a \dl system
resulted in a common criticism that its development is like \emph{alchemy}.

In order to make \dl less of alchemy,
this section provides a simple, principled guide to building a complex but statistically justified system yourself.
Keep in mind that unsound methods are theoretically fragile or incorrect
because they lack the soundness (\refsec{sec:ml-is-sound}).
\emph{Avoid ad-hoc hacks that make no sense!}
\footnote{
An irony is that even such an ad-hoc method often happens to work empirically
due to the extreme flexibility of neural networks,
the \emph{best-effort} nature of the task where correctness is less important, and
the culture of \emph{cherry-picking}.
}

\begin{conv}
 \emph{Statistical Modeling} is a general scientific procedure
 which roughly consists of the following steps \cite[section 1.1]{gelman1995bayesian}:
 \begin{enumerate}
  \item List observable (and labeled) variables.
        \label{line:statistical-modeling-step1}
  \item Hypothesize a list of latent variables for the mechanism that you believe to be behind the observations.
        \label{line:statistical-modeling-step2}
  \item Hypothesize the causal dependencies between the variables to specify the mechanism,
        and factorize the generative distribution based on the dependency.
        \label{line:statistical-modeling-step3}
  \item Hypothesize what distribution each variable should follow, including the priors.
        This is done as follows:
        \begin{enumerate}
         \item First, choose the distribution family based on Maximum Entropy Principle (\reftheo{theo:maxent}),
               e.g., $\N$.
               To find the right one for your case, consult Distribution Zoo (\refsec{sec:zoo}).
         \item Second, choose the parameters, e.g., $\mu$, $\sigma$ of $\N(\mu,\sigma)$.
               For conditional distributions, they are often
               outputs of trainable functions that take dependent variables,
               e.g., using $p(\rx|\rz)=\N(\mu=f(\rz),\sigma)$ where $f$ is a neural network.
               For distributions without dependent variables, assign constants (=prior distribution).
        \end{enumerate}
        \label{line:statistical-modeling-step4}
  \item Using data, test the hypothesis, i.e., your hypothetical latent mechanism,
        by training a system with a sound method and evaluate the result with test data.
        How to perform it efficiently is beyond the scope of this section.
        See \refsec{sec:variational} and \refsec{sec:likelihood-free}.
        \label{line:statistical-modeling-step5}
 \end{enumerate}
 \label{conv:statistical-modeling}
\end{conv}

The focus is on the first 4 items, which provide a \emph{specification} for the mechanism.
The dependencies (item 3) describe the structure of the mechanism,
and the distributions (item 4) describe the nature of the structure,
e.g., categorical with $\cat$, boolean with $\bern$, continuous $\R$ with $\N$,
continuous positive accumulation with $\Gamma$, and so on.
Consult Distribution Zoo (\refsec{sec:zoo}) for this choice.

Readers familiar with mathematical modeling (e.g., SAT, MAXSAT, MILP, CSP, SMT, ASP)
would easily see the similarity between statistical modeling and those paradigms.
Both first define a list of variables with their types, then define constraints over the variables.

\begin{conv}
 A \emph{statistical model} refers to a set of statements/assumptions made in item 1-4.
 The term ``model'' here is more than what is implied in Convention \ref{conv:model}.
\end{conv}

\begin{conv}
 If a statistical model mainly concerns with
 a generative distribution, it is called a \emph{generative model}.
\end{conv}

\begin{conv}
 Dependencies between variables defined in step 3 can be seen as a graph $G=(V,E)$
 whose nodes $V$ are variables and edges $E$ are dependencies.
 If such a graph is shown, it is often called a \emph{graphical model},
 a \emph{probabilistic graphical model} (PGM), or
 a \emph{structured probabilistic model}.
\end{conv}

\begin{conv}
 The graph typically forms a directed acyclic graph (DAG).
 Such a model is called a \emph{Bayesian network} or a \emph{directed graphical model}.
\end{conv}

\begin{conv}
 In a graphical model,
 stochastic variables are shown in circles;
 deterministic variables in squares;
 repetitions in plates;
 observable variables in gray nodes; and
 latent variables in white nodes.
\end{conv}

\begin{ex}
 \label{ex:vae}
 Variational AutoEncoder \cite[VAE]{kingma2014semi} is a simple graphical model (\refig{fig:graphical-model-vae}).
 The goal of training a VAE is to obtain a compact latent representation of images.
 Following the statistical modeling,
 \begin{enumerate}
  \item Let $\rx$ be an image.
  \item Let $\rz$ be a latent vector.
  \item Assume that $\rx$ depends only on $\rz$.
        Thus the generative distribution $p(\rx)$ is factored into:
        \[
         p(\rx)=\sum_\rz p(\rx,\rz)=\sum_\rz p(\rx|\rz)p(\rz).
        \]
  \item Assign $p(\rz)=\N(0,1)$, $p(\rx|\rz)=\N(f(\rz),\sigma)$, where $f$ is a decoder neural network and $\sigma$ is arbitrary.
 \end{enumerate}
\end{ex}

\begin{ex}
 Hidden Markov Model \cite{juang1991hidden} is a classic statistical model (\refig{fig:graphical-model-hmm})
 often used for speech modeling.
 It assumes that each latent state depends on the previous latent state.
 In this example, I depict only a single step, but it is originally unrolled for a sequence.
 Following the statistical modeling,
 \begin{enumerate}
  \item Let $\rx^0$ and $\rx^1$ be a pair of observations of the predecessor and the successor states (e.g., speech data).
  \item Let $\rz^0$ and $\rz^1$ represent their respective latent states.
  \item We assume that $\rx^0$ depends only on $\rz^0$,
        $\rx^1$ depends only on $\rz^1$,
        and $\rz^1$ depends only on $\rz^0$.
        Thus the generative distribution $p(\rx^0,\rx^1)$ is factored into:
        \[
        p(\rx^0,\rx^1)=\sum_{\rz^0,\rz^1} p(\rx^0|\rz^0)p(\rx^1|\rz^1)p(\rz^1|\rz^0)p(\rz^0).
        \]
  \item Assign
        $p(\rz^0)=\N(0,1)$,
        $\forall t\in \braces{0,1}; p(\rx^t|\rz^t)=\N(f_1(\rz^t),\sigma)$,
        $p(\rz^1|\rz^0)=\N(f_2(\rz^0),f_3(\rz^0))$,
        where $f_1, f_2, f_3$ are neural networks and $\sigma$ is arbitrary.
 \end{enumerate}
\end{ex}

\begin{ex}
 Latplan \cite{Asai2022} learns discrete latent states and latent actions from images
 (\refig{fig:graphical-model-latplan}).
 In addition to HMMs, it has a latent variable of actions that affect $\rz^1$.
 \begin{enumerate}
  \item Let $\rx^0$ and $\rx^1$ be a pair of images.
  \item Let $\rz^0$ and $\rz^1$ represent their respective latent states.
        Let $\ra$ represent an action.
  \item We assume that
        $\rx^0$ depends only on $\rz^0$,
        $\rx^1$ depends only on $\rz^1$,
        $\rz^1$ depends on $\rz^0$ and $\ra$ (action affects the states),
        and $\ra$ depends on $\rz^0$ (due to preconditions, $\rz^0$ affects which action is possible).
        Thus the generative distribution $p(\rx^0,\rx^1)$ is factored into:
        \[
        p(\rx^0,\rx^1)=\sum_{\rz^0,\rz^1,\ra} p(\rx^0|\rz^0)p(\rx^1|\rz^1)p(\rz^1|\rz^0,\ra)p(\ra|\rz^0)p(\rz^0).
        \]
  \item (Omitted: beyond the scope of this section.)
 \end{enumerate}
\end{ex}

\begin{figure*}[tb]
 \centering
 \hfill
 \begin{subfigure}{0.3\textwidth}
  \centering
  \includegraphics[height=4em]{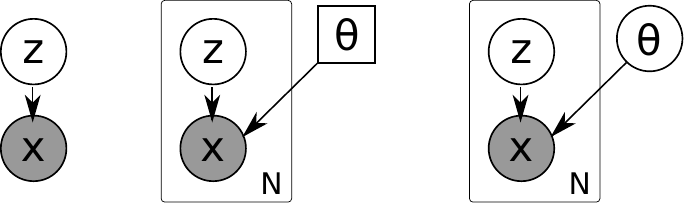}
  \caption{Variational Autoencoder.}
  \label{fig:graphical-model-vae}
 \end{subfigure}
 \hfill
 \begin{subfigure}{0.3\textwidth}
  \centering
  \includegraphics[height=5em]{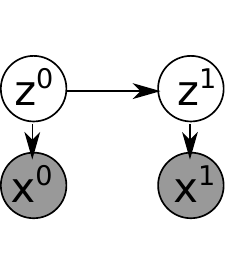}
  \caption{Hidden Markov Model (single step).}
 \label{fig:graphical-model-hmm}
 \end{subfigure}
 \hfill
 \begin{subfigure}{0.3\textwidth}
  \centering
  \includegraphics[height=4em]{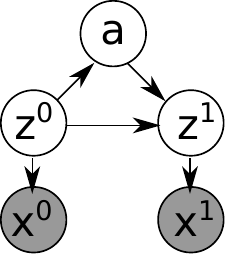}
  \caption{A latent dynamics model in Latplan.}
 \label{fig:graphical-model-latplan}
 \end{subfigure}
 \hfill
 \caption{
 \textbf{(a)}
 An observed variable $\rx$ and a latent variable $\rz$ are in a gray and a white node.
 In the center, you see a \emph{plate} notation, which indicates that we sample $\rx$ from $\rz$ independently $N$ times.
 You also see a parameter $\theta$ in $p_\theta(\rz|\rx)$ which is fixed over multiple sampling.
 $\theta$ is in a square node because it is a deterministic parameter stored in a neural network decoder as its weights.
 In Full Bayesian models (\refsec{sec:parameters}), weights are stochastic and sampled as shown on the right \cite{jospin2022hands}.
 \textbf{(b)} Hidden Markov Model (single time step), where states and actions are latent variables.
 Usually, HMM is unrolled into a sequence.
 \textbf{(c)} Latplan's latent action model (single time step), where states and actions are latent variables.
 }
\end{figure*}

Now that I have shown several generative models (focused on directed graphical models),
I describe how to train them next.
While pushing the envelope of available methods is an interesting topic,
I focus on two groups of training methods in the following section.

\section{Variational Method}
\label{sec:variational}

To maximize $\E_{q(x)}p(x)$, machine learning algorithms must compute the generative model $p(x)$,
which requires computing the integral/sum over the latent variables, e.g., $p(x)=\int p(x|z)p(z)dz$.
This is called Probabilistic Inference (PI):
\begin{defi}[Probabilistic Inference]
 Compute $p(\rx=x)$.
\end{defi}
This integration is ``intractable'' when the latent variables are high-dimensional.
To my surprise, although most textbooks mention this,
they rarely mention its exact complexity class.
PI is \#\textbf{P}-complete \cite{dagum1993approximating,roth1996hardness,dagum1997optimal}
which is at least as hard as \textbf{NP}-complete.
\#\textbf{P}-completeness is shown by a reduction to \#SAT (counting SAT) \cite{valiant1979complexity},
a problem of counting all solutions to a CNF formula,
by assuming all random variables are boolean.
Informally, \#SAT is compiled to PI as follows:
Given a \#SAT instance,
we convert each variable $v_i$, each clause $C_j$, and the satisfiability of the formula $T$,
as a boolean random variable.
Given a random assignments, $p(T=\text{true})$ equals to $\frac{\#\text{solutions}}{\#\text{all assignments}}$,
thus PI can solve a \#SAT instance.

To maximize $\E_{q(x)}p(x)$,
each iteration of machine learning must solve a PI, a \#\textbf{P}-complete problem.
To avoid this complexity, the \emph{variational method} instead computes its lower bound approximation (\emph{variational inference}).
The most basic example of a variational method is a VAE \cite{kingma2014semi}.
Variational methods use so-called \emph{variational distributions} to perform the approximation.
Like many other statistical conventions,
variational distributions are typically explained with interpretations rather than with formal definitions.
I provide the definitions below:

\begin{conv}
 The \emph{first principle derivation} is a derivation
 that uses the probability axioms (e.g., \refdefs{defi:joint}{defi:conditional}) only.
\end{conv}

\begin{conv}
 A distribution is \emph{exact}
 if it is derived from a generative model with first principles.
 \emph{Variational} otherwise.
\end{conv}

\begin{ex}
 Given $p(\rz)$ and $p(\rx|\rz)$,
 $p(\rz|\rx)=\frac{p(\rx|\rz)p(\rz)}{p(\rx)}$ (\reftheo{theo:bayes}) is exact.
 Any $q(\rz|\rx)\not=p(\rz|\rx)$ is variational.
\end{ex}

\begin{conv}
 A posterior distribution is often called a \emph{true posterior distribution} if it is exact.
 Confusingly \textbf{this does not imply} that it is a \textbf{ground truth} posterior distribution.
\end{conv}

\begin{conv}
 A \emph{variational model} is a set of variational distributions.
\end{conv}

\begin{defi}[Variational Inference]
 Given a variational model, compute a lower bound of log probability $\elbo(x)\leq \log p(\rx=x)$ in polynomial time.
\end{defi}
 The bound is typically called an \emph{Evidence Lower BOund (ELBO)} or a \emph{variational lower bound}.

\begin{conv}[Amortized Inference]
VIs used to train a separate copy of a variational model for each input data.
Modern methods \cite{kingma2014semi,rezende2014stochastic}
use \emph{amortized inference} which maintains
a single variational model for multiple observations.
For example, a VAE \cite{kingma2014semi} has a single encoder neural network $q(z|x)$ shared by all observation $x$.
\end{conv}

\subsection{Example: VAE}

Variational methods maximize $\E_{q(x)}\log p(x)$ by maximizing its ELBO.
Let me demonstrate a variational method performed on a VAE (\refex{ex:vae}).
Let $p^*(\rx)$ be the ground-truth distribution of $\rx$,
$p(\rx)$ be its current estimate,
$q(\rx)$ be its dataset distribution,
and $q(\rz|\rx)$ be its variational distribution, which is represented by an \emph{encoder} neural network
that maps an image to a latent state.
The design of $q(\rz|\rx)$ is arbitrary and can be done separately from the generative model.
It is considered an approximation of the \emph{true posterior} $p(\rz|\rx)$.

Using a variational posterior $q(\rz|\rx)$,
it derives the lower bound of the objective as follows:
\begin{align}
 &\text{ML Task:} \argmax_p \E_{q(x)} \log p(x)\\
 &\log p(x)
 = \log \sum_z p(x|z) p(z)                                            \label{eq:vae-0}\\
 &= \log \sum_z \green{q(z|x)} p(x|z) \frac{p(z)}{\green{q(z|x)}}                      \label{eq:vae-1} \\
 &= \red{\log} \parens{\blue{\E_{q(z|x)}} \brackets{p(x|z) \frac{p(z)}{q(z|x)}}}    \label{eq:vae-2} \\
 &\geq \blue{\E_{q(z|x)}} \brackets{\red{\log} \parens{p(x|z) \frac{p(z)}{q(z|x)}}} \label{eq:vae-3} \\
 &= \E_{q(z|x)} \log p(x|z) - \E_{q(z|x)} \log \frac{q(z|x)}{p(z)} \label{eq:vae-4} \\
 &= \E_{q(z|x)} \log p(x|z) - \KL(q(z|x)\Mid p(z)).                \label{eq:vae-5}
\end{align}
\refeqs{eq:vae-0}{eq:vae-1} simply multiplies $1=\frac{\green{q(z|x)}}{\green{q(z|x)}}$.
\refeqs{eq:vae-1}{eq:vae-2} is the definition of expectation (\refdef{def:expectation}).
\refeqs{eq:vae-2}{eq:vae-3} used Jensen's inequality (\reftheo{theo:jenssen}) that exchanges \blue{expectation} and \red{logarithm}.
\refeqs{eq:vae-4}{eq:vae-5} is a definition of KL divergence (\refdef{def:kl}).
When $q(\rz|\rx)$ is expressive enough and when the ELBO is maximized,
then $q(\rz|\rx)=p(\rz|\rx)$.
\refsec{sec:expectation} discusses the details of how to actually \emph{compute} each term in \refeq{eq:vae-5}
that includes expectations $\E$ and $\KL$.
Finally, I mention an autoencoder:

\begin{fact}
A non-variational autoencoder loss lacks the $\KL$ in \refeq{eq:vae-5}.
It is unsound (\refsec{sec:ml}) because it is not guaranteed to maximize $\E_{q(x)} p(x)$.
\textbf{Do not use it.}
\end{fact}

\subsection{A General Guide for Variational Distributions}

As mentioned above, the choice of variational distributions is arbitrary.
This gives us the flexibility to add as many heuristic design decisions into their neural networks as you wish
without sacrificing the theoretical integrity.
One way to see variational distributions is to \emph{attach} heuristic guidance to each random variable
based on the domain knowledge.
A VAE assumes that $z$ could be encoded from $x$ by a particular (e.g. Convolutional) neural network.
This is why variational distributions are sometimes called \emph{guides} in
\emph{automated variational inference} and \emph{probabilistic programming language}
frameworks \cite{goodman2012church,wingate2013automated,ranganath2014black}.

Using this intuition, the general strategy for designing variational distributions can be described as follows.
For each latent variable $\rz$:
\begin{enumerate}
 \item $\rz$ should have a single generative distribution $p(\rz|\ldots)$.
       You must already have one made during the statistical modeling.
       ``$\ldots$'' can be empty, in which case $p(\rz)$ is a fixed prior distribution.
 \item $\rz$ should have \emph{at least} one variational distribution $q(\rz|\ldots)$
       Its dependency ``$\ldots$'' does not have to match those of $p(\rz|\ldots)$.
       You can have more than one $q(\rz|\ldots)$ (there is no reason to restrict it to a single distribution),
       and their dependencies may also differ from one another.
 \item The variational distribution $q(\rz|\ldots)$ must be in the same distribution family as $p(\rz|\ldots)$.
       This typically gives the KL divergence
       $\KL(q(\rz|\ldots)\Mid p(\rz|\ldots))$ 
       an analytical form.
 \item Design $q(\rz|\ldots)$ so that they are ``surer/pointier/more informative'' than $p(\rz|\ldots)$
       so that it serves as a guide.
       If possible, make $q(\rz|\ldots)$ depend on more variables than $p(\rz|\ldots)$ does,
       which will make it surer due to having more information.
\end{enumerate}

\begin{ex}
An example of item 4 can be found in Latplan \cite{Asai2022}.
The action variable $\ra$ has
$p(\ra\mid\rz_0)$, a distribution predicted from the current state,
and
$q(\ra\mid\rx_0, \rx_1)$, a distribution predicted from the images before and after the transition.
The former is intrinsically more ambiguous because it lacks access to what has actually happened.
\end{ex}

Note that
the ``guide'' analogy works only when $p(\rz|\ldots)$ is trainable.
It does not make much sense when
$p(\rz|\ldots)$ is a prior, i.e. a constant distribution such as $p(\rz)=\N(0,1)$.

To train the resulting model,
you must derive an ELBO that contains multiple KL divergences and reconstruction losses.
The next section discusses how to perform this derivation for a complex model.

\subsection{Deriving an ELBO: A General Algorithm}

While the VAE provides a nice introductory example for how to derive a lower bound,
the tutorial is not sufficient for a more complex graphical model.
Here I describe a general algorithm for deriving the ELBO for a more complex graphical model.

Let $P=\braces{p(\cdot|\cdot)\ldots}$
and $\green{Q=\braces{q(\cdot|\cdot)\ldots}}$ be a set of distributions in the generative and the variational model.
We use $\cdot$ to represent a set of random variables that we don't care (a wildcard).
$P$ and $\green{Q}$ are defined by the user as inputs.
For example, Latplan used $P=\braces{p(x_0|z_0),p(x_1|z_1),p(z_1|z_0,a),p(a|z_0),p(z_0)}$ and
$\green{Q=\braces{q(z_0|x_0),q(z_1|x_1),q(a|x_0,x_1)}}$.
Let $X$ be a set of observable (and labeled) variables, and $Z$ be a set of latent variables.
$P$ can be seen as representing a factorization of $p(X)$
obtained in the \refline{line:statistical-modeling-step3} in \refconv{conv:statistical-modeling}, i.e.,
\begin{align*}
 p(X)
 &=\sum_{Z} p(X,Z)
 =\sum_{Z} \prod_{p(\ldots)\in P} p(\ldots).
\end{align*}
For example, the factorization in Latplan is
\begin{align*}
 p(x_0,x_1)
 &=\sum_{a,z_0,z_1}p(x_0|z_0)p(x_1|z_1)p(z_1|z_0,a)p(a|z_0)p(z_0).
\end{align*}

We select a subset $\green{Q'\subseteq Q}$ so that
for all $\green{q(A|\cdot)\in Q'}$,
there is a matching $\violet{p(A|\cdot)}\in P$ of the same set of random variables $A$
(we don't care about the dependency difference).
For example, Latplan used
$\green{Q'_1}=\braces{\green{q(z_0|x_0)},\green{q(a|x_0,x_1)}}$ where
$\green{q(z_0|x_0)}$ matches $\violet{p(z_0)}$ and $\green{q(a|x_0,x_1)}$ matches $\violet{p(a|z_0)}$.
Latplan also used $\green{Q'_2}=\green{Q}$.
Note that the opposite may not hold: Not every $\orange{p(\cdot|\cdot)}\in P$ has a corresponding distribution in $\green{Q'}$.
The choice of $\green{Q'}$ splits $P$ into three disjoint subsets ($P=\violet{P_1}\cup \orange{P_2} \cup \gray{P_3}$):
$\violet{P_1}$ contains all latent distributions with a matching $\green{q}$,
$\orange{P_2}$ contains those without a matching $\green{q}$, and
$\gray{P_3}$ is a set of distributions of observed variables.
Using these subsets, the lower bound of $\log p(X)$ is obtained as follows:

\begin{align}
 &\log p(X)
  =\log \sum_{Z} \prod_{p(A|\cdot)\in \violet{P_1}\cup \orange{P_2}\cup \gray{P_3}} p(A|\cdot)
 \notag
 \\
 &=\log \sum_{Z} \prod_{p(\cdot|\cdot)\in \orange{P_2}\cap \gray{P_3}} p(\cdot|\cdot) \prod_{\violet{p(A|\cdot)\in P_1}} \green{q(A|\cdot)} \frac{\violet{p(A|\cdot)}}{\green{q(A|\cdot)}}
 \label{eq:general-elbo2}
 \\
 &=\blue{\log} \red{\E}_{\substack{\orange{p(\cdot|\cdot)\in P_2},\\ \green{q(\cdot|\cdot)\in Q'}}}
 \brackets{
 \prod_{\gray{p(\cdot|\cdot)\in P_3}} \gray{p(\cdot|\cdot)}
 \prod_{\violet{p(A|\cdot)\in P_1}}   \frac{\violet{p(A|\cdot)}}{\green{q(A|\cdot)}}}
 \label{eq:general-elbo3}
 \\
 &\geq \red{\E}_{\substack{\orange{p(\cdot|\cdot)\in P_2},\\ \green{q(\cdot|\cdot)\in Q'}}}
 \brackets{
 \blue{\log}
 \prod_{\gray{p(\cdot|\cdot)\in P_3}} \gray{p(\cdot|\cdot)}
 \prod_{\violet{p(A|\cdot)\in P_1}} \frac{\violet{p(A|\cdot)}}{\green{q(A|\cdot)}}}
 \notag
 \\
 &= \E_{\substack{\orange{p(\cdot|\cdot)\in P_2},\\ \green{q(\cdot|\cdot)\in Q'}}}
 \brackets{
 \sum_{\gray{p(\cdot|\cdot)\in P_3}} \hspace{-0.7em} \log \gray{p(\cdot|\cdot)}
 + \hspace{-1em}\sum_{\violet{p(A|\cdot)\in P_1}} \hspace{-0.7em} \log \frac{\violet{p(A|\cdot)}}{\green{q(A|\cdot)}}}
 \label{eq:general-elbo5}
\end{align}

In \refeqs{eq:general-elbo2}{eq:general-elbo3}, note that the variables in $\orange{P_2}\cup \green{Q'}$ is $Z$.
\refeq{eq:general-elbo5} is
a sum of the reconstruction losses for the observables in $\gray{P_3}$ and
the $\KL$s
(or equivalents
\footnote{
For example,
$\E_{q(z|x)}\E_{q(y|z)}\log \frac{q(z|x)}{p(z|y)}$ is not a KL divergence due to $\E_{q(y|z)}$.
We can't remove $\E_{q(y|z)}$ as $p(z|y)$ depends on $y$.
})
for the latents in $\violet{P_1}$.

Note that each ELBO depends on $\green{Q'}$, which has exponentially many combinations.
$\green{Q'_1}$
results in $\orange{P_2}=\braces{\orange{p(z_1|z_0,a)}}$ and two $\KL$s
while
$\green{Q'_2}$
results in $\orange{P_2}=\emptyset$ and three $\KL$s.
In two ELBOs, $z_1$ follows different distributions ($\orange{p(z_1|z_0,a)}$ vs. $\green{q(z_1|x_1)}$)
which affects $\log\gray{p(x_1|z_1)}$.
\begin{align}
 Q'_1:\ &\E\substack{\green{q(z_0|x_0)}\\\orange{p(z_1|z_0,a)}\\\green{q(a|x_0,x_1)}}
 \brackets{
 \textstyle
 \substack{
 \log\gray{p(x_0|z_0)}\\
 + \log\gray{p(x_1|z_1)}}
 \substack{
 + \log \frac{\violet{p(z_0)}}{\green{q(z_0|x_0)}}\\
 \quad + \log\frac{\violet{p(a|z_0)}}{\green{q(a|x_0,x_1)}}
 }
 }
 \notag
 \\
 Q'_2:\ &\E\substack{\green{q(z_0|x_0)}\\\green{q(z_1|x_1)}\\\green{q(a|x_0,x_1)}}
 \brackets{
 \textstyle
 \substack{
 \log\gray{p(x_0|z_0)}\\
 + \log\gray{p(x_1|z_1)}}
 \substack{
 + \log \frac{\violet{p(z_0)}}      {\green{q(z_0|x_0)}}\\
 \quad + \log \frac{\violet{p(z_1|z_0,a)}}{\green{q(z_1|x_1)}}\\
 \qquad + \log \frac{\violet{p(a|z_0)}}    {\green{q(a|x_0,x_1)}}}
 }
 \notag
\end{align}

Not all lower bounds are useful.
To illustrate the issue, look at the second ELBO of a VAE ($Q=\braces{\green{q(z|x)}},Q'=\emptyset$):
\begin{align}
  \log p(x) \geq \E_{\orange{p(z)}} {\log \gray{p(x|z)}}\label{eq:second-vae-elbo}
\end{align}
It is less tight (= worse) than the normal VAE ELBO because
the generator $p(x|z)$ uses $z$ from a fixed distribution $p(z)$, ignoring the input data
and not training the encoder $q(z|x)$.

The criteria for selecting $\green{Q'}$ is not known.
Latplan empirically showed that
averaging ELBOs from $Q'_1,Q'_2$ was sufficient, but did not test $2^3$ combinations.
One heuristic is to form a set $\mathbf{Q'}=\braces{Q'_1,Q'_2\ldots}\subseteq 2^Q$
so that
(1) all trainable networks are covered once by $\orange{P_2}\cup \green{Q'_i}$ and
(2) ignore $\green{Q'}$s that ignore the input.
\refeq{eq:second-vae-elbo} violates both criteria.

The usefulness may also depend on how we estimate the expectation.
VAEs have a continuous distribution $p(z)=\N(0,1)$;
therefore $z$ must be estimated by Monte Carlo sampling.
However, this is not always necessary:
If $p(z)$ is a categorical distribution $p(z)=\cat(1/4,\ldots,1/4)$ of 4 categories,
I can enumerate 4 cases and compute an exact weighted sum,
which could be a better lower bound.
\refsec{sec:expectation} discusses more about how to compute an expectation.

\paragraph{Further Notes}

VAEs tend to generate blurry images.
The cause of this phenomena was identified as \reffact{fact:loss}
(assign a particular distribution, such as Gaussian, to observable variables $\rx$).
This gave rise to
\emph{likelihood-free} methods of machine learning that avoids assigning distributions to $p(\rx)$,
which includes Generative Adversarial Networks \cite[GANs]{goodfellow2014generative} and its variants.
However, many GAN variants (including the vanilla GAN) are unsound, leading to unstable training.
I discuss a sound likelihood-free method in \refsec{sec:likelihood-free}.

\section{Obtaining an Expectation}
\label{sec:expectation}

Having laid out the derivation of the loss functions,
we finally discuss how to actually compute them.
In doing so, computing an expectation $\E_{p(x)} g(x)$ is critical.
There are mainly three ways to compute an expectation.
\begin{enumerate}
 \item \textbf{A closed form is available.}
       This is often the case when $g(x)$ is a PDF of a distribution $q(x)$
       of the same family as $p(x)$.
       A KL divergence is also such an instance.
 \item \textbf{The random variable is discrete.}
       If it is a low-dimensional discrete variable,
       you can enumerate all cases and compute the expectation exactly.
 \item \textbf{Numerical sampling.}
       Otherwise, you must \emph{estimate} the expectation via random sampling.
       Monte-Carlo sampling is one such instance.
\end{enumerate}

\begin{defi}
 Given \iid random variables $\rx_1, \rx_2, \ldots, \rx_N$ all following $p(\rx)$,
 i.e., $p(\rx)=p(\rx_i)$ and $\rx_1\perp\ldots\perp \rx_N$,
 and its samples $x_i\sim p(\rx_i)$,
 the Monte Carlo (MC) estimate of $\E_{p(x)} g(x)$ is defined as $\frac{1}{N}\sum_{i=0}^N g(x_i)$.
\end{defi}
\noindent In practice, however, the MC estimate is extremely simplified.
\begin{fact}
 \label{fact:mc-single-sample}
 Each expectation is obtained by a Monte Carlo estimate with $N=1$, as popularized in \cite{kingma2014semi}.
\end{fact}
In other words, no averaging is performed in the source code. 
This helps deciphering a complex formula in a paper:
\begin{fact}
 Except for cases 1 and 2 (close form / discrete cases),
 an expectation $\E$ in a complex formula
 should read as a single random sampling from a distribution.
\end{fact}
\begin{fact}
  In a minibatch training, the expectation over the empirical distribution $\E_{q(x)}$ is
  done by computing the average over $B$ batch samples from the dataset.
\end{fact}

\begin{ex}
 The VAE's ELBO (\refeq{eq:vae-3}, including $\E_{q(x)}$) is
 \[
 \E_{q(x)}[\E_{q(z|x)} [\log p(x|z)] - \KL(q(z|x)\Mid p(z))].
 \]
 In a minibatch training, this should read as follows:
 \begin{algorithmic}[1]
  \State $\E_{q(x)}$ = Sample $x\sim q(x)$ for $B$ times and average it,
  \State $\E_{q(z|x)}$ = Sample $z\sim q(z|x)$ once = encode $x$ to $z$ once,
  \State Compute $L_1=\log p(x|z)$ (closed form, see \refsec{sec:loss}),
  \State Compute $L_2=\KL(q(z|x)\Mid p(z))$ (closed form),
  \State \Return $L_1-L_2$.
 \end{algorithmic}
\end{ex}

\section{Conclusion}

This memo discusses a concise protocol for designing a machine learning system
with a minimum reasonable theoretical guarantee.
I targeted a general computer science audience
not necessarily specialized in ML/Stats/DL,
especially those in the symbolic AI community.

In the first half of the memo,
I reviewed a minimal condition
that machine learning methods must satisfy
in order for the symbolic AI community to take it seriously.
I kept the discussion general enough that it is agnostic to the statistical model or the implementation.
(1) I minimally covered the basic (but often not easily accessible) statistical concepts.
(2) I defined machine learning as a standard optimization problem.
(3) Inspired by traditional theorem proving terminologies,
I defined the soundness and the completeness of machine learning.
(4) Based on the completeness, I shed light on the generalization in machine learning.
One novel aspect of this discussion was its focus on the support (non-zero region) of probability distributions,
a deliberate choice made for non-deterministic reasoning in symbolic AI.
It suggests that statistical learning methods need more focus on Extreme Value Theory to ensure the safety.

In the second half of the memo, I then
standardized the protocol for performing machine learning
while maintaining the guarantees discussed above.
I discussed
(1) the connection between loss functions and the choice of distributions,
(2) the maximum entropy principle for choosing a distribution,
(3) a principled procedure for designing a complex statistical model,
(4) a general guide for designing a complex variational model,
(5) an algorithm for deriving its loss formula, and finally,
(6) computing this formula.

In addition to providing the protocol for designing ML systems,
this memo would make existing papers less demanding to read,
give readers more confidence,
and as a result make them more accepting toward statistical approaches.
In other words, the true goal of the memo is to
shed a cautiously optimistic light on machine learning and
bridge the gap between connectionist and symbolic AI communities,
which would hopefully spark the development of neuro-symbolic systems that bring the best of both worlds.

\section*{References}

References appear after the appendix.

\clearpage
\appendix

\section*{Appendix}

\section{Axiomatic Measure / Probability Theory}
\label{sec:measure}

(This section is based on \citet{rohatgi2015introduction} and \citet{falconer2004fractal}.)
To define probability, I should minimally cover its measure-theoretic definition \cite{kolmogorov1933foundations}.
I don't delve into the details because it is not the core topic of this article.
However, to have a keyword that the readers can search later may be useful.
It might also be a good idea to keep these notions in mind if you are a job seeker:
Remember the existence of these notions
just in case someone asks you about them during a machine learning job interview.

Basically a measure is a mathematical generalization of volume,
where you can integrate the density to obtain the total mass.

\begin{defi}
 Given two sets $X,Y$, we denote a set of functions from $X$ to $Y$
 as $Y^X$ or $X\to Y$.
\end{defi}

\begin{defi}
 Given a set $\Omega$, we denote $2^\Omega$ as a power set of $\Omega$, i.e., the set of all subsets of $\Omega$.
 This is a special case of a set of functions, where $Y=\braces{0,1}$ is denoted by $2$.
\end{defi}

\begin{ex}
 When $\Omega$ is a set of numbers that you could get from throwing a dice,
 $\Omega=\braces{1,2,3,4,5,6}$,
 $2^\Omega=\braces{\emptyset, \braces{1}, \ldots, \braces{6}, \braces{1,2}, \ldots, \braces{5,6}, \braces{1,2,3}, \ldots, \Omega}$.
\end{ex}

\begin{defi}[Measure Axiom]
 Given a set $\Omega$, and a set of its subsets $\Sigma\subseteq 2^\Omega$,
 a function $\mu: \Sigma \to \R$ is a \emph{measure} iff
 \begin{enumerate}
  \item $\forall x \in \Sigma; \mu(x) \geq 0$,
  \item $\mu(\emptyset) = 0$,
  \item $\mu(A) \leq \mu(B)$ if $A \subseteq B$, and
  \item for a countably infinite sequence of sets $\braces{S_i}_{i=0}^\infty$,
        $\mu(\bigcup_i S_i) \leq \sum_i \mu(S_i)$.
        Equality holds when $S_i$ are mutually disjoint ($S_i\cap S_j = \emptyset$ if $i\not=j$).
 \end{enumerate}
\end{defi}

There is a certain condition called $\sigma$-algebra that $\Sigma$ must satisfy,
but its theoretical details are not important to us.
In short,
if $\Omega$ is a set of real numbers $\R$,
$\Sigma$ should ``behave well'' in order for $\mu$ to be well-defined
because some pathological sets such as $\R\setminus \Q$ or a Cantor set can cause complications
when $\mu$ defines and generalizes integration and summation.
So far, I can informally assume that $\Omega$ and $\Sigma$ are well-behaved.

\begin{defi}[Probability Axiom]
 A function $\mu$ is a probability measure when it is a measure,
 $\mu: \Sigma\to{} [0,1]$, and $\mu(\Omega) = 1$.
\end{defi}

\begin{ex}
 $\mu(\braces{1,3,5}) = 0.5$, i.e., the probability of observing an odd number from a fair dice is 0.5.
 $\mu(\braces{1,2,3,4,5,6}) = 1$ and $\mu(\emptyset) = 0$.
\end{ex}

\begin{defi}
 For a probability measure $\mu$ on $\Omega$ and $\Sigma$,
 $\Omega$ is called a \emph{sample space},
 $\Sigma$ is called an \emph{event space},
 $x\in\Sigma$ is called an \emph{event}.
 $(\Omega, \Sigma, \mu)$ is called a \emph{probability space}
 if a complement and a union of events are defined, i.e.,
 \begin{enumerate}
  \item $x \in \Sigma \iff \Omega \setminus x \in \Sigma$, and
  \item $x, y \in \Sigma \then x\cup y \in \Sigma$.
 \end{enumerate}
\end{defi}

$\Omega$ is a set of possible outcomes, $\Sigma$ is a set of (measurable) subset of possible outcomes,
and $\mu$ is a probability for each (measurable) subset.
Here the adjective ``measurable'' is used only to avoid the complications of $\R$,
and thus you can safely ignore them.
Finally,

\begin{defi}
 Given a probability space $(\Omega, \Sigma, \mu)$,
 and $(E, \mathcal{E})$ where $\mathcal{E}$ is a well-behaving subset of $2^E$,
 a random variable $X$ is a function from $\Omega$ to $E$ such that $\forall e\in \mathcal{E}; X^{-1}(e)\in \Sigma$.
 $E$ is also called an \emph{observation space}. $\Omega$ is also called a \emph{background space}.
\end{defi}

\begin{ex}
 Let $\Omega=[0,1)$.
 $E=\braces{1,2,3,4,5,6}$. An example of $X$ is
 $X([0,\frac{1}{6}))=1, \ldots, X([\frac{5}{6},1))=6$.
 In this case, $\Sigma=\braces{[0,\frac{1}{6}), \ldots [\frac{5}{6},1), [0,\frac{2}{6}), [0,\frac{1}{6})\cup [\frac{2}{6},\frac{3}{6}), \ldots, [0,1)}$.
\end{ex}

The distinction of the background space and the observation space is made
only for generalized, more complicated cases
\footnote{See \href{https://stats.stackexchange.com/questions/10789/why-are-random-variables-defined-as-functions}{this stackexchange post}}.
The ``observation space'' and ``background space'' have nothing
to do with ``observed variables'' and ``latent variables'' discussed later.

\section{Formal Concepts in Statistics : Tier 2}
\label{sec:tier2}

Here I cover less important concepts.
There are several variants of entropy that I can't think but they exist just for confusing readers.

\begin{theo}[Inclusion-exclusion principle]
 \[
 p(\rx=x\lor\ry=y)=p(\rx=x)+p(\ry=y)-p(\rx=x\land\ry=y)
 \]
\end{theo}

\citet{russell1995artificial} attributes this principle to Andrei Kolmogorov.

\begin{defi}
 A \emph{cross entropy} between $q(\rx)$ and $p(\rx)$ is $\E_{q(x)} \brackets{-\log p(x)}$.
\end{defi}

Cross entropy frequently appears as a loss function for classification in machine learning,
but it is not fundamental.

\begin{defi}
 A \emph{joint entropy} of $p(\rx,\ry)$ is $H(p(\rx,\ry))$, i.e., not different from normal entropy.
 Also written as $H(\rx,\ry)$ when $p$ is implied.
\end{defi}

\begin{defi}
 A \emph{conditional entropy} of $p(\rx|\ry)$ is $H(p(\rx|\ry))=\E_{p(\rx,\ry)} \brackets{-\log p(\rx|\ry)}=H(\rx,\ry)-H(\ry)$.
 Note that the expectation is over $p(\rx,\ry)$, not $p(\rx|\ry)$.
 Also written as $H(\rx|\ry)$ when $p$ is implied.
\end{defi}

\begin{defi}
 A random variable is \emph{discrete} / \emph{continuous} if its sample space is continuous / discrete.
 A distribution is \emph{discrete} / \emph{continuous} if its variable is continuous / discrete.
\end{defi}
\begin{defi}
 A random variable is \emph{multivariate} if it is a list/vector/array. It is \emph{univariate} otherwise.
 It is \emph{bivariate} if the length is two.
 A distribution is multivariate if its variable is multivariate and if they all follow the same type of distribution.
 It typically implies that variables correlates with each other.
 It is just a scarier way to call a joint distribution of all variables in a vector.
\end{defi}
\begin{defi}
 A distribution is a \emph{mixture} if it is a weighted sum of distributions.
 It is same as saying $p(x)=\sum_C p(x|C)p(C)$ where $p(C)$ is a categorical distribution.
\end{defi}
\begin{defi}
 A distribution is \emph{deterministic} when it is a Dirac's delta. It is stochastic otherwise.
\end{defi}
\begin{defi}
 A multivariate distribution has a \emph{mean field assumption} when its variables are mutually independent.
\end{defi}
\begin{defi}
 A \emph{support} of a function $f$ is where it is non-zero, $\function{supp}(f)=\braces{x| f(x)\not=0}$.
 For a measure $\mu$ (which is always positive), $\function{supp}(\mu)=\braces{x| \mu(x)>0}$.
 Probability distributions are measures (\refsec{sec:measure}), therefore the same definition applies.
\end{defi}

\subsection{Parameter Estimation}
\label{sec:parameters}

Typically, we assume $\hat{p}^*(\rx)$ and $p(\rx)$ are of the same family of functions
parameterized by $\theta$ such as neural network weights,
i.e., $\hat{p}^*(\rx)=p_{\theta^*}(\rx)$, $p(\rx)=p_{\theta}(\rx)$.
Thus, ML is often written as a task of finding $\argmax_\theta\E_{q(x)}p_{\theta}(x)$.
Formally,

\begin{defi}
Let $\vtheta$ be a vector of random variables representing the learned parameters in a machine learning system.
Then $p(\rx|\vtheta)$ below is called a \emph{likelihood}. $p(\rx)$ is in turn called a \emph{marginal likelihood}.
\[
 p(\rx) = \sum_\vtheta p(\rx|\vtheta) p(\vtheta).
\]
\end{defi}

\begin{conv}
 \emph{Maximum Likelihood Estimation (MLE)}
 is a machine learning with a MAP estimation on $\vtheta$.
\end{conv}

That being said, I avoided the term ``likelihood'' throughout this memo.
It is a particularly ill-named concept because
it calls a certain noun (distribution) with a different noun (likelihood)
with no particular reason and
disrupts the consistency of notations.

\emph{Frequentist} and \emph{Partial Bayesian} approaches
optimize $p(\rx|\vtheta)$ while treating $\vtheta$ as a point / a deterministic variable / Dirac's $\delta$
(\refsec{sec:point}, \refsec{sec:probability}).
\emph{Full Bayesian} methods instead optimize $p(\rx|\vtheta)$ and $p(\vtheta|\rx)$
using a prior distribution $p(\vtheta)$.
The difference between Frequentist and Partial Bayesian is that
the latter still has a prior on other non-weight variables.
See \refsec{sec:uncertainty} for more details on Bayesian reasoning.

\begin{ex}
 \emph{Bayesian Neural Network} \cite[BNN]{kendall2017uncertainties} is a Full Baysian method.
 Each neural network weight $w_i$ is represented as a distribution such as a Gaussian $w_i\sim\N(\mu_i,\sigma_i)$
 whose parameters $\mu_i,\sigma_i$ are optimized subject to
 the output accuracy and the KL divergence toward the prior distribution such as $\N(0,1)$.
 Each time it computes an output from the input, a new weight value is sampled from the distribution.
 While it can directly distinguish the aleatoric and the epistemic uncertainty (see \refsec{sec:uncertainty}),
 the sampling process makes it computationally intensive.
 Weight regularizations in neural networks (such as $\ell_1, \ell_2, \ell_\infty$ regularizations)
 are considered the special cases of BNNs.
\end{ex}

\section{Frequentist Approaches}
\label{sec:frequentist}

\emph{Frequentist approaches} are alternative approaches toward machine learning and hypothesis testing.
They are different from Bayesian approaches in a number of ways.

First, Frequentist approaches
use a frequentist interpretation of probability (\refconv{conv:frequency}).
When there are no observations made yet,
then the empirical probability distribution simply ``does not exist'' or is undefined,
because it is based on a frequency of the past events.
In contrast, probabilities always exist in Bayesian approaches as it uses a subjective view.

Next, they
try to obtain the \emph{true} parameters of the ground truth probability distributions.
In other words, they assume that such parameters are deterministic value with 0 variance, i.e., a Dirac's delta $\delta$.
Even when no observations are made, they still assume that there is some true value that is simply not known.
In doing so,
it relies on various \emph{limit theorems},
including the Laws of Large Numbers \cite{bernoulli1713lawsoflargenumbers,grattan2005landmark}, which says
the estimation converges to the true value given an infinite amount of data.
This is in contrast to \emph{Bayesian approaches} which admits that
the true parameter will be never known in our lifetime.
Instead, they obtain the \emph{distribution} of the parameters from a \emph{finite} amount of data.

Bayesian approaches can thus learn more effectively from limited data \cite{tenenbaum1998bayesian}.
Part of it is due to being able to leverage a fixed distribution called a prior distribution --
An initializing distribution that acts as a fake, pseudo samples of several pseudo-trials
and augments the lack of data by human intuition and common sense.
It updates this initial distribution with a finite data
and obtains a posterior distribution, a distribution closer to the ground truth.

Frequentist approaches claim that they do not use a prior.
One must be careful on these claims because they are sometimes political and dogmatic.
Any hyperparameter for a frequentist model can be seen as a prior from a Bayesian view,
but the Frequentist school of thoughts rejects the idea of subjectivity and prior knowledge.
Further discussion is out of the scope of this memo.

\begin{table}[htbp]
 \centering
 \begin{tabular}{c|cc}
\toprule
  Approach   & Frequentist   & Bayesian       \\
\midrule
Interpretation & Frequency     & Belief         \\
Result         & Deterministic & Distributional \\
Data assumed   & Infinite      & Finite         \\
Prior?         & No            & Yes            \\
\bottomrule
 \end{tabular}
 \caption{A table summarizing the difference of Frequentist and Bayesian approaches.}
 \label{tab:frequestist-bayesian}
\end{table}

\subsection{PAC Learning}

Frequentist learning theories
are built around the concept of
\emph{Provably Approximately Correct (PAC)} inequality and learnability \cite{valiant1984theory}.
PAC is a frequentist analogue of ELBO-based variational model.

Lets assume a dataset $\mathcal{D}=(\X,\Y)\subseteq X\times Y$ which consists of
an input dataset $\X=(x_i)_{i=0}^{N}\subseteq X$ and
an output dataset $\Y=(y_i)_{i=0}^{N}\subseteq Y$.

\begin{conv}
 \emph{Frequentist approaches} assume that the data distribution
 $q(\rx,\ry)$ was \iid sampled from the ground truth distribution
 $p^*(\rx,\ry)$, i.e., $(x_i,y_i)\sim p^*(\rx,\ry)$.
\end{conv}

Let a predictor function $\phi:X\to Y$ and
an arbitrary loss function $l: Y\times Y\to \R^+$.
We assume a class of predictors $\Phi \subseteq X\to Y$.
Notice that unlike Bayesian approaches, there is typically no interpretation provided to $l$.
It can be an arbitrary loss function and not necessarily connected to a likelihood of some distribution.
However, from a Bayesian point of view, you can always interpret $l$ as a NLL
by converting it back to $p(y|x)= A\exp \parens{- l(\phi(x), y)} $
with some normalizing constant $A$ that satisfies $\int p(y|x)dx = 1$.

There are two types of PAC inequalities: An empirical one and an oracle one \cite{guedj2019primer,alquier2021user}.
We first define \emph{oracle risk} and \emph{empirical risk}.
Oracle risk is not computable because $p^*(\rx,\ry)$ is unknown.

\begin{defi}
 The \emph{oracle risk} is defined as
 \begin{align}
  R(\phi)&= \E_{(x,y)\sim p^*(\rx,\ry)} [ l(\phi(x), y) ].
 \end{align}
\end{defi}

\begin{defi}
 An \emph{empirical risk} is defined as
 \begin{align}
  \tilde{R}(\phi)&= \frac{1}{N} \sum_i l(\phi(x_i), y_i).
 \end{align}
\end{defi}

\begin{defi}
 For a predictor $\phi\in \Phi$, and $\epsilon \in \R^+$,
 an \emph{empirical PAC inequality} is a condition where there exist some threshold $\delta$ such that:
 \begin{align}
  \Pr\parens{R(\phi) \leq \delta(\phi,\mathcal{D}),\ \mathcal{D}} \geq 1-\epsilon.
 \end{align}
\end{defi}

$\delta(\phi,\mathcal{D})$ is called an \emph{Empirical PAC bound}.
Empirical PAC inequality is able to quantify that,
for a given predictor $\phi$,
it is able to bind the oracle risk by a data-dependent metric $\delta(\phi,\mathcal{D})$,
where the definition of $\delta$ depends on each PAC-learning algorithm.
$\delta$ is often defined by adjusting the empirical risk $\tilde{R}$ with an additional term.
A PAC-learning algorithm optimizes the upper bound $\delta$ as a loss function
instead of $\tilde{R}$, in order to guarantee the inequality.
Notice the similarity with ELBO in a VAE, which adjusts the reconstruction loss $\log p(x|z)$
(which is a square error; similar to $\tilde{R}$)
with a KL divergence in order to keep the loss function a lower bound of the likelihood.

\begin{defi}
 An \emph{oracle PAC inequality} is a condition where there exist some fast-decaying function $\delta$ of $N$ such that:
 \begin{align}
  \Pr\parens{R(\phi) \leq \inf_{\phi\in\Phi} R(\phi) + \delta(N, \epsilon),\ \mathcal{D}} \geq 1-\epsilon.
 \end{align}
\end{defi}

Oracle PAC bound is a more theoretical concept which says the more, the merrier.
It is able to bind the oracle risk by the \emph{best} predictor among $\Phi$,
plus some residual $\delta(N,\epsilon)$ that is fast decaying as more data become available.

\subsection{Limit Theorems}

These frameworks rely on a group of mathematical theorems called \emph{limit theorems}.
Limit theorems include
\emph{law of large numbers} \cite[LLN]{bernoulli1713lawsoflargenumbers,grattan2005landmark},
\emph{central limit theorem} \cite[CLT]{laplace1812centrallimittheorem},
\emph{law of iterated logarithm} \cite[LIL]{Kolmogoroff1929}.
We first define two forms of function convergence:
\begin{defi}
 A series of functions $(f_n)_{n=0}^\infty$ converges to $f$
 \begin{align*}
  \text{\emph{pointwise} when\ } & \forall x; \forall \epsilon; \exists n; |f_n(x)- f(x)|<\epsilon.  && f_n\to f \\
  (\text{\emph{uniform}   when\ } & \forall \epsilon; \forall x; \exists n; |f_n(x)- f(x)|<\epsilon. && f_n\rightrightarrows f)
 \end{align*}
\end{defi}
\noindent Then we define three forms of probabilistic convergence with decreasing strengths \cite{rohatgi2015introduction}:
\begin{defi}
 Random variables $X_n$ converges to $X$
 \begin{align*}
  &\text{\emph{almost surely}}: \hfill \forall \delta\in\R^+; \lim_{n\to\infty} \Pr(\sup_{m\geq n}|X_m-X|<\delta)=1,
  \\
  &\text{\emph{in probability}}: \hfill \forall \delta\in\R^+; \lim_{n\to\infty} \Pr(|X_n-X|<\delta)=1,
  \\
  &\text{\spc{\emph{in distribution}\\or \emph{in law}}}: \hfill \Pr\parens{X_n}=f_n(x) \to f(x)=\Pr\parens{X},
 \end{align*}
 denoted as $X_n\to[p]\mu$, $X_n\to[a.s.]\mu$, and $X_n\to[L] X$, respectively.
\end{defi}

\begin{theo}
 $X_n\to[a.s.]X\ \then\ X_n\to[p]X\ \then\ X_n\to[L] X.$
\end{theo}

Let $\rx_1, \rx_2, \ldots \rx_n$ be i.i.d random variables following any distribution
with mean $\E[\rx_i]=\mu$ and variance $\Var[\rx_i]=\sigma^2$ for each $i$.
Let an \emph{empirical mean} be $\rmu_n=\frac{1}{n}\sum_i \rx_i$.

\begin{theo}[Weak LLN]
 $\rmu_n\to[p]\mu$.
\end{theo}
\begin{theo}[Strong LLN]
 $\rmu_n\to[a.s.]\mu$. \label{theo:lln}
\end{theo}
\begin{theo}[CLT]
 $Y_n=\sqrt{n}(\rmu_n-\mu) \to Y\sim\N(0,\sigma^2)$. \label{theo:clt}
\end{theo}
\begin{theo}[LIL]
 $\frac{\sqrt{n}|\rmu_n-\mu|}{\sqrt{2 \log\log n}} \to[a.s.] \sigma$.
 In other words, the speed of convergence of LLN is $\sqrt{\frac{\log \log n}{n}}$.
 \label{theo:lil}
\end{theo}

Note that the shape of the distribution of each $\rx_i$ does not matter.
For example, $\rx_i\sim \text{Uniform}(\mu-\sigma^2/2, \mu+\sigma^2/2)$
has mean $\mu$ and variance $\sigma^2$, but CLT still applies.
CLT does not apply to distributions which lack the mean,
such as a Pareto distribution (\refsec{zoo:pareto}) or a Cauchy distribution (\refsec{zoo:cauchy}).

The proofs of these theorems further rely on related laws on tail events
(Kolmogorov's zero-one law, Hewitt-Savage zero-one law, L\'evy's zero-one law, etc).
I am not knowledgeable enough to discuss these issues yet.
My layman understanding of these laws is similar to
anecdotal Murphy's law which states ``bad thing surely happens''.
Future versions of this memo may cover this topic.

\paragraph{Further notes:}

PAC-Bayes \cite{mcallester2003simplified} is a frequentist approach to analyse Bayesian learning methods.
Recently, there are work on theoretical bridges between PAC-Bayes and Bayes \cite{germain2016pac}.

Statistical testing is a frequentist concept.
LLN and CLT play an important role in Frequentist learning,
but \href{https://stats.stackexchange.com/questions/317521/does-central-limit-theorem-apply-to-bayesian-inference}{the effect is reduced in Bayesian learning}.
Frequentist papers tends to be heavy on math, which is another reason for us to avoid.
Convergence theories of Reinforcement Learning approaches seem to be based on PAC, thus is frequentist.
Recently, Bayesian RL tackles a similar problem from a Bayesian perspective.

\section{Likelihood-Free Variational Methods}
\label{sec:likelihood-free}

VAEs tend to generate blurry images.
Recently, \dl community started to realize that assuming \reffact{fact:loss}
(assign a particular distribution, such as Gaussian, to observable variables $\rx$)
could be the source of the issues preventing VAEs from generating crisp images.
This gave rise to
\emph{likelihood-free} methods of machine learning
that do not assign distributions to $p(\rx)$,
which includes Generative Adversarial Networks \cite[GANs]{goodfellow2014generative} and its variants.

The statistical framework behind likelihood-free methods is \emph{Density-Ratio Estimation}
which predates GANs \cite{sugiyama2012density}.
In this section, I describe VEEGAN \cite{srivastava2017veegan} that
more faithfully follows the philosophy of likelihood-free method.
Although Vanilla GANs contain some elements of density-ratio estimation,
it is an unsound, ad-hoc method.

In order to avoid assuming a particular distribution on the observed variable $\rx$,
VEEGAN flips the role of observed variables and latent variables.
Recall that the ELBO of a VAE was the following:
\begin{align}
 &\text{ML Task:} \argmax_p \E_{q(x)} \log p(x),\\
 &\log p(x) \geq  \E_{q(z|x)} \log p(x|z) - \KL(q(z|x)|| p(z)).
\end{align}
The lower bound used in VEEGAN is as follows:
\begin{align}
 &\text{ML Task:} \argmax_q \E_{p(z)} \log q(z),\label{eq:veegan1}\\
 &\log q(z)        \geq \blue{\E_{p(x|z)} \log q(z|x)} - \red{\KL(p(x|z)||q(x))},\label{eq:veegan2}
\end{align}
where $p(z)=\N(0,1)$,
$p(x|z)$ is the decoder and $q(z|x)$ is the encoder
\footnote{They are called a generator and a reconstructor in VEEGAN.}.

The first term is a \blue{cross entropy between $p(z)$ (\refeq{eq:veegan1})
and $\E_{p(x|z)} \log q(z|x)$},
which has a closed form similar to a squared error.
In other words, it is a reconstruction loss for the latent state.

One issue in this optimization objective is that
we do not know the functional closed form of $p(x|z)$ or $q(x)$ because we do not assume them to be Gaussians,
therefore we cannot compute \red{the KL divergence}.
\emph{Density-ratio estimation} addresses it by
approximating a \emph{density-ratio} $r(x,z)=\frac{q(x)}{p(x|z)}$.
\begin{align}
 \KL(p(x|z)||q(x))
 &= \E_{p(x|z)} \log \frac{q(x)}{p(x|z)} \\
 &= \E_{p(x|z)} \log r(x,z).
\end{align}

As a result, our optimization objective is:
\begin{align}
 \E_{p(z)} \log q(z) &\geq \E_{p(z)p(x|z)} \brackets{\log q(z|x) - \log r(x,z)}.
\end{align}

An actual implementation separately trains a \emph{discriminator} $D(x,z)=\log r(x,z)$
as a binary classifier between a real sample $(x,z)\sim (p(x), \E_{p(x)} q(z|x))$
and a fake, generated sample $(x,z)\sim (\E_{p(z)} p(x|z), p(z))$.

Remember that VAEs assume both $\rx$ and $\rz$ follows a Gaussian,
while density-ratio-based methods only assume $\rz$ to be a Gaussian,
which makes the representation of $\rx$ arbitrary and more flexible.

\subsection{Pitfalls of GANs are Now Largely Resolved}

Likelihood-free methods (GANs) are known for their numerous pitfalls.
The well-known pitfalls of GANs are as follows:
\begin{enumerate}
 \item \emph{Posterior / Mode Collapse}:
       It causes all latent vectors to map to the same visualization.
 \item \emph{Vanishing Gradient}: When the true and the fake distributions are too dissimilar,
       it is very easy for the discriminator to distinguish the two.
       Such a discriminator does not provide the generator the right amount of guidance.
 \item \emph{Unstable Convergence}:
       GANs train the loss function for maximization and minimization,
       which is formally understood as a saddle-point optimization problem.
       Such a training may not reach the global optima and has unstable convergence.
\end{enumerate}
However, these issues are largely addressed these days.
I focus only on methods which I regard as a fundamental solution to the underlying cause of issues.

\begin{ex}[Mode Collapse]
The mode collapse of a vanilla GAN \cite{goodfellow2014generative} was caused by its unsound optimization.
A vanilla GAN's loss function lacks the first \blue{cross-entropy} term in \refeq{eq:veegan2} (the opposite of an autoencoder),
thus does not solve the ML problem.
VEEGAN addresses the issue by using a sound optimization objective.
\end{ex}

\begin{ex}[Unstable Convergence]
While numerous ad-hoc training methods
(e.g., Wasserstein GAN \cite{arjovsky2017wasserstein}) tried to mitigate this issue,
it wasn't until MMD-Nets \cite{dziugaite2015training,li2015generative,srivastava2019generative}
that they address the core issue of GANs that their training is a saddle-point optimization.
MMD-Nets use \emph{a non-trainable $D$} and thus completely eliminates the saddle point issue from the fundamental level.
$D$ is based on Maximum Mean Discrepancy \cite{sugiyama2012density},
a metric directly computed from the sample data using kernel-tricks
(tangentially related to SVMs \cite{cortes1995support}).
\end{ex}

\section{Bayesian Reasoning}
\label{sec:uncertainty}

Bayesian Reasoning enables reasoning under uncertainty from the limited data.
We explain the following concepts in order:
\begin{itemize}
 \item Uncertainty: The measure of how much we don't know about a value.
 \item Confidence:  The measure of how much we know about how much we do or don't know about a value.
 \item Bayesian reasoning with conjugate priors.
\end{itemize}

\subsection{Uncertainty = Entropy}

Some probabilistic reasoning tasks are often said to deal with uncertainty.
Everyone would agree that a distribution is a less certain representation than a value.
A value $x=c$ itself can be identified as a Dirac's delta $\delta(\rx=c)$,
which is a pointy distribution with an infinite peak, i.e., an absolutely certain distribution.
Usual distributions are flatter.

\begin{conv}
 The uncertainty of a distribution is measured by its entropy.
 This is straightforward in $\cat(\vp)$.
\end{conv}
\begin{fact}
 The uncertainty of $\N(\mu,\sigma)$ is often attributed to $\sigma$
 because its entropy is $\frac{1}{2}+\log \sqrt{2\pi\sigma^2}$.
\end{fact}

There are two types of uncertainty in a machine learning system \cite{kendall2017uncertainties}:

\begin{conv}
 \emph{Aleatoric uncertainty} is an uncertainty in the observation,
 i.e., it is an uncertainty in the data collection agent $a_{data}$.
 The word \emph{aleatoric} means ``by chance.''
\end{conv}

\begin{conv}
 \emph{Epistemic (subjective) uncertainty} is an uncertainty inherent in the system / the hypothesis agent $a_{hypo}$.
\end{conv}

\begin{ex}
 Due to the physical restriction,
 a single pixel in an image represents a mean strength of various rays that hit an individual CMOS sensor.
 Distance, blur, ISO values, etc., all contributes to high aleatoric uncertainty.
 Meanwhile, a machine learning model may not be trained enough on a certain dataset and is unsure about the answer.
 This is a form of epistemic (subjective) uncertainty of the system.
\end{ex}

\begin{ex}[Distribution-to-distribution estimation]
 What if the dataset also contains a distributional information?
 For example, when each element in the dataset $\X$ is a pair $(\mu_i,\sigma_i)$ of
 the mean $\mu_i$ and the variance $\sigma_i$ of a Gaussian?
 In this case, we can replace \refeq{eq:data-delta} with such a distribution:
 \begin{align}
 q(x) &= \sum_i q(x|i)q(i), \\
 q(x|i) &= \N(\mu_i, \sigma_i),\\
 \hat{p}^*(x) &= \argmax_p \E_{q(i)} \E_{q(x|i)} \log p(x).
 \label{eq:ml-d-d}
 \end{align}
 The quantity $\E_{q(x|i)} \log p(x)$ is a cross entropy,
 which has a closed form when both $q(x|i)$ and $p(x)$ are Gaussians.
\end{ex}

\subsection{Confidence = Pseudocounts}
\label{sec:pseudocount}

Imagine someone (an agent) proposed to throw a coin and claims that the coin is fair,
i.e., a random variable $\rx$ about the flipped coin being a head follows $p(\rx)=\bern(\theta=0.5)$.
It says it is uncertain about $\rx$ --- It could be true or false.
What it does not say is how certain it is about \emph{this} judgment.
We can consider two cases:
\textbf{(Case 1)} This is a pure gut feeling with zero evidence, i.e.,
the agent knows nothing and has just applied a default principle
of maximum entropy (among $\bern(\theta)$, $\theta=0.5$ maximizes the entropy).
\textbf{(Case 2)} This is a judgment made after an infinite number of experimental trials
which concluded with absolute certainty that this is a fair coin.
As you can see in these two cases, $\bern(\theta=0.5)$ may be uncertain, but
it does not quantify the amount of confidence in two cases.

To quantify the difference, consider adding a new parameter $N$ to these distributions,
which represents the number of evidences to back up its claim.
The result is a \emph{beta distribution} $\Beta(\theta,N)$, where
\textbf{Case 1} corresponds to $N=0$, 
and
\textbf{Case 2} corresponds to $N=\infty$. 
There are also more reasonable cases, such as $N=1000$, which says it has seen 500 cases each for
heads and tails (because $\theta=0.5$).

\begin{conv}
 My notation of beta distribution is rather unconventional.
 Due to historical reasons,
 traditional notations for such a distribution is $\Beta(\alpha,\beta)$
 parameterized by a number of successes $\alpha$ and failures $\beta$,
 which is equivalent through $\theta=\frac{\alpha}{\alpha+\beta}$ and $N=\alpha+\beta$.
 Note that the same distribution sometimes has different notations depending on the literature.
\end{conv}

\begin{ex}
 Imagine after $\hat{N}=1000$ trials, I obtained 520 heads
 and therefore an empirical success ratio $\hat{\theta}=0.52$.
 I do not know the true value of the success ratio $\theta$;
 only its distribution from $\Beta(\hat{\theta}, \hat{N})$ instead.
 The true success ratio is distributed \emph{around} 0.52.
 This is formalized as follows:
 \[
 x\sim p(x|\theta)=\bern(\theta),\quad \theta\sim p(\theta|\hat{\theta}, \hat{N})=\Beta(\hat{\theta}, \hat{N}).
 \]
\end{ex}

Use of such an additional distribution is called \emph{hierarchical modeling}.
Every hierarchical modeling follows this pattern:
It adds a count parameter $N$, and considers the distribution of a parameter of a distribution.

\begin{conv}
 The counts $N,\alpha,\beta$, etc., can be extended to continuous values and they are called \emph{pseudocounts}.
\end{conv}

Notice that the naive model shown above is \emph{undefined} when we have not made any observations yet,
which is a sign of a frequentist approach.
However, in a Bayesian, \emph{subjective} view of probability (\refconv{conv:belief}),
anyone can have a belief, even before the experimentation.
A prior distribution (\refconv{conv:prior}) is a representation of
this default/prior belief/assumption on a random variable.
If the prior assumption happens to be accurate, reasoning from data tends to be more accurate.
If you have no justifiable prior assumption, then you should follow the principle of maximum entropy (\reftheo{theo:maxent}),
otherwise your default assumption is illegitimately biased and the resulting learning would become inefficient
because it would ignore a possible hypothesis that is otherwise considered equally important \cite{jaynes1968prior}.

\subsection{Reasoning with Conjugate Prior}
\label{sec:conjugate}

Bayesian statistics simplify the reasoning over complex hierarchical models
using a mathematical trick called \emph{conjugate prior distributions}.
Conjugacy has a property convenient for proofs:
When a prior distribution $p(\theta)$ is a conjugate of a posterior distribution $p(\theta|x)$
for a generative distribution $p(x|\theta)$, computing $p(\theta|X)$, $p(x|\theta)$, and $p(x|X)$ is easy.

Note that Bayesian reasoning leans toward the ``pure'' side of mathematical statistics outside the context of machine learning.
The main difference from machine learning literature and the standard discourse in Bayesian reasoning is that
there is no transformations (linear or non-linear) from data to parameters.
The parameters and the data, therefore, are always in the same unit of measurement (e.g., $\si{kg}$, $\si{m/s^2}$).
For example, when you have data of mass of things, then you estimate the distribution of their mass.

As is commonly the case with machine learning and statistics,
many textbooks on Bayesian inference lack a general, compact, yet concrete description of the entire procedure,
spending too much time on explaining specific examples.
This section provides a simple and general procedure that can be followed in a fill-in-the-blank style.
I demonstrate several examples in the following subsections (\refsec{sec:conjugate-examples}).
The procedure contains statistical modeling (\refconv{conv:statistical-modeling}),
but there are additional steps for handling conjugate priors.

I first introduce a few basic mathematical concepts.

\begin{defi}
 Distributions $p(\rx_1)=f(\rx_1, \theta_1,\ldots,\theta_N)$, $p(\rx_2)=f(\rx_2, \phi_1,\ldots,\phi_N)$ are of the \emph{same family} if
 they have the same functional form $f$ except the parameters $\theta_1,\ldots,\theta_N$ and $\phi_1,\ldots,\phi_N$.
\end{defi}
\begin{defi}
 A distribution is a \emph{conjugate} of another distribution when they are of the same variable and of the same family.
 The two distributions are then \emph{conjugates}.
\end{defi}
\begin{ex}
 Two Gaussian distributions
 $p(x)=\N(0,1)$ and $p(y)=\N(2,3)$ are of the same family (0,1,2,3 are the parameters).
 $p(x)$ and $q(x)=\N(4,5)$ are conjugates.
 $p(x)$ and $p(x|z)=\N(2z+1,1)$ are also conjugates.
\end{ex}

\begin{conv}
 Let $x$ be an observable and $z$ be a latent.
 When a prior distribution $p(z)$ is a conjugate of a posterior distribution $p(z|x)$,
 $p(z)$ is a \emph{conjugate prior distribution} \uline{for}
 a generative distribution $p(x|z)$ (note: not \uline{of}).
\end{conv}

We should add a few more conventions for prior distributions.
Recall that statistical modeling (\refconv{conv:statistical-modeling})
did not specify how to choose the prior parameters.
Prior distributions are classified into three subsets \cite{gelman1995bayesian}:

\begin{conv}
 A prior distribution is \emph{informative} if its parameters are chosen by
 the domain knowledge.
\end{conv}

\begin{conv}
 A prior distribution is \emph{non-informative} if its parameters are selected
 by the principle of maximum entropy due to the lack of such domain knowledge.
\end{conv}

\begin{conv}
 A non-informative prior is \emph{improper} if
 the prior does not integrate to 1 due to the entropy maximization.
\end{conv}

Now we describe the general procedure for Bayesian reasoning.
As the first step, we discuss a simple case where there is only one unknown parameter $\theta$.
As you can see below,
the main difficulty of Bayesian reasoning is finding the appropriate prior distribution
and proving it, which is primarily of mathematical nature rather than computational.

\begin{conv}[Bayesian Reasoning with a single unknown parameter]
 \label{conv:bayesian-reasoning}
 Bayesian reasoning is a form of statistical modeling (\refconv{conv:statistical-modeling}) applied as follows:
 \begin{enumerate}
  \item Observables: $n$ observations $X=(x_1,\ldots,x_n)$.
  \item Latents: A parameter $\theta$.
  \item Causal dependency:
        Assume each observation is \iid given $\theta$, i.e.,
        $x_i\perp x_j \mid \theta$ and $p(x_i|\theta)=p(x_j|\theta)$.
        In other words, $p(X)=\sum_{\theta} p(\theta)\prod_i p(x_i|\theta)$.
  \item Choose a distribution family and its parameters.
        \begin{enumerate}
         \item Choose the family for $p(x_i|\theta)$.
         \item Choose the family and the parameters for $p(\theta)$.
         \item Choose the family and the parameters for $p(\theta|X)$.
        \end{enumerate}
  \item Using held-out data,
        verify the hypothesis made above with a \emph{(posterior) predictive distribution} $p(x|X)$ of a future observation $x$.
 \end{enumerate}
 In Bayesian reasoning, we have an additional restriction:
 We must prove, by derivation, that
 a prior distribution $p(\theta)$ and the posterior distribution $p(\theta|X)$ are conjugates.
 The proof is necessary
 for verifying the hypothesis (\reflines{line:statistical-modeling-step1}{line:statistical-modeling-step4})
 with held-out data (\refline{line:statistical-modeling-step5}),
 which is the core of the scientific methods.
 Although the proof is the most mathematically elaborate part,
 it can be roughly summarized as follows:
 \begin{enumerate}
  \setcounter{enumi}{5}
  \item Write down $p(x_i|\theta)$.
  \item Write down $p(X|\theta) = p(x_1,\ldots,x_n|\theta) = \prod_i p(x_i|\theta)$.
  \item Write down $p(\theta)$.
  \item Derive $p(\theta|X)$ and prove that it is a conjugate of $p(\theta)$.
        This is done in one of the following manners:
        \begin{enumerate}
         \item Use $p(\theta|X)=\frac{p(X|\theta)p(\theta)}{p(X)}\propto p(X|\theta)p(\theta)$ with $p(X)$ being constant.
               Ignore constant factors and match $\theta$'s coefficients in the result with the pdf of $\theta$.
         \item Derive $p(\theta,X) = p(X|\theta)p(\theta)$,
               derive $p(X)=\int p(\theta,X)d\theta$,
               then derive $p(\theta|X)=\frac{p(\theta,X)}{p(X)}$.
        \end{enumerate}
 \end{enumerate}

 Finally,
 we derive the predictive distribution $p(x|X)$ for a future data $x$ given historical data $X$ as follows:
 $p(x|X)=\int p(x|\theta,X)p(\theta|X)d\theta=\int p(x|\theta)p(\theta|X)d\theta$.
 This is using the fact that $x$ does not depend on $X$ given $\theta$.
 Using held-out data $x'_1,\ldots x'_M$, compute $p(X'|X)=\prod_i p(x=x'_i|X)$
 (or, alternatively, $\log p(X'|X)=\sum_i \log p(x=x'_i|X)$).
\end{conv}

This procedure can be extended to multi-parameter case quite easily.
The goal of multi-parameter Bayesian reasoning is
to obtain the joint posterior distribution of the parameters $p(\theta_1\ldots,\theta_N|X)$
and subsequently obtain the predictive distribution $p(x|X)$ to test the hypothesis.

The overall procedure of Bayesian reasoning with multiple parameters is also same as the single-parameter scenario.
Using a single parameter reasoning process for the distribution family of each of $\theta_1,\ldots \theta_N$,
one should derive $p(\theta_i|X,\theta_1,\ldots,\theta_{i-1})=\frac{p(X|\theta_1,\ldots,\theta_{i-1})p(\theta_i|\theta_1,\ldots,\theta_{i-1})}{p(X|\theta_1,\ldots,\theta_{i-1})}$
for each $i$, starting from $p(\theta_1|X)$.
Therefore, the prerequisite for this procedure is that
you already know how to perform the single-parameter Bayesian reasoning for individual parameters.

The only difference between single- and multi-parameter cases is that
the process involves decomposing a more complex hierarchical model,
and that it requires a \emph{joint} prior distribution which,
for mathematical convenience,
tends to be an uninformative improper prior (although this is not always necessary).
Although uninformative priors are not particularly helpful in the reasoning,
this is not an issue because
the effect/importance of the choice of prior distributions diminishes
as the depth of the model hierarchy increases.

\subsection{Exponential Family of Distributions}

Many textbooks covers that conjugate priors tend to be found in the \emph{exponential family of distributions}
due to Pitman-Koopman-Darmois theorem \cite{pitman1936sufficient,koopman1936distributions,darmois1935exhaustive},
but do not emphasize enough that it only applies to distributions with a fixed support (fixed minima and maxima).
There are many exceptions, namely Uniform, Pareto, Power, or Generalized Pareto distributions.

\begin{defi}
 A distribution belongs to the \emph{exponential family of distributions}
 when it is of the form $p(x|\vtheta) = A(\vtheta)B(x)e^{C(\vtheta)\cdot D(x)}$
 where $\vtheta$ is a parameter vector,
 $A(\vtheta), B(x)$ are scalars, and $C(\vtheta), D(x)$ are vectors.
\end{defi}

\begin{ex}
 $\N(x| \mu,\sigma)=\frac{1}{\sqrt{2\pi\sigma^2}}e^{-\frac{(x-\mu)^2}{2\sigma^2}}$
 belongs to the exponential family by $\vtheta = (\mu,\sigma)$,
 $A(\vtheta)=\frac{1}{\sqrt{2\pi\sigma^2}}$,
 $B(x)=1$,
 $C(\vtheta)=(\frac{-1}{2\sigma^2}, \frac{2\mu}{2\sigma^2}, \frac{-\mu^2}{2\sigma^2})$,
 $D(x)=(1, x, x^2)$.
\end{ex}

The parameters $\vtheta$ are called a \emph{sufficient statistic} because,
if you have it, you no longer have to explicitly store the data of individual trials
(e.g., a sequence of successes / failures).
There is a related concepts called \emph{complete} and \emph{ancilliary statistic}
which are out of the scope of this memo.
Formally,

\begin{defi}
 Let $X=(x_1,\ldots,x_n)$  be \iid random variables following a distribution $p(\rx|\vtheta)$.
 A scalar or a vector function $T(X)$ is called a \emph{statistic} of $X$.
 A statistic $T(X)$ is \emph{sufficient} for $\vtheta$ (or for a distribution family $p(\rx|\vtheta)$)
 iff $X\perp \vtheta \mid T(X)$. \cite{rohatgi2015introduction}
\end{defi}

\begin{theo}[Pitman-Koopman-Darmois theorem \cite{pitman1936sufficient,koopman1936distributions,darmois1935exhaustive}]
 Let $X=(x_1,\ldots,x_n)$  be \iid random variables following a distribution $p(\rx|\vtheta)$ which has a fixed support.
 Let $T(X)$ be a sufficient statistics of $X$.
 $T(X)$ is a fixed sized vector iff $p(\rx|\vtheta)$ is in the exponential family.
\end{theo}

\begin{ex}
Gaussian distribution $N(\mu,\sigma^2)$ follows this theorem because its support is the whole $\R$.
Uniform distribution $U(l,u)$ are not covered by this theorem because its support $[l,u]$ changes due to its parameters.
(Same in Pareto distributions etc.)
\end{ex}

\begin{theo}
 If the prior is in the exponential family, so does the posterior.
\end{theo}

\begin{proof}
 Given $n$ \iid observations $X=(x_1\ldots x_n)$,
 \begin{align*}
  \textstyle
  p(X|\vtheta)=\prod_i p(x_i|\vtheta) &\propto A(\vtheta)^n e^{C(\vtheta)^\top \sum_i D(x_i)}\\
  p(\vtheta)&\propto A(\vtheta)^N e^{C(\vtheta)^\top M}\\
  p(\vtheta|X)&\propto p(X|\vtheta)p(\vtheta)\quad (\text{\reftheo{theo:bayes}})\\
              &\propto A(\vtheta)^{N+n} e^{C(\vtheta)^\top (M+\sum_i D(x_i))}.
 \end{align*}
\end{proof}

\paragraph{Further notes:}

Stochastic neural networks, also called Bayesian neural networks,
can sample the parameters of the latent distributions multiple times for the same data \cite{jospin2022hands}.
VAE is already such an example, which has a stochastic activations.
Other networks have stochastic weights that should be sampled each time \cite{kendall2017uncertainties}.
Use of conjugate priors is typically limited to the pure Bayesian hypothesis testing settings.
However, a recent work \cite{pavel2020conjugate} showed how to use conjugate priors
for parameter updates in a neural network.

\subsection{Examples}
\label{sec:conjugate-examples}

In each subsection in the following pages,
we show an example scenario that requires a Bayesian reasoning,
and how to construct a proof for each scenario.
Each proof is quite compact, always presented in half a page unlike existing literature.
In each proof, we follow the items in \refconv{conv:bayesian-reasoning} in this order: 1,2,3,4,(6,7,8,9),5.
The odd order is due to (6,7,8,9) being an additional process inside step 4.
In the following, we skip item 1 and 3 because they are always same as \refconv{conv:bayesian-reasoning}.

\newpage

\subsubsection{Gaussian $\N(\mu,\sigma^2)$ with Unknown $\mu$ and Known $\sigma^2$}

\def\var{\sigma^2}
\def\evar{\bar{\sigma}^2}
\def\emu{\bar{\mu}}

\begin{ex}
 We have $n=20$ data points $X=(x_1,\ldots,x_{20})$ that follows $p(x_i|\mu)=\N(\mu,\var=1.2)$,
 where I don't know $\mu$.
 I believe $\mu$ is distributed somewhere around 5.2 with variance 2,
 i.e., a prior assumption $p(\mu)=\N(5.2, 2)$.
 Can I improve $p(\mu)$ using data?
\end{ex}

\begin{enumerate}
 \setcounter{enumi}{1}
 \item Latents: $\mu$.
 \setcounter{enumi}{3}
 \item Distribution family and parameters:
       \begin{enumerate}
        \item $p(x_i|\mu)=\N(x_i|\mu,\var)$. ($\var$ is a known constant)
        \item $p(\mu)=\N(\mu_0, \var_0/n_0)$. (often $\var_0=\var$)
        \item $p(\mu|X)=\N(\mu_n, \var_n/(n+n_0))$.
       \end{enumerate}
 \setcounter{enumi}{5}
 \item $p(x_i|\mu)=(2\pi\var)^{-\frac{1}{2}}\exp -\frac{(x_i-\mu)^2}{2\var}$.
 \item $p(X|\mu)=(2\pi\var)^{-\frac{n}{2}}\exp -\frac{\sum_i (x_i-\mu)^2}{2\var}$.
 \item $p(\mu)=(2\pi\var_0/n_0)^{-\frac{1}{2}}\exp -\frac{(\mu-\mu_0)^2}{2\var_0/n_0}$.
 \item Using the first strategy. $p(\mu|X)\propto p(X|\mu)p(\mu)=$
       \begin{align*}
        &\textstyle(2\pi\var_0/n_0)^{-\frac{1}{2}}(2\pi\var)^{-\frac{n}{2}}\exp \parens{-\frac{\sum_i (x_i-\mu)^2}{2\var}-\frac{(\mu_0-\mu)^2}{2\var_0/n_0}}\\
        &\textstyle\propto \exp \parens{-\frac{n\mu^2 -2\mu \sum_i x_i}{2\var}-\frac{\mu^2 -2\mu \mu_0}{2\var_0/n_0}}\\
        &\textstyle\propto \exp \parens{-\frac{\mu^2 -2\mu \emu}{2\var/n}-\frac{\mu^2 -2\mu \mu_0}{2\var_0/n_0}}
        \textstyle\propto \exp -\frac{\mu^2-2\mu_n\mu}{2\var_n/(n+n_0)},\\
        &\textstyle \mu_n = \frac{{\emu}n/{\var}+{\mu_0 n_0}/{\var_0}}{{n}/{\var}+{n_0}/{\var_0}}, (n+n_0)/\var_n = {n}/{\var}+{n_0}/{\var_0}.
       \end{align*}
       where $\emu=\frac{\sum_i x_i}{n}$.
       In other words, the new mean $\mu_n$
       is the weighted mean of $\mu_0$ and $\emu$ adjusted by the scale and the pseudocount difference (${n}/{\var}$, ${n_0}/{\var_0}$).
 \setcounter{enumi}{4}
 \item $p(x|\mu)p(\mu|X) \propto \exp (A(\mu)x^2+B(\mu)x+C(\mu))$.
       This takes a form of a Gaussian pdf with regard to $x$.
       Since the integration with $\mu$ does not change this,
       the result $p(x|X)$ should also be a Gaussian, i.e., for some $\mu_{\text{pred}},\var_{\text{pred}}$,
       \[
       p(x|X)=\int p(x|\mu)p(\mu|X)d\mu \propto \N(\mu_{\text{pred}},\var_{\text{pred}}).
       \]
       \begin{align*}
        \text{Then}\
        \mu_{\text{pred}}
        = \E_{p(x|X)}[x]
        = \E_{p(\mu|X)p(x|\mu)}[x]
        = \E_{p(\mu|X)}[\mu]
        = \mu_n.
       \end{align*}
       \begin{align*}
        &\var_{\text{pred}}
        = \Var_{p(x|X)}[x]
        = \E_{p(x|X)}[(x-\mu_n)^2]\\
        &=\E_{p(x|X)}[(x-\mu+\mu-\mu_n)^2]\\
        &=\E_{p(x|\mu)p(\mu|X)}[(x-\mu)^2 + 2(x-\mu)(\mu-\mu_n) + (\mu-\mu_n)^2]\\
        &= \E_{p(\mu|X)}[\var] + 0 + \E_{p(x|\mu)}[\var_n] \quad =\var + \var_n.
       \end{align*}
\end{enumerate}

Note: A non-informative improper prior distribution is obtained by the limit of $\var_0\to \infty$ or $n_0\to 0$:

\begin{align*}
 p(\mu)=(2\pi\var_0)^{-\frac{1}{2}}\exp -\frac{(\mu-\mu_0)^2}{2\var_0/n_0}
 &\to \mathrm{Const.}\\
 p(\mu|X)=\N(\mu_n, \var_n/(n+n_0))
 &\to \N(\emu, \var/n).
\end{align*}

\newpage

\subsubsection{Gaussian $\N(\mu,\sigma^2)$ with Known $\mu$ and Unknown $\sigma^2$}

\begin{ex}
 We have $n=20$ data points $X=(x_1,\ldots,x_{20})$ that follows $p(x_i|\var)=\N(\mu=5.2,\var)$, where I don't know $\var$.
 I'm an expert on this dataset and
 my long experience tells me that $\var$ should be around 2,
 thus I have a prior assumption $p(\var)=\invchi(1000, 2)$.
 I'm proud of this belief, but can I improve $p(\var)$ using data?
\end{ex}

\begin{enumerate}
 \setcounter{enumi}{1}
 \item Latents: $\var$.
 \setcounter{enumi}{3}
 \item Distribution family and parameters:
       \begin{enumerate}
        \item $p(x_i|\var)=\N(x_i|\mu,\var)$. ($\mu$ is a known constant)
        \item $p(\var)=\invchi(n_0, \var_0)$.
        \item $p(\var|X)=\invchi(n+n_0, \var_n)$.
       \end{enumerate}
 \setcounter{enumi}{5}
 \item $p(x_i|\var)=(2\pi\var)^{-\frac{1}{2}}\exp -\frac{(x_i-\mu)^2}{2\var}$.
 \item $p(X|\var)=(2\pi\var)^{-\frac{n}{2}}\exp -\frac{\sum_i (x_i-\mu)^2}{2\var}$.
 \item $p(\var)=\frac{\parens{\var_0 \frac{n_0}{2}}^{\frac{n_0}{2}}}{\Gamma(\frac{n_0}{2})}(\var)^{-(\frac{n_0}{2}+1)}\exp -\frac{n_0\var_0}{2\var}$.
 \item Using the first strategy. $p(\var|X)\propto p(X|\var)p(\var)$
       \begin{align*}
        &\textstyle\propto (\var)^{-(\frac{n_0+n}{2}+1)}\exp -\frac{n_0\var_0+n\evar}{2\var}\\
        &\textstyle\propto (\var)^{-(\frac{n_0+n}{2}+1)}\exp -\frac{(n_0+n)\frac{n_0\var_0+n\evar}{n_0+n}}{2\var}\\
        &\textstyle\propto \invchi\parens{n_0+n, \var_n=\frac{n_0\var_0+n\evar}{n_0+n}} = p(\var|X)
       \end{align*}
       where $\evar=\frac{\sum_i (x_i-\mu)^2}{n}$.
       In other words, the new variance $\var_n$ updated from the prior variance $\var_0$
       is the weighted mean of $\var_0$ and $\evar$ adjusted by the number of samples.
 \setcounter{enumi}{4}
 \item 
       $p(x|\var)p(\var|X) \propto \exp (A(\var)x^2+B(\var)x+C(\var))$.
       This takes a form of a Gaussian pdf with regard to $x$.
       Since the integration with $\var$ does not change this,
       the result $p(x|X)$ should also be a Gaussian, i.e., for some $\mu_{\text{pred}},\var_{\text{pred}}$,
       \[
       p(x|X)=\int p(x|\mu)p(\mu|X)d\mu \propto \N(\mu_{\text{pred}},\var_{\text{pred}}).
       \]
       \begin{align*}
        \text{Then}\
        \mu_{\text{pred}}
        = \E_{p(\var|X)p(x|\var)}[x]
        = \E_{p(\var|X)}[\mu]
        = \mu.
       \end{align*}
       \begin{align*}
        &\var_{\text{pred}}
        = \Var_{p(x|X)}[x]
        = \E_{p(x|X)}[(x-\mu)^2]\\
        &= \E_{p(\var|X)p(x|\var)}[(x-\mu)^2]\\
        &= \E_{p(\var|X)p(x|\var)}[(x-\E_{p(x|\var)}[x])^2]\\
        &= \E_{p(\var|X)}[\Var_{p(x|\var)}[x]]\\
        &\textstyle = \E_{p(\var|X)}[\var] = \frac{(n_0+n)\var_n}{n_0+n+2}
       \end{align*}
       The last line is due to the mean of $\invchi$.
\end{enumerate}

Note: A non-informative improper prior distribution is obtained by the limit of $n_0\downto 0$:

\begin{align*}
 p(\var)
 =(\var)^{-(\frac{n_0}{2}+1)}\exp -\frac{n_0\var_0}{2\var}
 &\to[n_0\downto 0] (\var)^{-1}\exp 0 = (\var)^{-1}.\\
 p(\var|X)=\invchi(n+n_0, \var_n)
 &\to[n_0\downto 0] \invchi(n, \evar).
\end{align*}

\newpage

\subsubsection{Gaussian $\N(\mu,\sigma^2)$ with Unknown $\mu$ and Unknown $\sigma^2$}

\begin{enumerate}
 \setcounter{enumi}{1}
 \item Latents: $\mu,\var$.
 \setcounter{enumi}{3}
 \item Distribution family and parameters:
       \begin{enumerate}
        \item $p(x_i|\mu,\var)=\N(x_i|\mu,\var)$.
        \item $p(\mu|\var)=\N(\mu_0, \var/n_0)$, $p(\var)=\invchi(n'_0, \var_0)$.

              Often $n_0=n'_0$, but in this article I used the most general form that allows having
              a different amount of confidence in the prior of $\mu$ and the prior of $\var$.

        \item $p(\mu|\var,X)=\N(\mu_n, \var/(n+n_0))$,

              $p(\var|X)=\invchi(n+n'_0, \var_n)$.
       \end{enumerate}
 \setcounter{enumi}{5}
 \item $p(x_i|\mu,\var)=(2\pi\var)^{-\frac{1}{2}}\exp{-\frac{(x_i-\mu)^2}{2\var}}$.
 \item $p(X|\mu,\var)=(2\pi\var)^{-\frac{n}{2}}\exp{-\frac{\sum_i (x_i-\mu)^2}{2\var}}$.
 \item $p(\mu,\var)=p(\mu|\var)p(\var)$
       \begin{align*}
        &\textstyle\propto (\var/n_0)^{-\frac{1}{2}}\exp{-\frac{(\mu-\mu_0)^2}{2\var/n_0}} (\var)^{-(\frac{n'_0}{2}+1)}\exp{-\frac{n'_0\var_0}{2\var}}\\
        &\textstyle\propto (\var)^{-(\frac{n'_0+1}{2}+1)}\exp{-\frac{n'_0\var_0+n_0(\mu-\mu_0)^2}{2\var}}
       \end{align*}

 \item Using the first strategy.
       First, note that $\sum_i (x_i-\mu)^2$
       \begin{align*}
        &\textstyle = \sum_i (x_i-\emu+\emu-\mu)^2
        \textstyle = \sum_i (x_i-\emu)^2 + 0 + n(\emu-\mu)^2\\
        &\textstyle = n\evar + n(\emu-\mu)^2.
        \ (\emu=\frac{\sum_i x_i}{n},
        \evar=\frac{\sum_i (x_i-\emu)^2}{n})
       \end{align*}

       Then $p(\mu|\var,X)p(\var|X)\propto p(X|\mu,\var)p(\mu,\var)$
       \begin{align*}
        &\textstyle \propto(\var)^{-(\frac{n_0+n+1}{2}+1)}\exp{-\frac{n'_0\var_0+n_0(\mu-\mu_0)^2}{2\var} -\frac{n\evar + n(\emu-\mu)^2}{2\var}}\\
        &\textstyle \propto {\parens{\frac{\var}{n+n_0}}^{-\frac{1}{2}}\exp{-\frac{(\mu-\mu_n)^2}{2\var/(n+n_0)}}} \\
        &\textstyle \hspace{0.3\linewidth}\cdot   {(\var)^{-(\frac{n'_0+n}{2}+1)}\exp{-\frac{(n+n'_0)\var_n}{2\var}}}\\
        &\textstyle \propto\N(\mu_n, \var/(n+n_0)) \cdot \invchi(n+n'_0, \var_n)
       \end{align*}
       where
       $\mu_n=\frac{n\emu + n_0 \mu_0}{n+n_0}$, and $(n+n'_0)\var_n$
       \begin{align*}
        &\textstyle = n\evar + n'_0\var_0 + n \emu^2 + n_0\mu_0^2 - (n+n_0)\mu_n^2\\
        &\textstyle = n\evar + n'_0\var_0 + \frac{nn_0}{n+n_0}(\emu-\mu_0)^2.
       \end{align*}

 \setcounter{enumi}{4}

 \item $p(x|X)=\N(\mu_{\text{pred}},\var_{\text{pred}})$, where $\mu_{\text{pred}}=\E_{p(x|X)}[x]$
       \begin{align*}
        &
        = \E_{p(\mu,\var|X)p(x|\mu,\var)}[x]
        = \E_{p(\mu,\var|X)}[\mu]\\
        &
        = \E_{p(\var|X)p(\mu|\var,X)}[\mu]
        = \E_{p(\var|X)}[\mu_n]
        = \mu_n.
        \\
        &\var_{\text{pred}}
        = \Var_{p(x|X)}[x]
        = \E_{p(x|X)}[(x-\E_{p(x|X)}[x])^2]\\
        &= \E_{p(x|X)}[(x-\mu_n)^2]\\
        &= \E_{p(x|\mu,\var)p(\mu|\var,X)p(\var|X)}[(x-\mu+\mu-\mu_n)^2]\\
        &= \E_{p(x|\mu,\var)p(\mu|\var,X)p(\var|X)}[(x-\mu)^2+(\mu-\mu_n)^2]\\
        &=
        \E_{p(x|\mu,\var)p(\mu|\var,X)p(\var|X)}[(x-\E_{p(x|\mu,\var)}[x])^2]\\
        &
        +
        \E_{p(\mu|\var,X)p(\var|X)} [(\mu-\E_{p(\mu|\var,X)}[\mu])^2]\\
        &=
        \E_{p(\mu|\var,X)p(\var|X)}[\var]
        +
        \E_{p(\var|X)}[\var/(n+n_0)]\\
        &=
        \E_{p(\var|X)}[\var]
        +
        \E_{p(\var|X)}[\var/(n+n_0)]\\
        &=
        \textstyle
        \E_{p(\var|X)}[\var\frac{n+n_0+1}{n+n_0}]
        =
        \textstyle
        \frac{n'_0+n}{n'_0+n+2}\frac{n+n_0+1}{n+n_0}\var_n.
       \end{align*}
       The rest is omitted.
\end{enumerate}

\newpage
\subsubsection{Pareto with Known $\alpha$ and Unknown Lower Bound $l$}

\begin{ex}
 Checking laptop prices online,
 I found $n=20$ offers and I believe it follows $p(x_i|l)=\pareto(\alpha=1.2, l)$
 for some cheapest / minimum price $l$ that I want to know.
 I know conservatively a laptop should cost at least $l_0=100$ USD, i.e., a prior assumption.
 Can we improve $l_0$ using data?
\end{ex}

\begin{enumerate}
 \setcounter{enumi}{1}
 \item Latents: $l$.
 \setcounter{enumi}{3}
 \item Distribution family and parameters:
       \begin{enumerate}
        \item $p(x_i|l)=\pareto(x_i|\alpha,l)$, with $0<\alpha$ known, $0<l<x_i$.
        \item $p(l)=\power(\alpha n_0,l_0)$, $0<l<l_0$.
        \item $p(l|X)=\power(\alpha (n_0+n),l_n)$, $0<l<l_n$.
       \end{enumerate}
 \setcounter{enumi}{5}
 \item $p(x_i|l)=\alpha l^\alpha x_i^{-\alpha-1}$ where $0<l<x_i$.
 \item $p(X|l)=\alpha^n l^{\alpha n} \prod_i x_i^{-\alpha-1}$ where $0<l<\min_i x_i=\el$.
 \item $p(l)= \alpha n_0 l_0^{-\alpha n_0}l^{\alpha n_0-1}$ where $0<l<l_0$.
 \item Using the second strategy.
       Let $l_n=\min (l_0, \el)$.
       \begin{align*}
        p(l,X)
        &=p(X|l)p(l)\\
        &= {\alpha^n \prod_i x_i^{-\alpha-1}} {\alpha n_0l_0^{-\alpha n_0}} l^{\alpha (n+n_0)-1}\\
        &= A l^{\alpha (n+n_0)-1}. \  (0<l<l_n)\\
        p(X)
        &\textstyle
        =\int_{0}^{l_n} p(l,X)dl
        =\frac{A}{\alpha (n+n_0)} l_n^{\alpha (n+n_0)} - 0.\\
        p(l|X)&
        =\frac{p(l,X)}{p(X)}
        =\alpha (n+n_0) l_n^{-\alpha (n+n_0)} l^{\alpha (n+n_0)-1} \\
        &
        =\power(\alpha (n_0+n),l_n).
       \end{align*}
       In other words, the new lower bound $l_n$ updated from the prior lower bound $l_0$
       is the minimum of $l_0$ and the empirical minimum $\epi$.
 \setcounter{enumi}{4}
 \item $p(x|X)=\int_{0}^{l_n} p(x|l)p(l|X)dl$
       \begin{align*}
        &=\int_{0}^{l_n} \alpha x^{-\alpha-1} \alpha (n+n_0) l_n^{-\alpha (n+n_0)} l^{\alpha (n+n_0+1)-1} dl\\
        &=\brackets{\alpha x^{-\alpha-1} \alpha (n+n_0) l_n^{-\alpha (n+n_0)} \frac{l^{\alpha (n+n_0+1)}}{\alpha (n+n_0+1)}}_{0}^{l_n}\\
        &=x^{-\alpha-1} \frac{\alpha (n+n_0)}{n+n_0+1} l_n^{\alpha}
       \end{align*}
\end{enumerate}

Note: A non-informative improper prior distributions is obtained by the limit of $n_0\downto 0$:
\begin{align*}
 p(l)
 &\to[n_0\downto 0] \alpha n_0 l_0^{-\alpha n_0} l^{-1}\propto l^{-1}.\\
 p(l|X)
 &\to[n_0\downto 0] \power(\alpha n, \el)\\
\end{align*}

\newpage

\subsubsection{Pareto with Known Lower Bound and Unknown $\alpha$}

\begin{ex}
 Checking laptop prices online,
 I found $n=20$ offers which follow $p(x_i|\alpha)=\pareto(\alpha, l=100)$.
 I want to know $\alpha$ which tells the variability.
 I have a a prior assumption $p(\alpha)=\Gamma(2, 2)$ (Gamma distribution).
\end{ex}

\begin{enumerate}
 \setcounter{enumi}{1}
 \item Latents: $\alpha$.
 \setcounter{enumi}{3}
 \item Distribution family and parameters:
       \begin{enumerate}
        \item $p(x_i|\alpha)=\pareto(x_i|\alpha,l)$, with $0<l$ known and $0<\alpha$.
        \item $p(\alpha)=\Gamma(n_0, n_0\log \frac{l_0}{l})$.
        \item $p(\alpha|X)=\Gamma(n_0+n, (n+n_0)\log \frac{l_n}{l})$.
       \end{enumerate}
 \setcounter{enumi}{5}
 \item $p(x_i|\alpha)=\alpha l^\alpha x_i^{-(\alpha+1)}$ where $l<x_i$, otherwise 0.
 \item $p(X|\alpha)=\alpha^n l^{n\alpha} \prod_i x_i^{-(\alpha+1)}$.

       Let the geometric mean of the data be $\epi=\prod_i x_i^{\frac{1}{n}}$.

       Then
       $p(X|\alpha)= \alpha^n l^{n\alpha} \epi^{-n(\alpha+1)}\propto \alpha^n \parens{\frac{l}{\epi}}^{n\alpha}$.

 \item $p(\alpha)\propto \alpha^{n_0-1}e^{-n_0 \log \frac{l_0}{l}\alpha}=\alpha^{n_0-1}\parens{\frac{l}{l_0}}^{n_0\alpha}$.
 \item Using the first strategy.
       $p(\alpha,X)=p(X|\alpha)p(\alpha)$
       \begin{align*}
        \textstyle \propto \alpha^{(n+n_0)-1} \parens{\frac{l^{n_0}l^n}{l_0^{n_0}\epi^n}}^{\alpha}
        \hspace{-0.7em}\propto \alpha^{(n+n_0)-1} \parens{\frac{l}{l_n}}^{(n+n_0)\alpha}\hspace{-2.2em}\propto\hspace{0.7em} p(\alpha|x),
       \end{align*}
       where $l_n=\parens{l_0^{n_0}\epi^n}^{\frac{1}{n+n_0}}$.
       In other words, the new lower bound $l_n$ updated from the prior lower bound $l_0$
       is a weighted geometric mean of $l_0$ and the empirical geometric mean $\epi$.
 \setcounter{enumi}{4}
 \item Let $p(\alpha|X)=\Gamma(A=n_0+n, B=(n+n_0)\log \frac{l_n}{l})$.
       \begin{align*}
        p(x,\alpha|X)
        &\textstyle =p(x|\alpha)p(\alpha|X)\\
        &\textstyle = \alpha l^\alpha x^{-(\alpha+1)} \cdot \frac{B^A}{\Gamma(A)}\alpha^{A-1} e^{-B\alpha}\\
        &\textstyle = \frac{B^A}{x\Gamma(A)}\alpha^{A} e^{-(B+\log \frac{x}{l})\alpha}\\
        p(x|X)
        &=\int_0^\infty p(x,\alpha|X) d\alpha \\
        &= \frac{B^A}{x\Gamma(A)}\frac{\Gamma(A+1)}{(B+\log \frac{x}{l})^{A+1}}\\
        &= \frac{AB^A}{x(B+\log \frac{x}{l})^{A+1}}\\
       \end{align*}

       The final step is due to $\int_0^\infty x^ae^{-bx}dx=\frac{\Gamma(a+1)}{b^{a+1}}$
       and $\Gamma(z+1)=z\Gamma(z)$.
\end{enumerate}

Note: A non-informative improper prior distribution is obtained by the limit of $n_0\downto 0$:

\begin{align*}
 p(\alpha)&=\alpha^{n_0-1}\parens{\frac{l}{l_0}}^{n_0\alpha} &&\to[n_0\downto 0] \alpha^{-1}.\\
 p(\alpha|X)&=\Gamma(n_0+n, (n+n_0)\log \frac{l_n}{l})       &&\to[n_0\downto 0] \Gamma(n, n\log \frac{\epi}{l}).
\end{align*}

\newpage

\subsubsection{Uniform $U(l,u=l+w)$ with Unknown Width $w$}

\begin{ex}[German Tank Problem]
The Allies have captured $N=100$ Nazi tanks each of which has a serial number painted on the side, starting from $l=1$.
Currently, the maximum number observed so far is $\eu=993$.
Assuming that the number is assigned uniformly,
how many tanks were likely produced?
\end{ex}

\begin{enumerate}
 \setcounter{enumi}{1}
 \item Latents: $w$.
 \setcounter{enumi}{3}
 \item Distribution family and parameters:
       \begin{enumerate}
        \item $p(x_i|w)=U(x_i|l,l+w). \quad (w>0)$
        \item $p(w)=\pareto(n_0, w_0)$.
        \item $p(w|X)=\pareto(n_0+n, w_n)$.
       \end{enumerate}
 \setcounter{enumi}{5}
 \item $p(x_i|w)=w^{-1}$, $l<x_i<l+w$.
 \item $p(X|w)=w^{-n}$, $l<\min_i x_i < \max_i x_i<l+w$.

       Let $\ew=\max_i x_i-l$.
 \item $p(w)=n_0u_0^{n_0} w^{-n_0-1}$ where $0<w_0<w$, otherwise 0.
 \item Using the second strategy.
       $p(w,X)=p(X|w)p(w)=n_0u_0^{n_0} w^{-n-n_0-1}$ where $0<\max (w_0, \ew)<w$,
       otherwise 0. Let $w_n=\max (w_0, \ew)$.
       \begin{align*}
        p(X)
        &
        \textstyle=\int_\R p(w,X)dw
        \textstyle=\int_{w_n}^\infty p(w,X)dw\\
        &
        \textstyle=\brackets{n_0u_0^{n_0} \frac{w^{-n-n_0}}{-n-n_0}}_{w_n}^\infty
        \textstyle=\frac{n_0u_0^{n_0}}{n+n_0} w_n^{-n-n_0}.\\
        p(w|X)
        &=\frac{p(w,X)}{p(X)}=\frac{(n+n_0)w_n^{n+n_0}}{w^{(n+n_0)+1}} = \pareto(n_0+n, w_n).
       \end{align*}
       In other words, the new max $w_n$ updated from the prior max $w_0$
       is the max of $w_0$ and the empirical max width $\ew$.
 \setcounter{enumi}{4}
 \item $p(x|w)p(w|X) = \frac{(n+n_0)w_n^{n+n_0}}{w^{(n+n_0)+2}}$.

       \begin{align*}
        p(x|X)
        &=\int_{w_n}^\infty p(x,w|X)dw = \int_{w_n}^\infty p(x|w)p(w|X)dw \\
        &=0-\frac{(n+n_0)w_n^{n+n_0}}{-(n+n_0+1)w_n^{(n+n_0)+1}}.\\
        &=\frac{n+n_0}{n+n_0+1}w_n^{-1}\\
        &=U\parens{l, l+\frac{n+n_0+1}{n+n_0}w_n}.
       \end{align*}

       Note that the updated uniform posterior predictive distribution has a wider range
       than the empirical distribution $U(l,l+w_n)$,
       thus ``has the ability to extrapolate from the data'' \cite{tenenbaum1998bayesian}.
\end{enumerate}

Note: A non-informative improper prior is obtained by the limit of $w_0\downto 0$:

\begin{align*}
 p(w)=n_0w_0^{n_0} w^{-n_0-1}
 &\to[w_0\downto 0] \text{Const.}\\
 p(w|X)=\pareto(n_0+n, w_n)
 &\to[w_0\downto 0] \pareto(n, \ew).
\end{align*}

\newpage

\subsubsection{Uniform $U(l,u=l+w)$ with Unknown Lower Bound $l$}

\begin{ex}
The Allies have captured $N=100$ latest Nazi tanks each of which has a serial number painted on the side.
We know they produced $u=10000$ tanks in total.
The minimum number on these latest tanks that we observed so far is $\el=1945$.
When did they stop producing the older version?
\end{ex}

\begin{enumerate}
 \setcounter{enumi}{1}
 \item Latents: $l$.
 \setcounter{enumi}{3}
 \item Distribution family and parameters:
       \begin{enumerate}
        \item $p(x_i|l)=U(x_i|l,l+w). \quad (w>0)$
        \item $p(l)=U(u_0-w, l_0).\quad (u_0-w<l<l_0)$.
        \item $p(l|X)=U(u_n-w, l_n)$.
       \end{enumerate}
 \setcounter{enumi}{5}
 \item $p(x_i|l)=w^{-1}$, $l<x_i<l+w$.
 \item $p(X|l)=w^{-n}$, $l<\min_i x_i < \max_i x_i<l+w$.

       Let $\eu=\max_i x_i$ and $\el=\min_i x_i$.
       Then $\eu-w<l<\el$.
 \item $p(l)=(l_0-u_0+w)^{-1}$.
 \item Using the second strategy.
       $p(l,X)=p(X|l)p(l)=\text{Const}.$ where $\max(u_0-w,\eu-w)<l<\min(l_0,\el)$,
       otherwise 0.
       Let $u_n=\max (u_0, \eu)$, $l_n=\min (l_0, \el)$.
       \begin{align*}
        p(X)
        &
        \textstyle=\int_\R p(l,X)dl
        \textstyle=\int_{u_n-w}^{l_n} p(l,X)dl\\
        &
        \textstyle=\brackets{w^{-n}(l_0-u_0+w)^{-1}\cdot l}_{u_n-w}^{l_n}
        \textstyle=w^{-n}\frac{l_n-u_n+w}{l_0-u_0+w}.\\
        p(l|X)
        &=\frac{p(l,X)}{p(X)}=(l_n-u_n+w)^{-1} = U(u_n-w,l_n).
       \end{align*}
 \setcounter{enumi}{4}
 \item $p(x,l|X)=p(x|l)p(l|X) = w^{-1}(l_n-u_n+w)^{-1}.$
       Note that $l<x<l+w$ for $p(x|l)$, thus $x-w<l<x$. Therefore
       \begin{align*}
        p(x|X)
        &=\int_{\max(u_n, x)-w}^{\min(l_n,x)} p(x,l|X)dl\\
        &=\frac{\min(l_n,x)-\max(u_n, x)+w}{w(l_n-u_n+w)}\\
        &=\left\{
        \begin{array}{ll}
         \frac{l_n-x+w}{w(l_n-u_n+w)},& (u_n<x<l_n+w)\\
         w^{-1}                      ,& (l_n<x<u_n)\\
         \frac{x-u_n+w}{w(l_n-u_n+w)}.& (l_n>x>u_n+w)
        \end{array}
        \right.
       \end{align*}

\end{enumerate}

Note: A non-informative improper prior is obtained by the limit of $l_0\to\infty, u_0\to-\infty$.

\newpage

\subsubsection{Uniform $U(l,u=l+w)$ with Unknown $l,w$}

\begin{enumerate}
 \setcounter{enumi}{1}
 \item Latents: $l,w$.
 \setcounter{enumi}{3}
 \item Distribution family and parameters:
       \begin{enumerate}
        \item $p(x_i|l,w)=U(x_i|l,l+w). \quad (w>0)$
        \item $p(w)=\pareto(w|n_0,w_0), \quad (w_0<w)$

              $p(l|w)=U(l|u_0-w, l_0).\quad (u_0-w<l<l_0)$

       \end{enumerate}
 \setcounter{enumi}{5}
 \item $p(x_i|l,w)=w^{-1}$, $l<x_i<l+w$.
 \item $p(X|l,w)=w^{-n}$, $l<\min_i x_i < \max_i x_i<l+w$.

       Let $\eu=\max_i x_i$, $\el=\min_i x_i$, $\ew=\eu-\el$.

       Then $\eu-w<l<\el$ and $\ew<w$.
 \item $p(l,w)=p(l|w)p(w)=(l_0-u_0+w)^{-1}\cdot n_0w_0^{n_0} w^{-n_0-1}$

       $=(l_0-u_0+w)^{-1}\cdot A w^{-n_0-1}. \quad (A=n_0w_0^{n_0})$
 \item Using the second strategy.

       Let $u_n=\max (u_0, \eu)$, $l_n=\min (l_0, \el)$, $w_n=u_n-l_n$.
       (This implies $w_0=u_0-l_0$. Not sure if this is necessary.)
       \begin{align*}
        p(l,w,X)&=p(X|l,w)p(l|w)p(w)\\
        &=w^{-n}\cdot (w-w_0)^{-1} \cdot A w^{-n_0-1}\\
        &=A w^{-(n+n_0+1)}(w-w_0)^{-1}.\\
        p(w,X)
        &=\int_{u_n-w}^{l_n} p(l,w,X)dl
         =A w^{-(n+n_0+1)}\frac{w-w_n}{w-w_0}.\\
        p(l|w,X)
        &=\frac{p(l,w,X)}{p(w,X)}=\parens{w-w_n}^{-1}
         =U(l|u_n-w,l_n).
       \end{align*}

       Let $z=w^{-1}$, $z_n=w_n^{-1}$, $w\in[w_n,\infty]$, $z\in [0,z_n]$, $N=n+n_0$.
       Then $p(X)=\int_{w_n}^\infty p(w,X)dw$
       \begin{align*}
        &\hspace{-2em}\textstyle=\int_{z_n}^0 A z^{N+1}\frac{z^{-1}-z_n^{-1}}{z^{-1}-z_0^{-1}}\parens{\frac{dz}{-z^2}}
         \textstyle=\int_{z_n}^0-A z^{N-1}\frac{z-z_n}{z-z_0}\frac{-zz_0}{-zz_n}dz\\
        &\hspace{-2em}\textstyle=\int_0^{z_n} \parens{\frac{z^{N}}{z-z_0}-\frac{z_nz^{N-1}}{z-z_0}} A\frac{z_0}{z_n} dz
         \textstyle=\int_0^{\frac{z_n}{z_0}} \parens{\frac{z_0}{z_n} \frac{t^{N}}{1-t} - \frac{t^{N-1}}{1-t}} A\frac{z_0^{N}}{z_n} dt\\
        &\hspace{-2em}\textstyle=\parens{\frac{z_0}{z_n} \Beta\parens{\frac{z_n}{z_0}; N+1, 0} - \Beta\parens{\frac{z_n}{z_0}; N, 0}}A\frac{z_0^{N}}{z_n}\\
        &\hspace{-2em}\textstyle=\parens{\frac{w_n}{w_0} \Beta\parens{\frac{w_0}{w_n}; N+1, 0} - \Beta\parens{\frac{w_0}{w_n}; N, 0}}A\frac{w_n}{w_0^{N}}
         \textstyle=AC(N).\\
        &\textstyle p(w|X)=\frac{p(w,X)}{p(X)}=w^{-(N+1)}\frac{w-w_n}{w-w_0} C(N)^{-1}.
       \end{align*}
       Note that the posterior $p(w|X)$ is not conjugate with $p(w)=\pareto(w|n_0,w_0)$.
       However, as we see below, this does not affect the posterior predictive $p(x|X)$.

 \setcounter{enumi}{4}
 \item $p(x|X)=\iint p(x|w,l)p(l|w,X)p(w|X)dldw$
       \begin{align*}
        &\hspace{-2em}\textstyle =\iint w^{-1} \cdot (w-w_n)^{-1} \cdot w^{-(N+1)}\frac{w-w_n}{w-w_0} C(N)^{-1} dldw\\
        &\hspace{-2em}\textstyle =\iint \frac{w^{-(N+2)}}{w-w_0} C(N)^{-1} dldw
         \textstyle =\int w^{-(N+2)}\frac{w-w_n}{w-w_0} C(N)^{-1} dw \\
        &\hspace{-2em}\textstyle = \frac{C(N+1)}{C(N)} = U(x|u_n-\frac{C(N+1)}{C(N)},u_n).\\
        &\textstyle\frac{C(N+1)}{C(N)}=w_0^{-1}\frac{\frac{w_n}{w_0} \Beta\parens{\frac{w_0}{w_n}; N+2, 0} - \Beta\parens{\frac{w_0}{w_n}; N+1, 0}}{\frac{w_n}{w_0} \Beta\parens{\frac{w_0}{w_n}; N+1, 0} - \Beta\parens{\frac{w_0}{w_n}; N, 0}}.
       \end{align*}

\end{enumerate}

\newpage

\subsubsection{Bernoulli with an Unknown Success Ratio $r$}

\begin{ex}
 I've thrown a coin $N=100$ times, and the result was the head 75 times,
 giving me an empirical success ratio $\er=0.75$.
 I thought the coin is fair with some confidence equivalent to 1000 trials, but now I am in doubt.
 How much I should update my belief and suspect that the coin is rigged?
\end{ex}

\begin{enumerate}
 \setcounter{enumi}{1}
 \item Latents: $r$.
 \setcounter{enumi}{3}
 \item Distribution family and parameters:
       \begin{enumerate}
        \item $p(x_i|r)=\bern(x_i|r).\ (x_i\in \braces{0,1}, r\in [0,1])$
        \item $p(r)=\Beta(n_0, r_0)$.
        \item $p(r|X)=\Beta(n_0+n, r_n)$.
       \end{enumerate}
 \setcounter{enumi}{5}
 \item $p(x_i|r)=r^{x_i}(1-r)^{1-x_i}$
 \item $p(X|r)=r^{\sum_i x_i}(1-r)^{n-\sum_i x_i}$.
       Let $\er=\frac{1}{n}\sum_i x_i$. Then
       $p(X|r)=r^{n\er}(1-r)^{n(1-\er)}$.
 \item $p(r)\propto r^{n_0r_0-1}(1-r)^{n_0(1-r_0)-1}$.
 \item Using the first strategy.
       \[
       p(X|r)p(r)\propto r^{n_0r_0+n\er-1}(1-r)^{n_0(1-r_0)+n(1-\er)-1}.
       \]
       Let $r_n=\frac{n_0r_0+n\er}{n_0+n}$. Then
       \begin{align*}
        p(r|X)
        &\propto r^{(n_0+n)r_n-1}(1-r)^{(n_0+n)(1-r_n)-1}\\
        &=\Beta(n_0+n, r_n).
       \end{align*}
       In other words, the new success ratio $r_n$ updated from the prior success ratio $r_0$
       is a weighted average of $r_0$ and the empirical success ratio $\er$.
 \setcounter{enumi}{4}
 \item $p(x,r|X)=p(x|r)p(r|X)$
       \begin{align*}
        &\propto r^{(n_0+n)r_n+x-1}(1-r)^{(n_0+n)(1-r_n)+(1-x)-1}\\
        &p(x|X)= \int_0^1 p(x,r|X)dr \propto \Beta(A+x, B+1-x)\\
        &\textstyle = \Beta((n_0+n)r_n+x, (n_0+n)(1-r_n)+1-x).\\
        &\textstyle p(x=0|X)\propto \Beta(A, B+1)=\Beta(A,B)\frac{B}{A+B}\\
        &\textstyle p(x=1|X)\propto \Beta(A+1, B)=\Beta(A,B)\frac{A}{A+B}\\
        &\textstyle \therefore p(x=0|X)=\frac{A}{A+B}=\frac{n_0+n}{n_0+n+1}(1-r_n)\\
        &\textstyle \therefore p(x=1|X)=\frac{B}{A+B}=\frac{n_0+n}{n_0+n+1}r_n
       \end{align*}
\end{enumerate}

Note: Bernoulli has several non-informative priors
with historically important specific names.
The modeler should select the prior that corresponds to the belief.

\emph{Haldane's improper prior} is the limit of $n_0\downto 0$ and $r_0=1/2$.
It converges to a sum of Dirac's deltas ($\infty$ at $r=0$ and $r=1$), i.e.,
encodes a belief that \emph{the coin flip should be deterministic, but I do not know which result (tail/head) is true}.

\emph{Jeffery's prior} $n_0=1$ is most common,
which has a mathematical justification due to Fisher Information matrix and
encodes a moderate belief of ``I don't know''.

\emph{Bayes-Laplace prior} $n_0=2$ is oldest historically
and is due to the conventional notation for Beta distribution $B(1,1)$.

\begin{align*}
 p(r) &\to[n_0\downto 0] \delta(\theta=0)\delta(\theta=1) \propto r^{-1}(1-r)^{-1},\\
 p(r) &\to[n_0\to 1] \Beta\parens{1,1/2},\\
 p(r) &\to[n_0\to 2] \Beta\parens{2,1/2}.
\end{align*}

\newpage

\section{Distribution Zoo}
\label{sec:zoo}

Textbook sources and Wikipedia articles are not useful
because they are usually littered with unnecessary detailed information for users.
In particular, while existing textbooks and such articles describe what they are,
they do not give you an \emph{instruction} of how and when to use them,
as is done in a documentation of a program library.
Documentations of \href{https://pytorch.org/docs/stable/distributions.html}{pytorch distributions}
list plenty of mainstream distributions, but they do not contain much information for each distribution.
This section provides a down-to-earth explanation and a clear-cut instruction for how/when to use them.
As a complementary material, I also recommend \cite{leemis2008univariate} as a useful comprehensive
source of the list of distributions.

\begin{table*}[tb]
 \centering
 \begin{tabular}{|c|l|c|c|}
  \toprule
  Name & Use it for variables that are... & Long tail? & Sparse? \\
  \midrule
  \multicolumn{4}{|c|}{Continuous Distributions}\\
  \midrule
 Gaussian                  &  Unbounded and centered around the mean.       & & \\
 Gamma (and special cases) &  Monotonically increasing sum.                 & & \\
 Cauchy                    &  Tangent $\tan x$, slope, ratio between Gaussians. & Yes &\\
 Logistic                  &  Modeling the logit of a probability.          & Yes &\\
 Laplace                   &  Gaussian with outliers, or sparse (mostly 0). & Yes & Yes \\
 Horseshoe                 &  Sparse (mostly 0). (Superior to Laplace)      & Yes & Yes \\
  \midrule
  \multicolumn{4}{|c|}{Continuous Distributions whose Tails Matter}\\
  \midrule
 Uniform                   &  Lower/upper-bounded.                          & & \\
 Pareto                    &  Upper limits of something, and many other.    & Yes &\\
 Weibull/Gumbel/Fr\'echet  &  Block maxima of \iid short/Gaussian/long tail distributions. & & \\
 Generalized Pareto        &  Tail data above a certain threshold (Peaks-Over-Threshold).    & & \\
 Truncated Gaussian        &  Bounded and centered around the mean.         & & \\
  \midrule
  \multicolumn{4}{|c|}{Discrete Distributions}\\
  \midrule
    Bernoulli   &  Boolean.                                             & & \\
    Beta        &  Boolean with uncertainty.                            & & \\
    Categorical &  Unordered Categorical.                               & & \\
    Dirichlet   &  Unordered Categorical with uncertainty.              & & \\
    Binomial    &  Interval Categorical (ordinal + uniform spacing).    & & \\
  \midrule
  \multicolumn{4}{|c|}{Directional Distributions}\\
  \midrule
  von Mises-Fisher             &  A direction in a Euclid space.                           & & \\
  Riemannian Normal            &  A direction in a non-Euclid (Elliptic/Hyperbolic) space. & & \\
  \bottomrule
 \end{tabular}
\end{table*}
\subsection{Continuous Distributions}

All distributions in this subsection are instances of so-called
exponential family of distributions (\refsec{sec:conjugate}).
Gaussian distribution occupies a special place due to the Central Limit Theorem.

\paragraph{Gaussian $\N(\mu,\sigma)$}
\label{zoo:gaussian}
\begin{itemize}
 \item Use it for unbounded continuous variables.
 \item Max-entropy distribution for
       $X\in\R$ with a known $\E[X]$ and a known $\Var[X]$.
 \item The mean has a conjugate prior $\mu\sim\N(\mu_0,\sigma^2/N)$.
       Its frequentist characterization is Student's t distribution.
 \item The variance has a conjugate prior $\sigma\sim\text{Inv}\chi^2(\sigma_0,N)$.
\end{itemize}

\paragraph{Gamma $\Gamma(k, \theta)$}
\label{zoo:gamma}

\begin{itemize}
 \item Use it for a positive, continuous aggregated sum
       that increases monotonically with the same speed.
 \item For example,
       when $k\in\Z^+$, it is a wait time until the $k$-th event happens when each event occurs roughly every $\theta$ seconds.
 \item Max-entropy distribution for
       $X\in\R^+$ with a fixed $\E[X]$ and a fixed $\E[\log X]$.
 \item Scaled-Inv$\chi^2$ (Chi-Squared) distribution is a distribution of variances.
       As more observations are made, the variance $\sigma^2$ decreases and
       its inverse, the \emph{precision} $1/\sigma^2$, increases at a constant rate.
       In other words, if $X\sim \text{Scaled-Inv}\chi^2$, then $1/X\sim \Gamma$.
 \item See below for a summary of special cases.
\end{itemize}

\begin{center}
 \begin{tabular}{|c|c|c|c|}\toprule
  Special case & $X$    & $k$        & $\theta$ \\\midrule
  Gamma        & $\R^+$ & $\R^+$     & $\R^+$   \\
  Poisson      & $\Z^{0+}$ & $\Z^{0+}$ & $\R^+$   \\
  Exponential  & $\R^+$ & $k=1$      & $\R^+$   \\
  Erlang       & $\R^+$ & $\Z^{0+}$ & $\R^+$   \\\bottomrule
 \end{tabular}
\end{center}

\paragraph{Multivariate normal $\N(\mu,\Sigma)$}
\label{zoo:multivariate-normal}

\begin{itemize}
 \item Multiple random variables that correlates with each other with a covariance $\Sigma$.
 \item Is a max-entropy distribution.
 \item Its conjugate prior is Normal-inverse-Wishart distribution.
\end{itemize}

\paragraph{Cauchy $C(x_0, \gamma)$}
\label{zoo:cauchy}

\begin{itemize}
 \item Use it as a tangent $\tan X$ of a random variable.
 \item Use it as a ratio between two Gaussian random variables $X/Y$ each with mean 0.
 \item It has a longer tail than Gaussian.
 \item Max-entropy distribution for
       $X\in\R$ with $\E[\log(1+(X-x_{0})^{2}/\gamma ^{2})]=\log 4$.
 \item A Cauchy distribution has a median, but lacks the mean and the variance,
       therefore the CLT (\reftheo{theo:clt}) does not apply,
       i.e., even with an infinite sample, it does not converge to the mean.
 \item A half-Cauchy distribution $C^+(x_0, \gamma)$ has only one side of the median.
\end{itemize}

\paragraph{Logistics}
\label{zoo:laplace}

\begin{itemize}
 \item Use it for a logit of probability.
 \item It has a longer tail than Gaussian.
 \item Max-entropy distribution for
       $X\in\R$ with $\E[X]=\mu$ and $\E[ \log (e^{\frac{x-\mu}{2s}}+e^{-\frac{x-\mu}{2s}})]=1$.
\end{itemize}

\paragraph{$\text{LogNormal}(\mu,\sigma^2)$}
\label{zoo:lognormal}

\begin{itemize}
 \item Use it for a variable that is logarithm of a Gaussian variable.
 \item Max-entropy distribution for
       $X\in\R^+$ with a known $\E[\log X]$ and a known $\Var[\log X]$.
       It is different from pareto because it only assumes a known mean.
\end{itemize}

\subsection{Sparse Distributions}
\label{zoo:sparse}

Sparse distributions have a stronger concentration toward 0.
This is helpful for obtaining a distribution that is mostly 0.
\citet{carvalho2009handling} unified several sparse distributions into a single framework.
I describe Laplace and Horseshoe only.

\paragraph{Laplace}
\label{zoo:laplace}

\begin{itemize}
 \item Use it for distributions with outliers.
 \item Use it for sparse modeling. See Horseshoe prior.
 \item It has a longer tail than Gaussian.
 \item Max-entropy distribution for
       $X\in\R$ with a known $\E[X]=\mu$
       and a known $\E[|X-\mu|]=b$.
       \cite{kotz2001laplace}
 \item It is a mixture of Gaussians with $\N(\mu,\lambda^2\sigma^2)$
       and $\lambda^2\sim \text{Exp}(2)$.
\end{itemize}

\paragraph{Horseshoe}
\label{zoo:horseshoe}

\begin{itemize}
 \item Use it for sparse modeling.
 \item It has a longer tail than Gaussian.
 \item Unlike Laplace, it has an infinitely large density at 0,
       resulting in a much sparser distribution than Laplace.
 \item It is a mixture of Gaussians with $\N(\mu,\lambda^2\sigma^2)$
       and $\lambda^2\sim \text{C}^+(0,1)$.
 \item So far, it is not shown to be a maximum entropy distribution.
\end{itemize}

\subsection{Continuous Distributions whose Tails Matter}

Regular statistics are typically built around the Central Limit Theorem,
which deals with the limit behavior of a sum/average of multiple samples.
In contrast, a branch of statistics called \emph{Extreme Value Theory} \cite{beirlant2004statistics}
is built around the \emph{Extremal Limit Theorem},
a theorem that describes the limit behavior of the maximum of multiple samples.

I believe they are underrepresented in the current mainstream ML research
due to its focus on the most likely value (MAP estimate).
In essence, extreme value theory was built for predicting the \emph{least likely} worst case
that is at the edge of the distribution.
However, in decision making tasks that are traditionally handled by symbolic AI,
these rare, least likely values are often precisely what we want to know --- For example,
a predictor for a maximum/minimum value should be highly useful because
they typically focus on some form of optimization problems.
I hope to see more frequent adaptations of these distributions in the future.

\paragraph{Uniform $U(l,u)$}
\label{zoo:uniform}

\begin{itemize}
 \item Use it for continuous variables when its maximum and the minimum (an upper and a lower bound) matters.
 \item Max-entropy distribution for $X\in[l,u]$.
 \item The conjugate prior for $l$ and $u$ is a Pareto distribution.
       Uniform distribution is a rare case that has a conjugate prior
       despite not being in an exponential family of distributions (\refsec{sec:conjugate}).
\end{itemize}

A Bayesian approach for predicting the minimum/maximum of a random variable is
done by Uniform-Pareto Conjugate Prior \cite{kiefer1952sequential,degroot1970optimal,rossman1998bayes,tenenbaum1998bayesian}.
It models the variable with a uniform distribution,
and further model their upper/lower bounds with Pareto distributions.

An illustrative example of Uniform-Pareto conjugate is called a \emph{taxicab problem}:
Watching the streets in a train going through a dense city, you notice each taxi is assigned a number.
Assuming that the number is assigned uniformly, and only seeing finite taxis,
can you guess the maximum number used in this entire city?
The estimated distribution of such an upper bound is slightly higher than
the largest number you would actually observe.
You might have seen a very large number,
but it would be probably overconfident to believe that
you actually saw the largest number in this city.
This showcases an example where a Bayesian method can predicts a range slightly wider
than what is seen in the limited data, making more realistic assumption than the frequentist approach.
\citet{tenenbaum1998bayesian} observed that humans show a similar reasoning/learning behavior.

\paragraph{Pareto $\pareto(\alpha,\theta)$}
\label{zoo:pareto}

\begin{itemize}
 \item Use it for a variable that shows power law \cite{newman2005power,lin2015power}.
       See if your variable fits one of several mechanisms that cause it.
 \item Max-entropy distribution for
       $X\in[\theta,\infty)$ with a known $\E[\log X]$.
       \cite{preda1984informational}
 \item Example: Use it as a conjugate prior distribution for the maximum/minimum of a Uniform distribution.
 \item Example: \emph{Self-Organized Criticality}.
       Constant cumulative effects cause an avalanche.
       For example, the size of an earthquake (caused by accumulating stress).
 \item Example: \emph{Yule process}, in which each species in a genus will get
       an equal chance of splitting into two new species / forming a new species,
       and new species sometimes form a new genus by chance.
       For example, a taxonomy of biological species or research fields,
       or Zipf's law (vocabulary tends to diverge).
 \item Example: \emph{Preferential Attachment}.
       The size accelerates the accumulation.
       For example, sales of a book/movie tickets (driven by reputation),
       the size of social clusters and cities (larger ones attract more people),
       and wealth distribution (richer gets richer).
       Preferential attachment and Yule process are almost identical
       because a large genus gets more new species.

 \item It does not have a variance for $\alpha>2$, due to \refdef{def:expectation}.
 \item It does not have a expectation for $\alpha>1$, due to \refdef{def:expectation}.
\end{itemize}

\paragraph{Generalized Extreme Value Distributions $\mathrm{GEV}(\mu,\sigma,\gamma)$}
\label{zoo:gumbel}

\begin{itemize}
 \item Use it for the \emph{block maxima} of data, i.e.,
       maximum values of multiple blocks that contain \iid measurements.
 \item For example, the annual maximum discharge of a river.
       Each discharge is supposed to follow a Gaussian distribution,
       and the block maxima is the maximum value over the year.
       You will predict the maximum of the next year from the multi-year historical data of maxima.
\end{itemize}

CLT says that the limit average of \iid variables asymptotically follows a Gaussian distribution.
In contrast, Fisher--Tippett--Gnedenko theorem \cite{fisher1928limiting,gnedenko1943limiting} shows that
the maximum of \iid variables asymptotically follows one of three \emph{Extreme Value Distributions} (EVDs):
Gumbel, Fr\'echet, or Weibull distributions.
If each measurement follows an exponential-tail distributions (e.g. Gaussian),
then its block maxima follows a Gumbel distribution.
It follows a Fr\'echet distribution when each measurement has a heavier tail than a Gaussian,
and a Weibull distribution when it has a lighter tail.
The heaviness of a tail distribution is characterized by \emph{Extreme Value Index} (EVI) typically denoted by $\gamma$;
A Gaussian distribution has $\gamma=0$.
In other literature, a term \emph{Tail Index} $\alpha$ is also used.

EVDs are used in \emph{Block-Maxima} modeling of the maximum.
The typical application is as follows:
Given a set of time-series data $(t_i,x_i)$,
divide it into blocks with equal intervals, e.g., an hourly / daily / weekly / monthly block $t_{kM}\ldots t_{(k+1)M}$.
If you extract the maximum in each block,
then the maximum of each block and the maximum of future blocks follow EVDs.

\begin{center}
 \begin{tabular}{|c|c|}\toprule
  Special case & $\gamma$   \\\midrule
  Fr\'echet    & $\gamma>0$ \\
  Gumbel       & $\gamma=0$ \\
  Weibull      & $\gamma<0$ \\\bottomrule
 \end{tabular}
\end{center}

\paragraph{Generalized Pareto Distribution $\gp(\mu,\sigma,\xi)$}

\begin{itemize}
 \item Use it for the \emph{Peaks-Over-Threshold} modeling of the rare events that happen in the distribution tail.
 \item For example, the distribution of the excess over the safety threshold
       of the water level of a river at an embankment.
\end{itemize}

As a second limit theorem in Exterme Value Theory,
Pickands--Balkema--de~Haan theorem
\cite{pickands1975statistical,balkema1974residual} showed that
the \emph{excess} of a random variable over a certain threshold
asymptotically follows a Generalized Pareto (GP) distribution.

Exponential, Uniform, Pareto, Lomax, (reversed) Power distributions are special cases of GP distribution.
They share the characteristics that the probability density is 0 below a certain threshold $\mu$.
This makes sense when you try to predict the \emph{true} maximum from an \emph{empirical} maximum
--- The true maximum must be above the largest value an agent has seen before.

The typical application is as follows:
Given a set of time-series data $(t_i,x_i)$,
extract a subset whose $x_i$ exceeds a certain threshold $\mu$.
Then $x_i$ in the subset, as well as the future exceeding data, follow a GP distribution.

\begin{center}
 \begin{tabular}{|c|c|c|c|}\toprule
  Special case                            & $\mu$ & $\sigma$                 & $\xi$                \\\midrule
  Exponential     $\mathrm{Exp}(\lambda)$ & $0$   & $\frac{1}{\lambda}$      & $0$                  \\
  Uniform     $U(0,\sigma)$               & $0$   & $\sigma$                 & $-1$                 \\
  Pareto   $\pareto(\alpha,x_m)$          & $x_m>0$ & $\frac{x_m}{\alpha}$     & $\frac{1}{\alpha}>0$ \\
  Lomax  $\lomax(\alpha,\lambda)$         & $0$   & $\frac{\lambda}{\alpha}$ & $\frac{1}{\alpha}>0$ \\
  $-x\sim \power(\alpha,\beta)$ & $-\beta<0$ & $\frac{\beta}{\alpha}$ & $-\frac{1}{\alpha}<0$ \\\bottomrule
 \end{tabular}
\end{center}

\paragraph{Truncated Gaussian $\N(\mu,\sigma, l, u)$}
\label{zoo:truncated-gaussian}

\begin{itemize}
 \item Use it for continuous variables when its mean, variance, maximum, and minimum all matters.
 \item It can be seen as a combination of Uniform and Gaussian.
 \item Max-entropy distribution for
       $X\in[l,u]$ with a known $\E[X]$ and a known $\Var[X]$.
 \item Note that the $\mu$ and $\sigma$ are the mean/variance \emph{before the truncation}.
 \item Naive calculation method is numerically unstable.
       Use a existing statistics library to compute it.
\end{itemize}

\subsection{Discrete Distributions}
\paragraph{Bernoulli $\bern(p)$}
\label{zoo:bernoulli}

\begin{itemize}
 \item Use it for boolean variables.
 \item Probability for being true, $p$, has a conjugate prior $B(\alpha,\beta)$ or $B(p_0, N)$.
\end{itemize}

\paragraph{Beta $B(\alpha,\beta), B(p_0, N)=B(\frac{\alpha}{\alpha+\beta}, \alpha+\beta)$}
\label{zoo:beta}

\begin{itemize}
 \item Use it for a variable representing a boolean with uncertainty.
 \item Use it for a variable representing a probability of success.
 \item Do not confuse it with ``a variable representing a success,'' which is $\bern$.
 \item In the traditional notation $B(\alpha,\beta)$,
       $\alpha$ is the pseudocount of observed success and $\beta$ is the pseudocount of observed failures.
 \item The second notation $B(p_0, N)$ is instead parameterized by
       the empirical success rate and the total pseudocount.
\end{itemize}

\paragraph{Categorical $\cat(\vp)$}
\label{zoo:categorical}

\begin{itemize}
 \item Use it for categorical variables.
 \item Probability for each category, $\vp$, has a conjugate prior $\mathrm{Dir}(\valpha)$ or $\mathrm{Dir}(\vp_0, N)$.
\end{itemize}

\paragraph{Dirichlet $\dir(\valpha)$, $\dir(\vp_0,N)=\dir(\valpha/N, \sum_i \valpha_i)$}
\label{zoo:dirichlet}

\begin{itemize}
 \item Use it for categorical variables with uncertainty.
 \item Use it for a variable representing a categorical distribution.
 \item Do not confuse it with ``a variable representing a categorical choice,'' which is $\cat$.
 \item In the traditional notation $\dir(\valpha)$,
       $\valpha$ is a vector of pseudocounts for categories.
 \item The second notation $\dir(\vp_0,N)$ is instead parameterized by
       the empirical distribution and the total pseudocount.
\end{itemize}

\paragraph{Binomial $\bin(n, p)$}
\label{zoo:binomial}

\begin{itemize}
 \item Use it for an \emph{interval} variable, i.e.,
       an \emph{ordinal} categories with \emph{equal intervals} between them.
       Equal intervals imply that the distances between the categories are meaningful.
 \item Use it for a random positive integer $\rx\in \braces{0..n}$.
 \item Use it for a counter.
 \item Use it for a discrete quantity with a rough mean $np$ and a variance $np(1-p)$.
 \item The variable is typically explained as a number of success among $n$ trials, with probability of success $p$.
 \item Probability for being true, $p$, has a conjugate prior $B(\alpha,\beta)$ or $B(p_0, N)$.
 \item If they are labels, assign 0 to the first element according to the order.
 \item For example, imagine classifying the hotness of curries served in a nearby Indian curry restaurant.
       The restaurant owner ensures that
       the amount of spice they add between the hotness levels
       (``not spicy,'' ``mild spicy,'' ``very spicy,'' and ``crazy spicy'')
       is constant. Then it is probably safe to use $\bin$.
\end{itemize}

\subsection{Directional Distributions}

Directional distributions deals with \emph{directions} in a unit ball.
Related keywords: Spherical, hyperspherical, Poincare ball, Riemannian.

\paragraph{von Mises-Fisher $\text{vMF}(\vmu, \kappa)$}
\label{zoo:vmf}

\begin{itemize}
 \item Use it for a random high-dimensional unit vector $\vmu$ and spread $\kappa$.
       It assumes a Euclidean space.
 \item A max-entropy distribution.
 \item For example, word embedding.
 \item KL divergence and $\log p(x)$ yields a cosine distance.
\end{itemize}

\paragraph{Riemannian Normal $(\vmu, \kappa)$}
\label{zoo:riemannian}

\begin{itemize}
 \item Use it for a random high-dimensional unit vector $\vmu$ and spread $\kappa$.
       It generalizes vMF by assuming non-Euclidean space, e.g., hyperbolic/elliptic space.
 \item A max-entropy distribution.
 \item For example, word embedding.
\end{itemize}

\section{Peripheral Topics}

\label{sec:future-work}

Finally,
I discuss peripheral topics that I could not include in this memo
in order to limit its scope and maintain the focus.
I have varying levels of understandings of these topics and
some of the topics below have a section that I have completed but did not make it into this memo.

\emph{Markov Chain Monte Carlo} (MCMC) is
an exact Bayesian method for computing an expectation.
Simulated Annealing, a nature-inspired optimization method, is theoretically an instance of MCMC.
Due to its sequential nature,
it is difficult to leverage highly parallel accelerator (GPU)
used to build neural networks.
However, recently,
a theoretical connection was made between MCMC and
\emph{Diffusion models} \cite{sohl2015deep}, a group of methods that achieved a \sota performance
and was subsequently adapted in image-generation models including DALL-E 2 \cite{dalle2}.
Related keywords include:
Gibbs sampling, rejection sampling, importance sampling, ergodicity.

\emph{Causal Inference} \cite{pearl2018bookofwhy} is a general framework for
discovering causal relationship between random variables, while eliminating pure correlations.
A large effort was spent on pruning spurious correlations in an undirected graphical model
by interventions (additional experiments).
Since these interventions are performed sequentially,
it is also currently considered incompatible with modern machine learning frameworks.
Related keywords include:
Do-calculus, intervention, causality, causation, D-separation.

\emph{Multi-Armed Bandit} (MAB) is
a fundamental group of methods for making an optimal decision under uncertainty
while balancing the exploration and exploitation.
Related keywords include:
Upper Confidence Bound \cite{auer2002finite},
Cumulative Regret vs.\ Simple Regret \cite{feldman2014simple}, Best Arm Identification,
Monte-Carlo Tree Search \cite{kocsis2006bandit}.

\emph{Active Learning} (AL) \cite{burr2012active} is a group of methods that
tries to expand the dataset dynamically and selectively by a certain strategy.
It has connections with MAB in terms of maximizing the information gain \cite{antos2008active}.
Related keywords include:
Uncertainty Sampling, Mutual Information Maximization, Fisher Information Minimization.

\emph{Reinforcement Learning} (RL) \cite{sutton2018reinforcement,bertsekas2019reinforcement}
can be, frankly speaking, seen in various forms.
It is particularly complicated because there are three factions that claim the dominance in this field:
(1) Optimal Control community from which RL inherited various theoretical and algorithmic ideas.
(2) Symbolic MDP community who sees it as Stochastic Shortest Path problem.
(3) Psychology / connectionist community who believes RL is how living being learns from the environment.

In one sense, it is a generalization of dynamic programming.
Optimal Control community sees it as an instance of Optimal Control.
Symbolic AI community sees it as an instance of stochastic shortest path.
I personally see it as Active Learning + Supervised Learning.
Since this is a large topic, I would rather avoid discussing this topic in depth.

\fontsize{9.5pt}{10.5pt}
\selectfont

\section*{Acknowledgments}

I appreciate the proofreading by the following people (in alphabetic order):
Akash Srivastava,
Carlos N\'u\~nez Molina,
Dan Gutfreund,
Hiroshi Kajino,
Hector Palacios,
Marlyse Reeves,
Ryo Kuroiwa,
Sarath Sreedharan.


\begin{thebibliography}{90}
\providecommand{\natexlab}[1]{#1}

\bibitem[{Alquier(2021)}]{alquier2021user}
Alquier, P. 2021.
\newblock User-friendly introduction to PAC-Bayes bounds.
\newblock \emph{arXiv preprint arXiv:2110.11216}.

\bibitem[{Antos, Grover, and Szepesv{\'a}ri(2008)}]{antos2008active}
Antos, A.; Grover, V.; and Szepesv{\'a}ri, C. 2008.
\newblock {Active Learning in Multi-Armed Bandits}.
\newblock In \emph{International Conference on Algorithmic Learning Theory},
  287--302. Springer.

\bibitem[{Arjovsky, Chintala, and Bottou(2017)}]{arjovsky2017wasserstein}
Arjovsky, M.; Chintala, S.; and Bottou, L. 2017.
\newblock {Wasserstein Generative Adversarial Networks}.
\newblock In \emph{{Proc. of the International Conference on Machine Learning
  (ICML)}}, 214--223. PMLR.

\bibitem[{Asai(2022{\natexlab{a}})}]{prolog-elbo}
Asai, M. 2022{\natexlab{a}}.
\newblock {Deriving Evidence Lower BOund (ELBO) with Prolog}.
\newblock github.com/guicho271828/prolog-elbo.

\bibitem[{Asai(2022{\natexlab{b}})}]{elbonara}
Asai, M. 2022{\natexlab{b}}.
\newblock {Elbonara}.
\newblock github.com/guicho271828/elbonara.

\bibitem[{Asai et~al.(2022)Asai, Kajino, Fukunaga, and Muise}]{Asai2022}
Asai, M.; Kajino, H.; Fukunaga, A.; and Muise, C. 2022.
\newblock {Classical Planning in Deep Latent Space}.
\newblock \emph{{J. Artif. Intell. Res.(JAIR)}}, 74: 1599--1686.

\bibitem[{Auer, Cesa-Bianchi, and Fischer(2002)}]{auer2002finite}
Auer, P.; Cesa-Bianchi, N.; and Fischer, P. 2002.
\newblock {Finite-Time Analysis of the Multiarmed Bandit Problem}.
\newblock \emph{Machine Learning}, 47(2-3): 235--256.

\bibitem[{Balkema and {De Haan}(1974)}]{balkema1974residual}
Balkema, A.~A.; and {De Haan}, L. 1974.
\newblock {Residual Life Time at Great Age}.
\newblock \emph{Annals of Probability}, 2(5): 792--804.

\bibitem[{Bayes(1763)}]{bayes1763}
Bayes, T. 1763.
\newblock {An Essay towards solving a Problem in the Doctrine of Chances. By
  the late Rev. Mr. Bayes, F.R.S. communicated by Mr. Price, in a Letter to
  John Canton, A.M.F.R.S.}
\newblock \emph{Philosophical transactions of the Royal Society of London},
  (53): 370--418.

\bibitem[{Beirlant et~al.(2004)Beirlant, Goegebeur, Segers, and
  Teugels}]{beirlant2004statistics}
Beirlant, J.; Goegebeur, Y.; Segers, J.; and Teugels, J.~L. 2004.
\newblock \emph{{Statistics of Extremes: Theory and Applications}}, volume 558.
\newblock John Wiley \& Sons.

\bibitem[{Bernoulli(1713)}]{bernoulli1713lawsoflargenumbers}
Bernoulli, J. 1713.
\newblock \emph{{Ars Conjectandi: Usum \& Applicationem Praecedentis Doctrinae
  in Civilibus, Moralibus \& Oeconomicis (in Latin)}}.

\bibitem[{Bertsekas(2019)}]{bertsekas2019reinforcement}
Bertsekas, D.~P. 2019.
\newblock \emph{Reinforcement Learning and Optimal Control}.
\newblock Athena Scientific Belmont, MA.

\bibitem[{Bishop(2006)}]{bishop2006pattern}
Bishop, C.~M. 2006.
\newblock {Pattern Recognition}.
\newblock \emph{Machine Learning}, 128(9).

\bibitem[{Carvalho, Polson, and Scott(2009)}]{carvalho2009handling}
Carvalho, C.~M.; Polson, N.~G.; and Scott, J.~G. 2009.
\newblock Handling sparsity via the horseshoe.
\newblock In \emph{Proc. of the International Conference on Artificial
  Intelligence and Statistics (AISTATS)}, 73--80. PMLR.

\bibitem[{Chopra, Hadsell, and LeCun(2005)}]{chopra2005learning}
Chopra, S.; Hadsell, R.; and LeCun, Y. 2005.
\newblock {Learning a Similarity Metric Discriminatively, with Application to
  Face Verification}.
\newblock In \emph{{Proc. of IEEE Conference on Computer Vision and Pattern
  Recognition (CVPR)}}, volume~1, 539--546. IEEE.

\bibitem[{Cimatti et~al.(2003)Cimatti, Pistore, Roveri, and
  Traverso}]{cimatti2003weak}
Cimatti, A.; Pistore, M.; Roveri, M.; and Traverso, P. 2003.
\newblock {Weak, Strong, and Strong Cyclic Planning via Symbolic Model
  Checking}.
\newblock \emph{{Artificial Intelligence}}, 147(1-2): 35--84.

\bibitem[{Cortes and Vapnik(1995)}]{cortes1995support}
Cortes, C.; and Vapnik, V. 1995.
\newblock {Support-Vector Networks}.
\newblock \emph{Machine Learning}, 20(3): 273--297.

\bibitem[{Dagum and Chavez(1993)}]{dagum1993approximating}
Dagum, P.; and Chavez, R.~M. 1993.
\newblock {Approximating Probabilistic Inference in Bayesian Belief Networks}.
\newblock \emph{{IEEE Transactions on Pattern Analysis and Machine
  Intelligence}}, 15(3): 246--255.

\bibitem[{Dagum and Luby(1997)}]{dagum1997optimal}
Dagum, P.; and Luby, M. 1997.
\newblock {An Optimal Approximation Algorithm for Bayesian Inference}.
\newblock \emph{{Artificial Intelligence}}, 93(1): 1--28.

\bibitem[{Darmois(1935)}]{darmois1935exhaustive}
Darmois, G. 1935.
\newblock Sur Les Lois de Probabilites a Estimation Exhaustive.
\newblock In \emph{Comptes Rendus de l{\rq}Acad{\'e}mie des Sciences}, volume
  200, 1265--1266.

\bibitem[{de~Laplace(1812)}]{de1812probability}
de~Laplace, P.-S. 1812.
\newblock On the Probability of Causes and of Future Events, Deduced From
  Observed Events (Translated by Richard J. Pulskamp).
\newblock \emph{Th{\'e}orie Analytique des Probabilit{\'e}s}.

\bibitem[{DeGroot(1970)}]{degroot1970optimal}
DeGroot, M.~H. 1970.
\newblock \emph{{Optimal Statistical Decisions}}.
\newblock John Wiley \& Sons.

\bibitem[{Dziugaite, Roy, and Ghahramani(2015)}]{dziugaite2015training}
Dziugaite, G.~K.; Roy, D.~M.; and Ghahramani, Z. 2015.
\newblock {Training Generative Neural Networks via Maximum Mean Discrepancy
  Optimization}.
\newblock In \emph{{Proc. of the International Conference on Uncertainty in
  Artificial Intelligence (UAI)}}, 258--267.

\bibitem[{Elkan and Noto(2008)}]{elkan2008learning}
Elkan, C.; and Noto, K. 2008.
\newblock {Learning Classifiers from Only Positive and Unlabeled Data}.
\newblock In \emph{{Proc. of ACM International Conference on Knowledge
  Discovery and Data Mining (SIGKDD)}}, 213--220. ACM.

\bibitem[{Falconer(2004)}]{falconer2004fractal}
Falconer, K. 2004.
\newblock \emph{{Fractal Geometry: Mathematical Foundations and Applications}}.
\newblock John Wiley \& Sons.

\bibitem[{Feldman and Domshlak(2014)}]{feldman2014simple}
Feldman, Z.; and Domshlak, C. 2014.
\newblock {Simple Regret Optimization in Online Planning for Markov Decision
  Processes}.
\newblock \emph{{J. Artif. Intell. Res.(JAIR)}}, 51: 165--205.

\bibitem[{Fisher(1922)}]{fisher1922mathematical}
Fisher, R.~A. 1922.
\newblock {On the Mathematical Foundations of Theoretical Statistics}.
\newblock \emph{Philosophical transactions of the Royal Society of London.
  Series A, containing papers of a mathematical or physical character},
  222(594-604): 309--368.

\bibitem[{Fisher and Tippett(1928)}]{fisher1928limiting}
Fisher, R.~A.; and Tippett, L. H.~C. 1928.
\newblock {Limiting Forms of the Frequency Distribution of the Largest or
  Smallest Member of a Sample}.
\newblock \emph{{Mathematical Proceedings of the Cambridge Philosophical
  Society}}, 24(2): 180--190.

\bibitem[{Gelman et~al.(1995)Gelman, Carlin, Stern, and
  Rubin}]{gelman1995bayesian}
Gelman, A.; Carlin, J.~B.; Stern, H.~S.; and Rubin, D.~B. 1995.
\newblock \emph{{Bayesian Data Analysis}}.
\newblock Chapman and Hall/CRC.

\bibitem[{Germain et~al.(2016)Germain, Bach, Lacoste, and
  Lacoste-Julien}]{germain2016pac}
Germain, P.; Bach, F.; Lacoste, A.; and Lacoste-Julien, S. 2016.
\newblock {PAC-Bayesian Theory Meets Bayesian Inference}.
\newblock In \emph{NIPS}, volume~29.

\bibitem[{Gnedenko(1943)}]{gnedenko1943limiting}
Gnedenko, B. 1943.
\newblock {Sur La Distribution Limite Du Terme Maximum D'Une Serie Aleatoire}.
\newblock \emph{Annals of Mathematics}, 44(3): 423--453.

\bibitem[{Goodfellow et~al.(2014)Goodfellow, Pouget-Abadie, Mirza, Xu,
  Warde-Farley, Ozair, Courville, and Bengio}]{goodfellow2014generative}
Goodfellow, I.~J.; Pouget-Abadie, J.; Mirza, M.; Xu, B.; Warde-Farley, D.;
  Ozair, S.; Courville, A.~C.; and Bengio, Y. 2014.
\newblock {Generative Adversarial Nets}.
\newblock In \emph{{Proc. of the Advances in Neural Information Processing
  Systems (Neurips)}}, 2672--2680.

\bibitem[{Goodman et~al.(2012)Goodman, Mansinghka, Roy, Bonawitz, and
  Tenenbaum}]{goodman2012church}
Goodman, N.; Mansinghka, V.; Roy, D.~M.; Bonawitz, K.; and Tenenbaum, J.~B.
  2012.
\newblock {Church: A Language for Generative Models}.
\newblock \emph{arXiv preprint arXiv:1206.3255}.

\bibitem[{Grattan-Guinness(2005)}]{grattan2005landmark}
Grattan-Guinness, I. 2005.
\newblock \emph{{Landmark Writings in Western Mathematics 1640-1940}}.
\newblock Elsevier.

\bibitem[{Guedj(2019)}]{guedj2019primer}
Guedj, B. 2019.
\newblock {A Primer on PAC-Bayesian Learning}.
\newblock \emph{arXiv preprint arXiv:1901.05353}.

\bibitem[{Gurevich and Stuke(2020)}]{pavel2020conjugate}
Gurevich, P.; and Stuke, H. 2020.
\newblock {Gradient Conjugate Priors and Multi-Layer Neural Networks}.
\newblock \emph{{Artificial Intelligence}}, 278(C).

\bibitem[{Gutmann and Hyv{\"a}rinen(2010)}]{gutmann2010noise}
Gutmann, M.~U.; and Hyv{\"a}rinen, A. 2010.
\newblock {Noise-Contrastive Estimation: A New Estimation Principle for
  Unnormalized Statistical Models}.
\newblock In \emph{Proc. of the International Conference on Artificial
  Intelligence and Statistics (AISTATS)}, 297--304. JMLR Workshop and
  Conference Proceedings.

\bibitem[{Jaynes(1957)}]{maxent}
Jaynes, E.~T. 1957.
\newblock {Information Theory and Statistical Mechanics}.
\newblock \emph{Phys. Rev.}, 106: 620--630.

\bibitem[{Jaynes(1968)}]{jaynes1968prior}
Jaynes, E.~T. 1968.
\newblock {Prior Probabilities}.
\newblock \emph{IEEE Transactions on Systems Science and Cybernetics}, 4(3):
  227--241.

\bibitem[{Jospin et~al.(2022)Jospin, Laga, Boussaid, Buntine, and
  Bennamoun}]{jospin2022hands}
Jospin, L.~V.; Laga, H.; Boussaid, F.; Buntine, W.; and Bennamoun, M. 2022.
\newblock {Hands-on Bayesian Neural Networks -- A Tutorial for Deep Learning
  Users}.
\newblock \emph{IEEE Computational Intelligence Magazine}, 17(2): 29--48.

\bibitem[{Juang and Rabiner(1991)}]{juang1991hidden}
Juang, B.~H.; and Rabiner, L.~R. 1991.
\newblock {Hidden Markov Models for Speech Recognition}.
\newblock \emph{Technometrics}, 33(3): 251--272.

\bibitem[{Kendall and Gal(2017)}]{kendall2017uncertainties}
Kendall, A.; and Gal, Y. 2017.
\newblock {What Uncertainties Do We Need in Bayesian Deep Learning for Computer
  Vision?}
\newblock In \emph{{Proc. of the Advances in Neural Information Processing
  Systems (Neurips)}}, 5574--5584.

\bibitem[{Kiefer(1952)}]{kiefer1952sequential}
Kiefer, J. 1952.
\newblock {Sequential Minimax Estimation for the Rectangular Distribution with
  Unknown Range}.
\newblock \emph{Annals of Mathematical Statistics}, 586--593.

\bibitem[{Kingma et~al.(2014)Kingma, Mohamed, Rezende, and
  Welling}]{kingma2014semi}
Kingma, D.~P.; Mohamed, S.; Rezende, D.~J.; and Welling, M. 2014.
\newblock {Semi-Supervised Learning with Deep Generative Models}.
\newblock In \emph{{Proc. of the Advances in Neural Information Processing
  Systems (Neurips)}}, 3581--3589.

\bibitem[{Kocsis and Szepesv{\'a}ri(2006)}]{kocsis2006bandit}
Kocsis, L.; and Szepesv{\'a}ri, C. 2006.
\newblock {Bandit Based Monte-Carlo Planning}.
\newblock In \emph{{Proc. of the European Conference on Machine Learning and
  Principles and Practice of Knowledge Discovery in Databases}}, 282--293.
  Springer.

\bibitem[{Kolmogoroff(1929)}]{Kolmogoroff1929}
Kolmogoroff, A. 1929.
\newblock {\"U}ber das Gesetz des iterierten Logarithmus.
\newblock \emph{Mathematische Annalen}, 101(1): 126--135.

\bibitem[{Kolmogorov and Bharucha-Reid(1933)}]{kolmogorov1933foundations}
Kolmogorov, A.~N.; and Bharucha-Reid, A.~T. 1933.
\newblock \emph{Foundations of the Theory of Probability}.

\bibitem[{Koopman(1936)}]{koopman1936distributions}
Koopman, B.~O. 1936.
\newblock {On Distributions Admitting a Sufficient Statistic}.
\newblock \emph{Transactions of the American Mathematical society}, 39(3):
  399--409.

\bibitem[{Kotz, Kozubowski, and Podg{\'o}rski(2001)}]{kotz2001laplace}
Kotz, S.; Kozubowski, T.; and Podg{\'o}rski, K. 2001.
\newblock \emph{{The Laplace distribution and generalizations: a revisit with
  applications to communications, economics, engineering, and finance}}.
\newblock 183. Springer Science \& Business Media.

\bibitem[{Laplace(1812)}]{laplace1812centrallimittheorem}
Laplace, P.-S. 1812.
\newblock \emph{{Th{\'e}orie analytique des probabilit{\'e}s}}.

\bibitem[{Leemis and McQueston(2008)}]{leemis2008univariate}
Leemis, L.~M.; and McQueston, J.~T. 2008.
\newblock {Univariate Distribution Relationships}.
\newblock \emph{The American Statistician}, 62(1): 45--53.

\bibitem[{Li, Swersky, and Zemel(2015)}]{li2015generative}
Li, Y.; Swersky, K.; and Zemel, R. 2015.
\newblock {Generative Moment Matching Networks}.
\newblock In \emph{{Proc. of the International Conference on Machine Learning
  (ICML)}}, 1718--1727. PMLR.

\bibitem[{Lin and Whitehead(2015)}]{lin2015power}
Lin, Z.; and Whitehead, J. 2015.
\newblock {Why Power Laws? An Explanation from Fine-Grained Code Changes}.
\newblock In \emph{2015 IEEE/ACM 12th Working Conference on Mining Software
  Repositories}, 68--75. IEEE.

\bibitem[{McAllester(2003)}]{mcallester2003simplified}
McAllester, D. 2003.
\newblock {Simplified PAC-Bayesian margin bounds}.
\newblock In \emph{Learning theory and Kernel machines}, 203--215. Springer.

\bibitem[{Merriam-Webster(2022)}]{merriam-webster}
Merriam-Webster. 2022.
\newblock {"Statistics"}.
\newblock In \emph{{Merriam-Webster.com Dictionary,
  \url{https://www.merriam-webster.com/dictionary/statistics}, Accessed
  10/18/2022}}.

\bibitem[{Mikolov et~al.(2013)Mikolov, Sutskever, Chen, Corrado, and
  Dean}]{mikolov2013distributed}
Mikolov, T.; Sutskever, I.; Chen, K.; Corrado, G.~S.; and Dean, J. 2013.
\newblock {Distributed Representations of Words and Phrases and Their
  Compositionality}.
\newblock In \emph{{Proc. of the Advances in Neural Information Processing
  Systems (Neurips)}}, 3111--3119.

\bibitem[{Muise et~al.(2015)Muise, Felli, Miller, Pearce, and
  Sonenberg}]{muise2015leveraging}
Muise, C.; Felli, P.; Miller, T.; Pearce, A.~R.; and Sonenberg, L. 2015.
\newblock Leveraging FOND planning technology to solve multi-agent planning
  problems.
\newblock \emph{Distributed and Multi-Agent Planning (DMAP-15)}, 83.

\bibitem[{Murphy(2012)}]{murphy2012machine}
Murphy, K.~P. 2012.
\newblock \emph{{Machine Learning: A Probabilistic Perspective}}.
\newblock MIT Press.

\bibitem[{Newman(2005)}]{newman2005power}
Newman, M.~E. 2005.
\newblock {Power laws, Pareto distributions and Zipf's law}.
\newblock \emph{Contemporary physics}, 46(5): 323--351.

\bibitem[{Neyman(1937)}]{neyman1937outline}
Neyman, J. 1937.
\newblock {Outline of a Theory of Statistical Estimation based on the Classical
  Theory of Probability}.
\newblock \emph{Philosophical Transactions of the Royal Society of London.
  Series A, Mathematical and Physical Sciences}, 236(767): 333--380.

\bibitem[{Neyman and Pearson(1933)}]{neyman1933ix}
Neyman, J.; and Pearson, E.~S. 1933.
\newblock {On the Problem of the Most Efficient Tests of Statistical
  Hypotheses}.
\newblock \emph{Philosophical Transactions of the Royal Society of London.
  Series A, Containing Papers of a Mathematical or Physical Character},
  231(694-706): 289--337.

\bibitem[{Pearl and Mackenzie(2018)}]{pearl2018bookofwhy}
Pearl, J.; and Mackenzie, D. 2018.
\newblock \emph{{The Book of Why: The New Science of Cause and Effect}}.
\newblock Basic books.

\bibitem[{Pfanzagl(1967)}]{phanzagl1967}
Pfanzagl, J. 1967.
\newblock {Subjective Probability Derived from the Morgenstern-von Neumann
  Utility Theory}.
\newblock In Shubik, M., ed., \emph{{Essays in Mathematical Economics, in Honor
  of Oskar Morgenstern}}, volume 2174. Princeton University Press.

\bibitem[{{Pickands III}(1975)}]{pickands1975statistical}
{Pickands III}, J. 1975.
\newblock {Statistical Inference using Extreme Order Statistics}.
\newblock \emph{Annals of Statistics}, 119--131.

\bibitem[{Pitman(1936)}]{pitman1936sufficient}
Pitman, E. J.~G. 1936.
\newblock {Sufficient Statistics and Intrinsic Accuracy}.
\newblock In \emph{Mathematical Proceedings of the cambridge Philosophical
  society}, volume~32, 567--579. Cambridge University Press.

\bibitem[{Preda(1984)}]{preda1984informational}
Preda, V.~C. 1984.
\newblock {Informational Characterizing of the Pareto and Power Distributions}.
\newblock \emph{Bulletin math{\'e}matique de la Soci{\'e}t{\'e} des Sciences
  Math{\'e}matiques de la R{\'e}publique Socialiste de Roumanie}, 77--79.

\bibitem[{Ramesh et~al.(2022)Ramesh, Dhariwal, Nichol, Chu, and Chen}]{dalle2}
Ramesh, A.; Dhariwal, P.; Nichol, A.; Chu, C.; and Chen, M. 2022.
\newblock {Hierarchical Text-Conditional Image Generation with Clip Latents}.
\newblock \emph{arXiv preprint arXiv:2204.06125}.

\bibitem[{Ramesh et~al.(2021)Ramesh, Pavlov, Goh, Gray, Voss, Radford, Chen,
  and Sutskever}]{dalle}
Ramesh, A.; Pavlov, M.; Goh, G.; Gray, S.; Voss, C.; Radford, A.; Chen, M.; and
  Sutskever, I. 2021.
\newblock Zero-shot text-to-image generation.
\newblock In \emph{{Proc. of the International Conference on Machine Learning
  (ICML)}}, 8821--8831. PMLR.

\bibitem[{Ranganath, Gerrish, and Blei(2014)}]{ranganath2014black}
Ranganath, R.; Gerrish, S.; and Blei, D.~M. 2014.
\newblock {Black Box Variational Inference}.
\newblock In \emph{Proc. of the International Conference on Artificial
  Intelligence and Statistics (AISTATS)}, 814--822. PMLR.

\bibitem[{Rezende, Mohamed, and Wierstra(2014)}]{rezende2014stochastic}
Rezende, D.~J.; Mohamed, S.; and Wierstra, D. 2014.
\newblock Stochastic backpropagation and approximate inference in deep
  generative models.
\newblock In \emph{{Proc. of the International Conference on Machine Learning
  (ICML)}}, 1278--1286. PMLR.

\bibitem[{Rohatgi and Saleh(2015)}]{rohatgi2015introduction}
Rohatgi, V.~K.; and Saleh, A. M.~E. 2015.
\newblock \emph{{An Introduction to Probability and Statistics}}.
\newblock John Wiley \& Sons.

\bibitem[{Rossman, Short, and Parks(1998)}]{rossman1998bayes}
Rossman, A.~J.; Short, T.~H.; and Parks, M.~T. 1998.
\newblock {Bayes Estimators for the Continuous Uniform Distribution}.
\newblock \emph{Journal of Statistics Education}, 6(3).

\bibitem[{Roth(1996)}]{roth1996hardness}
Roth, D. 1996.
\newblock {On the Hardness of Approximate Reasoning}.
\newblock \emph{{Artificial Intelligence}}, 82(1-2): 273--302.

\bibitem[{Russell et~al.(1995)Russell, Norvig, Canny, Malik, and
  Edwards}]{russell1995artificial}
Russell, S.~J.; Norvig, P.; Canny, J.~F.; Malik, J.~M.; and Edwards, D.~D.
  1995.
\newblock \emph{{Artificial Intelligence: A Modern Approach}}, volume~2.
\newblock Prentice hall Englewood Cliffs.

\bibitem[{Savage(1954)}]{savage1954foundations}
Savage, L.~J. 1954.
\newblock \emph{{The Foundations of Statistics}}.
\newblock Courier Corporation.

\bibitem[{Settles(2012)}]{burr2012active}
Settles, B. 2012.
\newblock \emph{{Active Learning}}.
\newblock Synthesis Lectures on Artificial Intelligence and Machine Learning.
  Morgan {\&} Claypool Publishers.

\bibitem[{Shannon(1949)}]{shannon1949synthesis}
Shannon, C.~E. 1949.
\newblock {The synthesis of two-terminal switching circuits}.
\newblock \emph{Bell System Technical Journal}, 28(1): 59--98.

\bibitem[{Sohl-Dickstein et~al.(2015)Sohl-Dickstein, Weiss, Maheswaranathan,
  and Ganguli}]{sohl2015deep}
Sohl-Dickstein, J.; Weiss, E.; Maheswaranathan, N.; and Ganguli, S. 2015.
\newblock {Deep Unsupervised Learning using Nonequilibrium Thermodynamics}.
\newblock In \emph{{Proc. of the International Conference on Machine Learning
  (ICML)}}, 2256--2265. PMLR.

\bibitem[{Srivastava et~al.(2017)Srivastava, Valkov, Russell, Gutmann, and
  Sutton}]{srivastava2017veegan}
Srivastava, A.; Valkov, L.; Russell, C.; Gutmann, M.~U.; and Sutton, C. 2017.
\newblock {VEEGAN: Reducing Mode Collapse in GANs using Implicit Variational
  Learning}.
\newblock In \emph{NIPS}, volume~30.

\bibitem[{Srivastava et~al.(2020)Srivastava, Xu, Gutmann, and
  Sutton}]{srivastava2019generative}
Srivastava, A.; Xu, K.; Gutmann, M.~U.; and Sutton, C. 2020.
\newblock {Generative Ratio Matching Networks}.
\newblock In \emph{{Proc. of the International Conference on Learning
  Representations (ICLR)}}.

\bibitem[{Sugiyama, Suzuki, and Kanamori(2012)}]{sugiyama2012density}
Sugiyama, M.; Suzuki, T.; and Kanamori, T. 2012.
\newblock \emph{{Density Ratio Estimation in Machine Learning}}.
\newblock Cambridge University Press.

\bibitem[{Sutton and Barto(2018)}]{sutton2018reinforcement}
Sutton, R.~S.; and Barto, A.~G. 2018.
\newblock \emph{{Reinforcement Learning: An Introduction}}.
\newblock MIT Press.

\bibitem[{Tenenbaum(1998)}]{tenenbaum1998bayesian}
Tenenbaum, J.~B. 1998.
\newblock {Bayesian Modeling of Human Concept Learning}.
\newblock In \emph{NIPS}, volume~11.

\bibitem[{Valiant(1979)}]{valiant1979complexity}
Valiant, L.~G. 1979.
\newblock {The Complexity of Computing the Permanent}.
\newblock \emph{Theoretical computer science}, 8(2): 189--201.

\bibitem[{Valiant(1984)}]{valiant1984theory}
Valiant, L.~G. 1984.
\newblock {A Theory of the Learnable}.
\newblock \emph{Communications of the ACM}, 27(11): 1134--1142.

\bibitem[{van~den Oord, Vinyals et~al.(2017)}]{van2017neural}
van~den Oord, A.; Vinyals, O.; et~al. 2017.
\newblock {Neural Discrete Representation Learning}.
\newblock In \emph{{Proc. of the Advances in Neural Information Processing
  Systems (Neurips)}}.

\bibitem[{Vaswani et~al.(2017)Vaswani, Shazeer, Parmar, Uszkoreit, Jones,
  Gomez, Kaiser, and Polosukhin}]{vaswani2017attention}
Vaswani, A.; Shazeer, N.; Parmar, N.; Uszkoreit, J.; Jones, L.; Gomez, A.~N.;
  Kaiser, {\L}.; and Polosukhin, I. 2017.
\newblock {Attention is All You Need}.
\newblock In \emph{{Proc. of the Advances in Neural Information Processing
  Systems (Neurips)}}, 5998--6008.

\bibitem[{{Von Neumann} and Morgenstern(1944)}]{von1944theory}
{Von Neumann}, J.; and Morgenstern, O. 1944.
\newblock \emph{{Theory of Games and Economic Behavior}}.
\newblock Princeton university press.

\bibitem[{Wingate and Weber(2013)}]{wingate2013automated}
Wingate, D.; and Weber, T. 2013.
\newblock {Automated Variational Inference in Probabilistic Programming}.
\newblock \emph{arXiv preprint arXiv:1301.1299}.

\bibitem[{Yang, Wu, and Jiang(2007)}]{YangWJ07}
Yang, Q.; Wu, K.; and Jiang, Y. 2007.
\newblock {Learning Action Models from Plan Examples using Weighted {MAX-SAT}}.
\newblock \emph{{Artificial Intelligence}}, 171(2-3): 107--143.

\end{thebibliography}
\end{document}